\documentclass[10pt]{article}

\usepackage[letterpaper,margin=1in]{geometry}
\usepackage[T1]{fontenc}
\usepackage[utf8]{inputenc}
\usepackage{lmodern}
\usepackage{microtype}
\usepackage[authoryear,round]{natbib}
\setcitestyle{authoryear,round,citesep={;},aysep={,},yysep={;}}
\usepackage{amsmath,amssymb,amsfonts,amsthm,mathtools,bm}
\usepackage{booktabs,graphicx,float,placeins,enumitem,array,multirow,multicol,xcolor,subcaption,url}
\usepackage{hyperref}
\usepackage[nameinlink,capitalise,noabbrev]{cleveref}
\newtheorem{theorem}{Theorem}[section]
\newtheorem{proposition}[theorem]{Proposition}
\newtheorem{lemma}[theorem]{Lemma}
\newtheorem{corollary}[theorem]{Corollary}

\theoremstyle{definition}

\theoremstyle{remark}
\newtheorem{remark}[theorem]{Remark}

\AddToHook{env/theorem/begin}{\crefalias{theorem}{theorem}\crefalias{section}{theorem}}
\AddToHook{env/proposition/begin}{\crefalias{theorem}{proposition}\crefalias{section}{proposition}}
\AddToHook{env/lemma/begin}{\crefalias{theorem}{lemma}\crefalias{section}{lemma}}
\AddToHook{env/corollary/begin}{\crefalias{theorem}{corollary}\crefalias{section}{corollary}}
\AddToHook{env/assumption/begin}{\crefalias{theorem}{assumption}\crefalias{section}{assumption}}
\AddToHook{env/definition/begin}{\crefalias{theorem}{definition}\crefalias{section}{definition}}
\AddToHook{env/example/begin}{\crefalias{theorem}{example}\crefalias{section}{example}}
\AddToHook{env/remark/begin}{\crefalias{theorem}{remark}\crefalias{section}{remark}}

\crefname{theorem}{theorem}{theorems}
\Crefname{theorem}{Theorem}{Theorems}
\crefname{proposition}{proposition}{propositions}
\Crefname{proposition}{Proposition}{Propositions}
\crefname{lemma}{lemma}{lemmas}
\Crefname{lemma}{Lemma}{Lemmas}
\crefname{corollary}{corollary}{corollaries}
\Crefname{corollary}{Corollary}{Corollaries}
\crefname{assumption}{assumption}{assumptions}
\Crefname{assumption}{Assumption}{Assumptions}
\crefname{definition}{definition}{definitions}
\Crefname{definition}{Definition}{Definitions}
\crefname{example}{example}{examples}
\Crefname{example}{Example}{Examples}
\crefname{remark}{remark}{remarks}
\Crefname{remark}{Remark}{Remarks}

\newcommand{\R}{\mathbb{R}}

\newcommand{\E}{\mathbb{E}}

\newcommand{\cX}{\mathcal{X}}
\newcommand{\cF}{\mathcal{F}}
\newcommand{\cH}{\mathcal{H}}
\newcommand{\cZ}{\mathcal{Z}}
\newcommand{\cT}{\mathcal{T}}
\newcommand{\cV}{\mathcal{V}}
\newcommand{\Rad}{\widehat{\mathfrak R}}
\newcommand{\kb}{k_{\mathrm B}}

\newcommand{\eps}{\varepsilon}
\newcommand{\op}{\mathrm{op}}
\newcommand{\HS}{\mathrm{HS}}
\newcommand{\relu}{\operatorname{ReLU}}
\newcommand{\tr}{\operatorname{Tr}}

\newcommand{\vc}{\operatorname{VCdim}}

\newcommand{\1}{\mathbf 1}
\newcommand{\dd}{\,\mathrm d}
\newcommand{\norm}[1]{\left\lVert #1\right\rVert}
\newcommand{\abs}[1]{\left\lvert #1\right\rvert}
\newcommand{\inner}[2]{\left\langle #1,#2\right\rangle}
\newcommand{\set}[1]{\left\{#1\right\}}
\newcommand{\given}{\,\middle|\,}

\setlist[itemize]{leftmargin=1.5em,itemsep=1.5pt,topsep=2pt}
\setlist[enumerate]{leftmargin=1.7em,itemsep=1.5pt,topsep=2pt}
\allowdisplaybreaks

\makeatletter
\renewcommand{\section}{\@startsection{section}{1}{\z@}%
  {-9pt plus -2pt minus -1pt}{5pt plus 1pt minus 1pt}%
  {\normalfont\large\bfseries}}
\renewcommand{\subsection}{\@startsection{subsection}{2}{\z@}%
  {-7pt plus -2pt minus -1pt}{4pt plus 1pt minus 1pt}%
  {\normalfont\normalsize\bfseries}}
\renewcommand{\paragraph}{\@startsection{paragraph}{4}{\z@}%
  {5pt plus 1pt minus 1pt}{-0.7em}%
  {\normalfont\normalsize\bfseries}}
\makeatother
\AtBeginDocument{%
\setlength{\abovedisplayskip}{6pt plus 2pt minus 2pt}%
\setlength{\belowdisplayskip}{6pt plus 2pt minus 2pt}%
\setlength{\abovedisplayshortskip}{0pt plus 2pt}%
\setlength{\belowdisplayshortskip}{4pt plus 2pt minus 2pt}%
}

\newcommand{\PaperTitle}{Brownian Heads for Deep ReLU Representations: Activation Mass and the Cost of Same-Sample Selection}
\newcommand{\PaperPDFAuthors}{Mahdi Mohammadigohari and Nicole M\"ucke}
\hypersetup{
  colorlinks=true,
  linkcolor=blue!50!black,
  citecolor=green!35!black,
  urlcolor=blue!55!black,
  pdftitle={\PaperTitle},
  pdfauthor={\PaperPDFAuthors},
  pdfpagemode=UseNone
}
\title{\PaperTitle}
\author{%
\begin{minipage}[t]{0.47\textwidth}\centering\small
Mahdi Mohammadigohari\\
Faculty of Engineering\\
Free University of Bozen--Bolzano\\
Bruno Buozzi 1\\
39100 Bolzano, Italy\\
\texttt{Mahdi.Mohammadigohari@gmail.com}
\end{minipage}\hspace{0.03\textwidth}%
\begin{minipage}[t]{0.47\textwidth}\centering\small
Nicole M{\"u}cke\\
Institute for Statistics and\\
Foundations of Machine Learning\\
Technische Universit{\"a}t Braunschweig\\
Universit{\"a}tsplatz 2\\
38106 Braunschweig, Germany\\
\texttt{nicole.muecke@tu-braunschweig.de}
\end{minipage}}
\date{}
\begin{document}
\maketitle
\raggedbottom

\begin{abstract}
Deep representation learning often selects hidden features and fits the final predictor on the same sample, so fixed-feature analysis performed after selection can omit selection cost. We study the conditional empirical Rademacher complexity of deep ReLU representations followed by bounded-norm predictors in additive or L\'evy--Brownian RKHSs, termed \emph{Brownian heads}. For a fixed representation, we derive an exact dual identity and sharp bounds in terms of activation mass, the average norm of the observed hidden vectors. Under same-sample selection, the representation supremum induces a quadratic Rademacher process. Brownian layer-cake and Gaussian-projection identities reduce it to coordinatewise or signed projected threshold traces, separating realized scale from selection complexity. For samples with pairwise-distinct inputs, explicit scalar ReLU families match the finite-trace and VC rates up to universal constants at the realized trace-and-envelope level. Induced-norm contraction also yields architecture-level bounds for rectangular, rank-deficient ReLU networks. Experiments verify the sharp bounds and rates, exhibit a selection gap at fixed activation mass, and assess the predictive feasibility of Brownian heads.
\end{abstract}

\section{Introduction and Related Work}
\label{sec:introduction}

Modern learning architectures are commonly decomposed into a representation
map and a comparatively simple prediction head. Given an input space
$\mathcal X$, such predictors take the form $f=h\circ\phi$, 
where $\phi\colon\mathcal X\to\mathcal Z$ extracts task-relevant features and the head 
$h$ maps the resulting representation to the output space. This viewpoint
covers end-to-end neural networks, feature extraction followed by kernel or
linear prediction, and the selection of representations from intermediate
training checkpoints. Its statistical analysis depends crucially on whether
the representation is fixed independently of the sample used to train the
head or is selected using that same sample.

Generalization bounds control the discrepancy between empirical and population
performance through the Rademacher complexity (RC) of the prediction
class~\citep{BartlettMendelson2002, Mohri2018}. 
For a representation $\phi$ fixed
independently of the training sample $S$, the terminal class
\[
    \mathcal F_\phi=\{h\circ\phi:h\in\mathcal H\}
\]
is fixed, and standard symmetrization applies. If the same sample selects a
representation from a family $\Phi(S)$, however, the relevant class is
\[
    \mathcal F_S
    =\bigcup_{\phi\in\Phi(S)}\{h\circ\phi:h\in\mathcal H\}.
\]
A fixed-representation bound applied only after selection treats this choice as
free and may therefore omit part of the statistical complexity. We call this
setting \emph{same-sample representation selection}.

The distinction covers several common regimes. An encoder pretrained without
the target sample and fixed before that sample is observed defines a fixed
representation; choosing among several such encoders using the target sample
already incurs a finite model-selection cost. Fine-tuning, adapter selection,
checkpoint selection, and end-to-end training make the representation
sample-dependent. A uniform analysis must then cover the representations the
procedure could have returned, not only its realized output. This issue occurs
within training and is distinct from validation- or test-set leakage.

To isolate this effect, we place a structured terminal class after a deep ReLU
representation. Let $z_\theta\colon\mathcal X\to\mathbb R^m$ denote the output
of the last hidden layer of a feedforward ReLU network. For $q\in\{1,2\}$, we
study the complexity of classes containing functions of the form  
\begin{equation}
    f_{a,\theta}^{(q)}(x)
    =
    \inner{a}{\Phi_q(z_\theta(x))}_{\cH_q},
    \qquad
    \norm{a}_{\cH_q}\leq B,
    \label{eq:main-model}
\end{equation}
where $\Phi_q$ is the canonical feature map of the additive Brownian kernel for
$q=1$ and of the rotation-invariant L\'evy--Brownian kernel for $q=2$. Thus the
network produces $z_\theta(x)$, while a bounded-norm predictor in the
corresponding Brownian RKHS produces the output. We call the latter a
\emph{Brownian head}. 

Since $K_q(z,z)=\norm{z}_q$, the fixed-representation RC is governed by
the empirical \emph{activation mass}
$\widehat M_{q,S}(z):=n^{-1}\sum_{i=1}^n\norm{z(x_i)}_q$.
Under same-sample selection, the representation supremum instead produces a
quadratic Rademacher process. Brownian layer-cake and Gaussian-projection
identities reduce this process to coordinatewise or signed projected threshold
traces, separating representation scale from selection complexity. We obtain
matching finite-trace and VC rates through explicit scalar ReLU constructions
and propagate Brownian feature distance through arbitrary rectangular and
rank-deficient ReLU layers using induced matrix norms.

Our bounds are empirical and conditional on the realized representation
family. For a sample-independent family $\mathcal Z$, standard symmetrization
and contraction yield population guarantees for bounded Lipschitz losses; see
\Cref{cor:fixed-generalization}. If $\mathcal Z=\mathcal Z(S)$ is itself
sample-dependent, \Cref{thm:main-adaptive} controls the conditional empirical
complexity, 
while a complete
population bound additionally requires control of the dependence between the
random class and $S$, for example through a random-class
argument~\citep{Dupuis2024}. Thus, our results isolate the selection term
omitted by a post hoc fixed-representation analysis.

\paragraph{Contributions.}
\begin{itemize}[itemsep=0.5pt,topsep=1pt]

\item
For a fixed representation $z$, we derive an exact dual identity and show that the empirical Rademacher complexity of its Brownian-head class scales as $B\sqrt{\widehat M_{q,S}(z)/n}$, with optimal lower and upper constants $1/\sqrt{2}$ and $1$. Both endpoints are attained by finite Brownian canonical-feature configurations; see \Cref{thm:main-fixed}.

    \item
    For a representation family $\mathcal Z$ that may be constructed from the
    same sample, we reduce the representation-indexed quadratic Rademacher
    process to coordinatewise or signed Gaussian-projected threshold traces.
    The resulting bounds separate representation scale from threshold-selection
    complexity; see \Cref{lem:main-chaos} and \Cref{thm:main-adaptive}.

    \item
    We prove that the finite-trace and VC dependences are unavoidable, up to
    universal constants, at the trace-and-envelope level. For pairwise-distinct
    sample inputs, explicit scalar one-hidden-layer ReLU families attain the
    corresponding rates; see
    \Cref{thm:main-matching} and \Cref{prop:main-indicator}.

    \item
    We propagate Brownian feature distance through arbitrary rectangular and
    rank-deficient exact-ReLU layers using induced matrix norms, without
    inverse maps or determinant factors, and obtain an explicit deep-network
    envelope; see \Cref{prop:main-propagation2} and \Cref{cor:main-explicit}.

    \item
    Exact finite-sample constructions recover the fixed-representation
    endpoints and the finite-trace and VC rates. Controlled same-sample adapter
    families isolate selection at fixed empirical activation mass, while
    frozen-feature transfer experiments assess the predictive feasibility of
    Brownian heads; see \Cref{sec:experiments,app:experiments}.

\end{itemize}

\paragraph{Related work.}\label{sec:related}
\textbf{Data-dependent hypothesis classes.}
Classical uniform-convergence arguments assume a sample-independent hypothesis
class. \citet{Foster2019} study sample-dependent classes through transductive
Rademacher complexity and hypothesis-set stability, including stable
representation learning; related approaches use algorithm-dependent empirical
Rademacher complexity~\citep{Sachs2023} or PAC-Bayesian bounds for random
hypothesis sets~\citep{Dupuis2024}. These frameworks control the dependence
between a random class and its generating sample. Conditional on a realized
family $\mathcal Z$, we instead exploit Brownian threshold structure to obtain
explicit empirical complexity bounds and matching trace-and-envelope lower
bounds for exact-ReLU representations.

\noindent\textbf{Brownian and learned kernels.}
The additive and L\'evy--Brownian kernels are classical distance-induced
kernels~\citep{Sejdinovic2013}. Data-dependent kernel-selection bounds
typically control a prespecified family through entropy or kernel
pseudodimension~\citep{YingCampbell2010,CortesMohriRostamizadeh2010,
ZhangZhang2023}. \citet{FollainBach2025} analyze a shallow Brownian-kernel
network that jointly learns a distribution of one-dimensional projections and
the predictor, producing a same-sample supremum over Brownian Gram quadratic
forms. We place a Brownian RKHS head after a potentially deep ReLU
representation and reduce this process to coordinatewise or projected
threshold traces, with matching sample-level ReLU constructions.

\noindent\textbf{Composition-operator bounds.}
Koopman or composition-operator approaches represent layers through
$h\mapsto h\circ\varphi$. On Sobolev RKHSs, \citet{Hashimoto2024} obtain
determinant-sensitive bounds requiring injective, full-rank layers and smooth
bi-Lipschitz activations. \citet{Hashimoto2026} allow bounded domains, changing
widths, noninjective weights, and several activations, but their assumptions
exclude exact ReLU. Our direct Brownian-metric argument permits rectangular,
rank-deficient exact-ReLU layers and controls same-sample selection through
threshold traces rather than composition-operator norms.

\noindent\textbf{Recursive RKHS constructions.}
Brownian Kernel Ladders recursively average Brownian comparisons of
preceding-level functions and distinguish a fully adaptive variational family
from a fixed ladder or sample-independent finite
dictionary~\citep{MohammadigohariBKL2026}. In \Cref{sec:brownian-chains}, we
combine the metric and energy recursions of a fixed ladder with our
fixed-representation theorem; this is a consequence of the Brownian-head
analysis, not part of the same-sample argument. Related layerwise models
include Neural Hilbert Ladders and reproducing-kernel
chains~\citep{Chen2023,HeeringaSpekBrune2025}.

\textbf{Notation.}
For $p\in\mathbb N$, let $[p]:=\{1,\ldots,p\}$, and for a finite set $I$, let
$2^I:=\{T:T\subseteq I\}$ denote its power set. For
$v=(v_j)_{j=1}^d\in\R^d$ and $q\in\{1,2\}$, write
$\norm{v}_q:=\bigl(\sum_{j=1}^d\lvert v_j\rvert^q\bigr)^{1/q}$. For
$W\in\R^{d'\times d}$, let
$\norm{W}_{q\to q}:=\sup_{v\in\R^d\setminus\{0\}}
\norm{Wv}_q/\norm{v}_q$ denote the induced operator norm. The ReLU function
$\relu(t):=\max\{t,0\}$ is applied coordinatewise to vectors, and $\tr(A)$
denotes the trace of a square matrix $A$. For nonnegative $a$ and $b$, we write
$a\lesssim b$ if $a\leq Cb$ and $a\asymp b$ if $cb\leq a\leq Cb$, where
$0<c<C$ are universal constants.

\section{Brownian Geometry and the Fixed-Representation Baseline}
\label{sec:main-geometry}\label{sec:main-fixed}

We first identify the empirical capacity of a Brownian head on one representation, then bound it through the network layers.

For $u,v\in\R^m$ and $q\in\{1,2\}$, define
\begin{equation}
K_q(u,v)=\frac{\norm{u}_q+\norm{v}_q-\norm{u-v}_q}{2},
\qquad \Phi_q(u)=K_q(\cdot,u).
\label{eq:main-brownian-kernel}
\end{equation}
These classical positive-definite, distance-induced kernels~\citep{Sejdinovic2013}
define RKHSs $\cH_q$ with canonical maps $\Phi_q\colon\R^m\to\cH_q$.
The additive kernel $K_1$ sums scalar Brownian kernels over coordinates;
the rotation-invariant L\'evy--Brownian kernel $K_2$ is a Gaussian mixture
of scalar Brownian projection kernels (\Cref{prop:pd-kernels}). Their basic identities are:

\begin{proposition}[Distance-induced Brownian geometry]
\label{prop:main-isometry}
For every $u,v\in\R^m$ and $q\in\{1,2\}$,
\begin{equation}
\norm{\Phi_q(u)}_{\cH_q}^2=\norm{u}_q,
\qquad
\norm{\Phi_q(u)-\Phi_q(v)}_{\cH_q}^2=\norm{u-v}_q.
\label{eq:main-isometry}
\end{equation}
\end{proposition}

The proof is in \Cref{proof:main-isometry}. These identities connect activation
mass to feature size and convert representation distances into squared feature distances.

Fix a map $z\colon\cX\to\R^m$, an observed
sample $S=(x_i)_{i=1}^n$, and a head radius $B>0$. 
The two objects associated with the pair $(S,z)$ are 
\begin{equation}
 \cF_{z,B}^{(q)}
 =
 \set{
 x\mapsto\inner{a}{\Phi_q(z(x))}_{\cH_q}
 \given
 a\in\cH_q,\ \norm{a}_{\cH_q}\le B
 },
 \qquad
 \widehat M_{q,S}(z)
 =
 \frac1n\sum_{i=1}^n\norm{z(x_i)}_q.
 \label{eq:main-fixed-class}
\end{equation}
Here $\widehat M_{q,S}(z)$ is the empirical \emph{activation mass}, the average
squared norm of the canonical features by \eqref{eq:main-isometry}.

We measure the size of the fixed terminal class by its \emph{empirical Rademacher}
complexity. For a class of real-valued functions $\mathcal G$, write
\begin{equation}
\Rad_S(\mathcal G)
=\frac1n\E_\eps\sup_{g\in\mathcal G}\sum_{i=1}^n\eps_i g(x_i),
\qquad
\eps_i\overset{\mathrm{i.i.d.}}{\sim}\operatorname{Unif}\{-1,+1\}.
\label{eq:rademacher-definition}
\end{equation}
The expectation in \eqref{eq:rademacher-definition} is taken over the auxiliary
Rademacher variables $\eps_1,\ldots,\eps_n$, conditionally on the sample $S$
and the representation $z$. Hence, the theorem below holds for every realized pair $(S,z)$, including
representations constructed from the same observed training data. 
Whenever a supremum over an uncountable class is not known to be
measurable, expectations and probabilities involving that supremum are
understood in the outer sense.

\begin{theorem}[Activation-mass characterization on a fixed representation]
\label{thm:main-fixed}
For $q\in\{1,2\}$,
\begin{equation}
 \Rad_S(\cF_{z,B}^{(q)})
 =\frac Bn\E_\eps\norm{\sum_{i=1}^n\eps_i\Phi_q(z(x_i))}_{\cH_q},
 \label{eq:main-exact-rad}
\end{equation}
with
\begin{equation}
 \frac{B}{\sqrt2}\sqrt{\frac{\widehat M_{q,S}(z)}{n}}
 \le \Rad_S(\cF_{z,B}^{(q)})
 \le B\sqrt{\frac{\widehat M_{q,S}(z)}{n}}.
 \label{eq:main-two-sided}
\end{equation}
The constants $1/\sqrt2$ and $1$ are optimal. If $n\ge2$, both endpoints are attained by nonnegative Brownian canonical feature configurations at every prescribed positive activation mass.
\end{theorem}
The proof in \Cref{proof:main-fixed} combines Hilbert-space duality,
the Brownian diagonal identity, and the Hilbert-valued Khintchine inequality.

\paragraph{Interpreting the fixed baseline.}
Let
$K_S:=\bigl(K_q(z(x_i),z(x_j))\bigr)_{i,j=1}^n$ be the realized Brownian
Gram matrix and let $\eps=(\eps_1,\ldots,\eps_n)^\top$. Then the empirical RC becomes
\[
    \Rad_S(\cF_{z,B}^{(q)})
    =
    \frac Bn\E_\eps\sqrt{\eps^\top K_S\eps},
    \qquad
    \tr(K_S)
    =
    n\widehat M_{q,S}(z).
\]

Thus, the activation mass records the total squared size of the canonical
features, whereas the off-diagonal entries of $K_S$ record their pairwise
alignment. It therefore determines the sharp universal range in
\eqref{eq:main-two-sided}, but not the exact complexity. Indeed, one nonzero
canonical feature eliminates cancellation and attains the upper endpoint,
whereas two identical nonzero features cancel with probability $1/2$ and
attain the lower endpoint. The Gram-sensitive bound in
\Cref{prop:main-gram} refines these trace-only bounds when the off-diagonal
energy is known.

When hidden states are unavailable, a layer-norm envelope gives an architecture-level bound.
Assume that $\cX\subseteq\R^{d_0}$ and let $d_L=m$. For each
$\ell\in[L]$, let $W_\ell\in\R^{d_\ell\times d_{\ell-1}}$, $b_\ell\in\R^{d_\ell}$, 
and define
$z_0(x)=x$, $z_\ell(x)=\relu\bigl(W_\ell z_{\ell-1}(x)+b_\ell\bigr)$, $\ell\in[L]$. 
No $W_\ell$ is assumed square or full rank, and we write
$s_{\ell,q}:=\norm{W_\ell}_{q\to q}$.

Pairwise feature-distance propagation and its one-point, bias-sensitive envelope
are stated and proved in \Cref{prop:main-propagation,proof:main-propagation}.
The pairwise argument uses coordinatewise $1$-Lipschitzness; the one-point bound
also uses $\relu(0)=0$. Neither needs determinants, inverse maps, or a positive
minimum singular value. Averaging the envelope gives the following consequence.

\begin{proposition}[Architecture-level fixed-head bound for rectangular ReLU layers]
\label{prop:main-propagation2}
For every sample $S=(x_i)_{i=1}^n$ and $q\in\{1,2\}$,
\begin{equation}
\Rad_S(\cF_{z_L,B}^{(q)})
\le
\frac{B}{\sqrt n}
\left[
\left(\prod_{\ell=1}^L s_{\ell,q}\right)
\frac1n\sum_{i=1}^n\norm{x_i}_q
+
\sum_{r=1}^L\norm{b_r}_q
\prod_{\ell=r+1}^L s_{\ell,q}
\right]^{1/2}.
\label{eq:main-architecture-rad2}
\end{equation}
No weight matrix is required to be square or full rank.
\end{proposition}

The bound follows by averaging the one-point estimate
\eqref{eq:main-affine-envelope} and applying
\eqref{eq:main-two-sided}; see \Cref{proof:main-propagation}.
It is necessarily one-sided, since global layer norms need not control the
activations realized on the sample; see \Cref{prop:no-product-lower}.

For a representation fixed independently of the sample, standard
symmetrization and loss contraction turn the empirical estimate into the risk
bound stated in \Cref{cor:fixed-generalization}. If the same sample is used to
select the representation, the conditional empirical identity remains valid,
but conditioning on the selected representation cannot justify the same
population argument: it removes the selection event from the analysis.
The next section retains the corresponding representation supremum explicitly.

\section{Same-Sample Representation Selection}
\label{sec:main-adaptive}

Let $\cZ$ be a family of representations $z\colon\cX\to\R^m$ with a common
finite output dimension $m$, and set
\[
 \cF_{\cZ,B}^{(q)}:=\bigcup_{z\in\cZ}\cF_{z,B}^{(q)},
 \qquad
 \widehat M^*_{q,S}(\cZ):=\sup_{z\in\cZ}\frac1n\sum_{i=1}^n\norm{z(x_i)}_q.
\]
The family may depend on the observed inputs and training labels, but is held
fixed before the auxiliary Rademacher signs are drawn. The upper bounds allow
arbitrary, possibly uncountable families for which the displayed quantities
are finite. For traces $\cT\subseteq2^{[n]}$, define
\begin{equation}
 \rho_n(\cT)=\frac1n\E_\eta\sup_{T\in\cT}\abs{\sum_{i\in T}\eta_i},
 \qquad \eta_i\overset{\mathrm{i.i.d.}}{\sim}\operatorname{Unif}\{-1,+1\}.
 \label{eq:main-trace-rad}
\end{equation}
Every realized trace family has at most $2^n$ elements, even when $\cZ$ is uncountable.

\paragraph{Why activation mass alone is insufficient.}
For pairwise distinct inputs, \Cref{prop:main-indicator} realizes every
$T\subseteq[n]$ by a scalar one-hidden-layer ReLU representation
$z_T(x_i)=A\mathbf1_{\{i\in T\}}$, $A>0$. The scalar kernels coincide, giving
\[
 \Rad_S\!\left(\bigcup_{T\in\cT}\cF_{z_T,B}^{(q)}\right)
 =B\sqrt A\,\rho_n(\cT),\qquad q\in\{1,2\}.
\]
For even $n$ and the family $\cT=\{T:|T|=n/2\}$ fixed before the signs,
every candidate has activation mass $A/2$ and fixed-head complexity of order
$B\sqrt{A/n}$. Yet $\rho_n(\cT)=1/2$, so the union complexity is
$B\sqrt A/2$. The fixed-head identity still holds for a selected member;
its value alone does not control the selection procedure.

\paragraph{From representation selection to threshold traces.}
Writing $z_i=z(x_i)$, Hilbert-space duality and the Brownian diagonal identity give
\[
 \Rad_S(\cF_{\cZ,B}^{(q)})
 =\frac Bn\E_\eps\sup_{z\in\cZ}
 \norm{\sum_{i=1}^n\eps_i\Phi_q(z_i)}_{\cH_q},
\]
\begin{equation}
 \norm{\sum_{i=1}^n\eps_i\Phi_q(z_i)}_{\cH_q}^2
 =\sum_{i=1}^n\norm{z_i}_q+2\sum_{i<j}\eps_i\eps_jK_q(z_i,z_j).
 \label{eq:main-feature-expansion-visible}
\end{equation}
The diagonal is at most $n\widehat M^*_{q,S}(\cZ)$. The off-diagonal term
has mean zero for each fixed $z$, but selection inside the expectation can
align the kernel values with $\eps_i\eps_j$.

For $q=1$, assume henceforth that $z(x_i)\in\R_+^m$ for all $z\in\cZ$
and $i\in[n]$, as after a final ReLU layer. The coordinatewise layer-cake identity is
\[
 K_1(z_i,z_j)=\sum_{r=1}^m\int_0^\infty
 \mathbf1_{\{z_r(x_i)\ge t\}}\mathbf1_{\{z_r(x_j)\ge t\}}\,\dd t.
\]
For $q=2$, let $c_0=\sqrt{\pi/2}$, $G\sim\mathcal N(0,I_m)$, and
$\kb(\alpha,\beta)=(|\alpha|+|\beta|-|\alpha-\beta|)/2$.
The Gaussian-mixture identity of \Cref{prop:pd-kernels} yields
\[
 K_2(u,v)=c_0\E_G\kb(\langle G,u\rangle,\langle G,v\rangle),\qquad
 \kb(\alpha,\beta)=\sum_{s\in\{-1,+1\}}\int_0^\infty
 \mathbf1_{\{s\alpha\ge t\}}\mathbf1_{\{s\beta\ge t\}}\,\dd t.
\]
Thus coordinates for $q=1$, or signed projections for $q=2$, reduce the
quadratic kernel process to threshold subsets of the sample:
\begin{align}
 \cT_{1,S}(\cZ)
 &:=\bigl\{\{i\in[n]:z_j(x_i)\ge t\}:z\in\cZ,\ j\in[m],\ t>0\bigr\},
 \label{eq:main-trace-q1}\\
 \cT_{2,S}(\cZ)
 &:=\bigl\{\{i\in[n]:s\langle g,z(x_i)\rangle\ge t\}:
 z\in\cZ,\ g\in\R^m,\ s\in\{-1,+1\},\ t>0\bigr\},
 \label{eq:main-trace-q2}\\
 \rho_{q,S}(\cZ)&:=\rho_n\bigl(\cT_{q,S}(\cZ)\bigr).
 \label{eq:main-rho-q}
\end{align}
Both signs in the projected traces account for the two Brownian half-lines.
To separate this selection freedom from the threshold ranges, define
\[
 \widehat A_{1,S}(\cZ):=\sup_{z\in\cZ}\sum_{j=1}^m\max_{i\in[n]}z_j(x_i),
 \qquad
 \widehat W_{2,S}(\cZ):=c_0\E_G\sup_{\substack{z\in\cZ\\i\in[n]}}
 |\langle G,z(x_i)\rangle|.
\]
These are the total coordinatewise range and expected largest Gaussian
projection, respectively; they satisfy
$\widehat M^*_{1,S}(\cZ)\le\widehat A_{1,S}(\cZ)$ and
$\widehat M^*_{2,S}(\cZ)\le\widehat W_{2,S}(\cZ)$.

\paragraph{From quadratic chaos to trace complexity.}
\begin{lemma}[Threshold-chaos reduction]
\label{lem:main-chaos}
Let $\cT\subseteq2^{[n]}$ and
$Q_T(\eps)=n^{-1}\sum_{i<j}\eps_i\eps_j\mathbf1_{\{i\in T\}}\mathbf1_{\{j\in T\}}$.
There is a universal $C_{\rm ch}>0$ such that
\begin{equation}
 \E_\eps\sup_{T\in\cT}|Q_T(\eps)|
 \le C_{\rm ch}\,n\rho_n(\cT)^2.
 \label{eq:main-chaos}
\end{equation}
\end{lemma}
Since $Q_T=((\sum_{i\in T}\eps_i)^2-|T|)/(2n)$, the lemma controls
centered squared trace sums by first-order trace complexity. Applying it
inside the threshold integrals gives the following bound; proofs are in
\Cref{proof:main-chaos,proof:main-adaptive}.

\begin{theorem}[Uniform threshold control under same-sample selection]
\label{thm:main-adaptive}
For $q\in\{1,2\}$, there are universal constants $C_q,C_q'>0$ such that
\begin{align}
 \Rad_S(\cF_{\cZ,B}^{(1)})
 &\le B\left[\frac{\widehat M^*_{1,S}(\cZ)}n+C_1\widehat A_{1,S}(\cZ)\rho_{1,S}(\cZ)^2\right]^{1/2}
 \le C_1'B\sqrt{\widehat A_{1,S}(\cZ)}\,\rho_{1,S}(\cZ),
 \label{eq:main-q1-adaptive}\\
 \Rad_S(\cF_{\cZ,B}^{(2)})
 &\le B\left[\frac{\widehat M^*_{2,S}(\cZ)}n+C_2\widehat W_{2,S}(\cZ)\rho_{2,S}(\cZ)^2\right]^{1/2}
 \le C_2'B\sqrt{\widehat W_{2,S}(\cZ)}\,\rho_{2,S}(\cZ).
 \label{eq:main-q2-adaptive}
\end{align}
\end{theorem}
The first inequalities separate diagonal activation mass from the selection
cost, factored into threshold scale and trace complexity. Absorbing the
diagonal gives the shorter second bounds, which may be looser when activation
mass is much smaller than the envelope. RKHS duality is generic; the diagonal,
layer-cake, and Gaussian-projection identities are Brownian-specific.
Other kernels would require an analogous controlled threshold-mixture representation.

The bounds are conditional on the sample and family. For sample-independent
$\cZ$, symmetrization and contraction give population guarantees for bounded
Lipschitz losses. A sample-generated $\cZ(S)$ needs an additional random-class
argument, and the empirical bound must cover all traces available for
selection. Selection-safe root-mass and intrinsic-width refinements appear in
\Cref{prop:main-root-mass,prop:main-intrinsic-width}.

\paragraph{Sharpness of the trace dependence.}
Write $\vc(\cT)$ for the Vapnik--Chervonenkis dimension
\citep{VapnikChervonenkis1971}. The next result concerns worst-case empirical
capacity at the specified trace and envelope level, not statistical minimax
excess risk. Its finite-family restriction does not apply to the preceding upper bounds.

\begin{theorem}[Matching finite-trace and VC dependence]
\label{thm:main-matching}
There are universal constants $0<c<C$ such that, for every $A,B>0$,
$2\le N\le2^n$, $1\le V\le n$, and $q\in\{1,2\}$, the following holds.
Each supremum ranges over finite families $\cZ$ with common finite output
dimension $m$ (which may vary between families), subject to
$\widehat A_{1,S}(\cZ)\le A$ for $q=1$ and
$\widehat W_{2,S}(\cZ)\le A$ for $q=2$; the $q=1$ states are nonnegative.
The upper directions hold for every sample $S$. For pairwise distinct inputs
$x_1,\ldots,x_n\in\R^d$, the matching lower directions also hold, and hence
\begin{align}
 \sup_{\substack{\cZ:\,|\cT_{q,S}(\cZ)|\le N}}
 \Rad_S(\cF_{\cZ,B}^{(q)})
 &\asymp B\sqrt A\min\left\{1,\sqrt{\frac{\log N}{n}}\right\},
 \label{eq:main-match-card}\\
 \sup_{\substack{\cZ:\,\vc(\cT_{q,S}(\cZ))\le V}}
 \Rad_S(\cF_{\cZ,B}^{(q)})
 &\asymp B\sqrt A\min\left\{1,\sqrt{\frac Vn}\right\}.
 \label{eq:main-match-vc}
\end{align}
\end{theorem}
The saturation follows from $\rho_n(\cT)\le1$. For sharpness, all unions of
$k$ disjoint blocks give $2^k$ traces, VC dimension $k$, and
$\rho_n(\cT)\asymp\min\{1,\sqrt{k/n}\}$; take $k\asymp\log N$ or
$k\asymp V$. Scalar ReLU realizations apply to both kernels, but may depend
on the sample and require width of order $n$. Thus the result does not claim
simultaneous sharpness under all architectural constraints.
See \Cref{proof:main-matching,cor:main-trace-entropy,prop:main-indicator}.

\paragraph{From realized to architecture-level control.}
Layer norms bound the realized scale via \Cref{prop:main-propagation}; ReLU
VC bounds control the traces through depth and parameter count. For $q=2$,
the projection contributes $m$ parameters, and the scalar threshold is
absorbed into the universal constant.

\begin{corollary}[Explicit ReLU-family specialization]
\label{cor:main-explicit}
Let $\Theta$ index a depth-$L$ ReLU representation family with final width $m$,
and set $\cZ_\Theta=\{z_L^\theta:\theta\in\Theta\}$. Assume at most $P$
scalar parameters, sample input radius $R_{0,q}=\max_i\|x_i\|_q$, and uniform
bounds $\|W_\ell\|_{q\to q}\le S_{\ell,q}$ and
$\|b_\ell\|_q\le B_{\ell,q}$. Set
\[
 R_{\Theta,q}=R_{0,q}\prod_{\ell=1}^LS_{\ell,q}
 +\sum_{r=1}^LB_{r,q}\prod_{\ell=r+1}^LS_{\ell,q},
\]
with empty products equal to one. Let $V_\Theta$ bound the VC dimensions of
both realized trace systems; one may take
$V_\Theta=C_{\rm net}(P+m)(L+1)\log(e(P+m))$ for a universal
$C_{\rm net}>0$. Then
\begin{align}
 \Rad_S(\cF_{\cZ_\Theta,B}^{(1)})
 &\lesssim B\min\left\{\sqrt{R_{\Theta,1}},
 \sqrt{\min\{m,n\}R_{\Theta,1}}\min\{1,\sqrt{V_\Theta/n}\}\right\},
 \label{eq:main-explicit-q1}\\
 \Rad_S(\cF_{\cZ_\Theta,B}^{(2)})
 &\lesssim B\min\left\{\sqrt{R_{\Theta,2}},
 \sqrt{R_{\Theta,2}}m^{1/4}\min\{1,\sqrt{V_\Theta/n}\}\right\}.
 \label{eq:main-explicit-q2}
\end{align}
If all states realized across the family and sample lie in one common
$r$-dimensional subspace, $m^{1/4}$ can be replaced by $r^{1/4}$.
\end{corollary}
Indeed, $\norm{z_L^\theta(x_i)}_q\le R_{\Theta,q}$ uniformly, so
$\widehat M^*_{q,S}(\cZ_\Theta)\le R_{\Theta,q}$ and
\[
 \widehat A_{1,S}(\cZ_\Theta)\le\min\{m,n\}R_{\Theta,1},\qquad
 \widehat W_{2,S}(\cZ_\Theta)\lesssim\sqrt m\,R_{\Theta,2}.
\]
For $q=1$, bound coordinate maxima individually or by sample sums. For $q=2$,
Gaussian projection gives $\sqrt m$ (or $\sqrt r$ in the common subspace).
Combine these with $\rho_{q,S}(\cZ_\Theta)\lesssim\min\{1,\sqrt{V_\Theta/n}\}$;
see \Cref{proof:main-explicit}. The first branch is the selection-safe
root-mass bound; the second is useful when $V_\Theta\ll n$.
These architectural bounds need no terminal-state inspection; the realized
bounds may be sharper. Representer, finite-realization, and recursive-chain
results are in \Cref{app:auxiliary-brownian,sec:brownian-chains}.

\section{Controlled Experiments}
\label{sec:experiments}

We use three complementary protocols. Exact finite-sample constructions evaluate the theorem endpoints and trace-rate formulas without optimization. Label-dependent one-step rank-one ReLU adapters form nested same-sample families on frozen ResNet-18 features; every candidate is positively rescaled to the same empirical activation mass before complexity is evaluated. Finally, matched linear, stationary-kernel, and Brownian heads are compared on frozen ImageNet features. Complete protocols, all six transfer settings, raw summaries, and secondary diagnostics are in \Cref{app:theory-core-experiments,app:same-sample-learned,app:frozen-transfer,app:experiments}. The exact formulas attain both endpoints and recover fitted slopes $0.979$ for $\sqrt{\log N/n}$ and $0.980$ for $\sqrt{V/n}$ ($R^2=0.9998$ in both cases).

\begin{figure}[t]
\centering
\includegraphics[width=0.846154\linewidth]{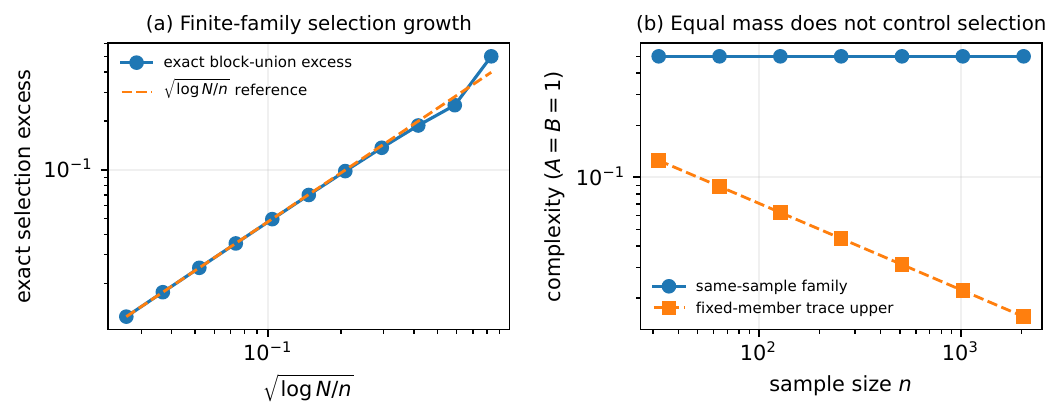}
\caption{Exact finite-sample constructions supporting the selection theorem. Panel (a) evaluates the closed-form finite-family selection excess. Panel (b) keeps every representation at the same activation mass while the same-sample family complexity remains constant and each fixed-member trace upper decays with $n$. These are deterministic evaluations, not independent stochastic experiments.}
\label{fig:main-theory-core}
\end{figure}

\begin{figure}[t]
\centering
\includegraphics[width=0.846154\linewidth]{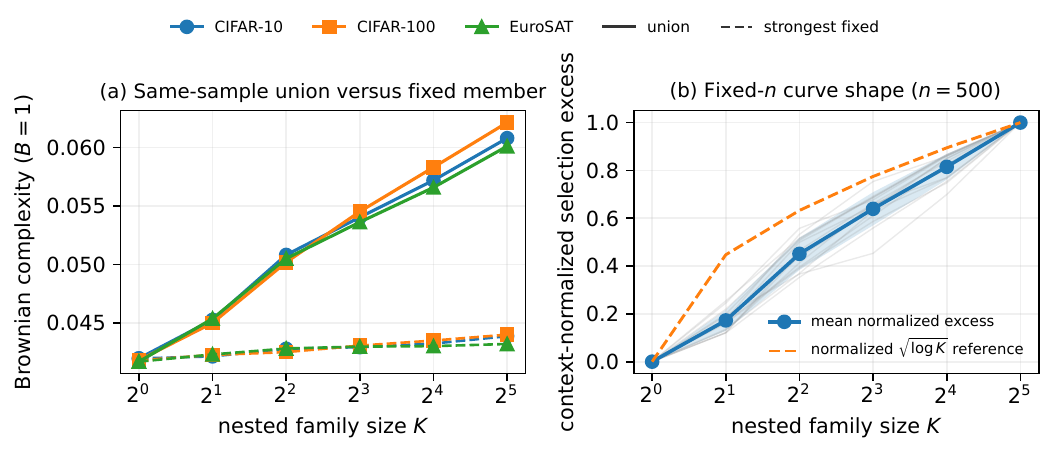}
\caption{Same-sample learned representation families at the fixed sample size $n=500$. Panel (a) compares the Brownian union complexity with the strongest fixed member after all candidates are rescaled to identical activation mass. Panel (b) shows the context-normalized union-minus-fixed excess; the dashed curve is a descriptive normalized $\sqrt{\log K}$ reference.}
\label{fig:main-learned-selection}
\end{figure}

Across three datasets and five paired seeds, the largest predeclared family increases empirical Brownian union complexity by a factor $1.40$ over its strongest fixed member (minimum across the 15 contexts: $1.338$), despite identical activation mass. At the fixed sample size $n=500$, the context-normalized selection excess grows consistently with a $\sqrt{\log K}$ curve across the nested families. The dependence on $n$ is established by the separate exact constructions, not by this fixed-$n$ study. This is a controlled selection experiment rather than an end-to-end benchmark.

Across the six dataset--budget combinations, the L\'evy--Brownian head ranks first by mean accuracy on four tasks and second on two, with average rank $1.33$. Its gains over linear ridge are modest but positive in all six settings, ranging from $0.65$ to $1.90$ percentage points; it is no more than $0.06$ percentage points below the best observed head in any setting. \Cref{fig:main-frozen-transfer} shows the paired differences across all six settings, and \Cref{tab:main-frozen-transfer} reports the $n_{\rm lab}=1000$ results. \Cref{app:frozen-transfer} gives the protocol and both budgets in \Cref{tab:app-frozen-transfer-full}. These frozen-representation experiments do not establish end-to-end superiority or remove the selection term required when the representation is selected using the same sample that trains the head.

\begin{figure}[t]
\centering
\includegraphics[width=0.846154\linewidth]{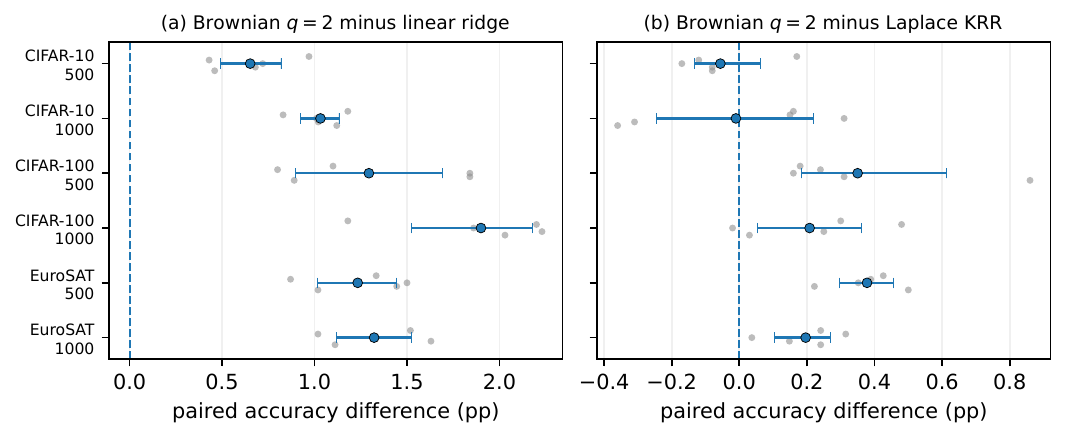}
\caption{Paired frozen-transfer differences for the L\'evy--Brownian head. Small points are the five seed-level paired differences; the large point is their mean with a descriptive paired percentile-bootstrap interval. All methods share the same frozen features and labeled splits. No multiplicity-corrected significance claim is made.}
\label{fig:main-frozen-transfer}
\end{figure}

\begin{table}[t]
\caption{Frozen ImageNet-pretrained ResNet-18 transfer with $n_{\rm lab}=1000$ labels. Values are mean $\pm$ standard deviation over five paired seeds; bold marks the largest observed mean and is not a significance claim. Ridge regularization is selected by training-set GCV. Test labels are never used for fitting or model selection and are accessed only for final evaluation.}
\label{tab:main-frozen-transfer}
\centering
\small
\begin{tabular}{lccc}
\toprule
Method & CIFAR-10 & CIFAR-100 & EuroSAT \\
\midrule
Linear ridge & $79.63\pm0.33$ & $45.41\pm0.75$ & $90.58\pm0.42$ \\
RBF KRR & $80.53\pm0.37$ & $46.55\pm0.56$ & $91.49\pm0.33$ \\
Laplace KRR & $\mathbf{80.68\pm0.45}$ & $47.10\pm0.43$ & $91.71\pm0.30$ \\
Brownian $q=1$ & $79.67\pm0.53$ & $45.68\pm0.54$ & $90.58\pm0.57$ \\
Brownian $q=2$ & $80.67\pm0.23$ & $\mathbf{47.31\pm0.48}$ & $\mathbf{91.90\pm0.30}$ \\
\bottomrule
\end{tabular}
\end{table}

\section{Limitations and Conclusion}
\label{sec:discussion}

Distance-induced Brownian geometry separates two sources of empirical capacity:
the activation mass of a fixed representation and the threshold-trace
complexity of selecting that representation on the same sample. For a fixed
representation, activation mass gives the sharp universal scale of the exact
dual complexity. Under same-sample selection, the representation supremum can
align threshold traces with the Rademacher signs, and the adaptive bound
captures this additional cost. Contraction and induced operator norms extend
the analysis to rectangular and rank-deficient ReLU representations. The resulting
architecture-level bound is worst-case, while the realized trace and envelope
bounds retain the geometry of the candidate family.

The experiments isolate these mechanisms. Exact constructions recover the
fixed-representation endpoints and selection rates, while the adapter study
exhibits the union gap at fixed activation mass. The frozen-transfer study
demonstrates the practical feasibility of Brownian heads, but neither measures
the adaptive selection penalty nor establishes predictive superiority.

The scope is limited to Brownian terminal geometries rather than arbitrary
heads or a complete theory of end-to-end learning. The matching results concern
worst-case empirical capacity, not minimax excess risk, and do not
simultaneously saturate all architectural constraints. Moreover, the empirical
study uses five seeds, one frozen architecture, and three datasets. Natural
directions are to localize the trace cost around a trained representation and
to analyze covariance-shaped or recursive Brownian geometries learned from the
same data.

\FloatBarrier
\label{page:main-end}
\clearpage
\section*{AI Use Statement}
Generative AI tools were used for literature discovery and comparison; feedback on problem framing, mathematical claims, proof writing and revision, experimental design, research-code implementation and editing, interpretation of experimental results, scientific figures, language editing, and \LaTeX{} formatting. The authors take full responsibility for the final content of this work.

\section*{Reproducibility Statement}
Complete proofs of all main results and the auxiliary results used by them are provided in Appendices~A and~B. Appendix~C documents the exact finite-sample constructions and experimental protocols, including preprocessing, candidate generation, seed counts, paired-split procedures, hyperparameter grids, complete transfer tables, and secondary diagnostics. The source package contains all \LaTeX{} and figure files required to compile the paper. Experimental code and machine-readable records, including individual random seeds and split/corruption hashes, are separate from this manuscript project.

\bibliographystyle{plainnat}
\bibliography{references}

@article{VapnikChervonenkis1971,
  author  = {Vapnik, V. N. and Chervonenkis, A. Ya.},
  title   = {On the Uniform Convergence of Relative Frequencies of Events
             to Their Probabilities},
  journal = {Theory of Probability and Its Applications},
  volume  = {16},
  number  = {2},
  pages   = {264--280},
  year    = {1971},
  doi     = {10.1137/1116025}
}

@inproceedings{Foster2019,
  author    = {Foster, Dylan J. and Greenberg, Spencer and Kale, Satyen
               and Luo, Haipeng and Mohri, Mehryar and Sridharan, Karthik},
  title     = {Hypothesis Set Stability and Generalization},
  booktitle = {Advances in Neural Information Processing Systems},
  volume    = {32},
  pages     = {6729--6739},
  year      = {2019},
  publisher = {Curran Associates, Inc.},
  url       = {https://proceedings.neurips.cc/paper/2019/hash/300d1539c3b6aa1793b5678b857732cf-Abstract.html}
}

@inproceedings{Sachs2023,
  author    = {Sachs, Sarah and van Erven, Tim and Hodgkinson, Liam
               and Khanna, Rajiv and {\c{S}}im{\c{s}}ekli, Umut},
  title     = {Generalization Guarantees via Algorithm-Dependent
               {R}ademacher Complexity},
  booktitle = {Proceedings of the Thirty-Sixth Conference on Learning Theory},
  series    = {Proceedings of Machine Learning Research},
  volume    = {195},
  pages     = {4863--4880},
  year      = {2023},
  publisher = {PMLR},
  url       = {https://proceedings.mlr.press/v195/sachs23a.html}
}

@article{Sejdinovic2013,
  author  = {Sejdinovic, Dino and Sriperumbudur, Bharath K. and Gretton, Arthur and Fukumizu, Kenji},
  title   = {Equivalence of Distance-Based and {RKHS}-Based Statistics in Hypothesis Testing},
  journal = {The Annals of Statistics},
  volume  = {41},
  number  = {5},
  pages   = {2263--2291},
  year    = {2013},
  doi     = {10.1214/13-AOS1140}
}

@article{YingCampbell2010,
  author  = {Ying, Yiming and Campbell, Colin},
  title   = {Rademacher Chaos Complexities for Learning the Kernel Problem},
  journal = {Neural Computation},
  volume  = {22},
  number  = {11},
  pages   = {2858--2886},
  year    = {2010},
  doi     = {10.1162/NECO_a_00028}
}

@inproceedings{CortesMohriRostamizadeh2010,
  author    = {Cortes, Corinna and Mohri, Mehryar and Rostamizadeh, Afshin},
  title     = {Generalization Bounds for Learning Kernels},
  booktitle = {Proceedings of the 27th International Conference on Machine Learning},
  pages     = {247--254},
  year      = {2010}
}

@inproceedings{ZhangZhang2023,
  author    = {Zhang, Yifan and Zhang, Min-Ling},
  title     = {Nearly-Tight Bounds for Deep Kernel Learning},
  booktitle = {Proceedings of the 40th International Conference on Machine Learning},
  series    = {Proceedings of Machine Learning Research},
  volume    = {202},
  pages     = {41861--41879},
  year      = {2023}
}

@article{Dupuis2024,
  author  = {Dupuis, Benjamin and Viallard, Paul and Deligiannidis, George and Simsekli, Umut},
  title   = {Uniform Generalization Bounds on Data-Dependent Hypothesis Sets via {PAC}-Bayesian Theory on Random Sets},
  journal = {Journal of Machine Learning Research},
  volume  = {25},
  number  = {409},
  pages   = {1--55},
  year    = {2024}
}

@misc{Hashimoto2026,
  author       = {Hashimoto, Yuka and Sonoda, Sho and Ishikawa, Isao and Ikeda, Masahiro},
  title        = {Why High-Rank Neural Networks Generalize?: An Algebraic Framework with {RKHS}s},
  howpublished = {arXiv preprint arXiv:2509.21895},
  year         = {2026}
}

@article{BartlettMendelson2002,
  author  = {Bartlett, Peter L. and Mendelson, Shahar},
  title   = {Rademacher and Gaussian Complexities: Risk Bounds and Structural Results},
  journal = {Journal of Machine Learning Research},
  volume  = {3},
  pages   = {463--482},
  year    = {2002}
}

@book{Mohri2018,
  author    = {Mohri, Mehryar and Rostamizadeh, Afshin and Talwalkar, Ameet},
  title     = {Foundations of Machine Learning},
  edition   = {2},
  publisher = {MIT Press},
  year      = {2018}
}

@book{LedouxTalagrand1991,
  author    = {Ledoux, Michel and Talagrand, Michel},
  title     = {Probability in Banach Spaces: Isoperimetry and Processes},
  publisher = {Springer},
  year      = {1991}
}

@article{BartlettHarvey2019,
  author  = {Bartlett, Peter L. and Harvey, Nick and Liaw, Christopher and Mehrabian, Abbas},
  title   = {Nearly-Tight {VC}-Dimension and Pseudodimension Bounds for Piecewise Linear Neural Networks},
  journal = {Journal of Machine Learning Research},
  volume  = {20},
  number  = {63},
  pages   = {1--17},
  year    = {2019}
}

@inproceedings{Hashimoto2024,
  author    = {Hashimoto, Yuka and Sonoda, Sho and Ishikawa, Isao and Nitanda, Atsushi and Suzuki, Taiji},
  title     = {Koopman-Based Generalization Bound: New Aspect for Full-Rank Weights},
  booktitle = {International Conference on Learning Representations},
  year      = {2024}
}

@misc{MohammadigohariBKL2026,
  author       = {Mohammadigohari, Mahdi and Di Fatta, Giuseppe and Nicosia, Giuseppe and Pardalos, Panos M.},
  title        = {Brownian Kernel Ladders},
  howpublished = {arXiv preprint arXiv:2606.15812},
  year         = {2026}
}

@article{FollainBach2025,
  author  = {Follain, Bertille and Bach, Francis},
  title   = {Enhanced Feature Learning via Regularisation: Integrating Neural Networks and Kernel Methods},
  journal = {Journal of Machine Learning Research},
  volume  = {26},
  number  = {172},
  pages   = {1--56},
  year    = {2025}
}

@inproceedings{Chen2023,
  author    = {Chen, Zhengdao},
  title     = {Multi-Layer Neural Networks as Trainable Ladders of Hilbert Spaces},
  booktitle = {Proceedings of the International Conference on Machine Learning},
  volume    = {202},
  pages     = {4294--4329},
  year      = {2023}
}

@article{Pedregosa2011,
  author  = {Pedregosa, Fabian and others},
  title   = {Scikit-Learn: Machine Learning in Python},
  journal = {Journal of Machine Learning Research},
  volume  = {12},
  pages   = {2825--2830},
  year    = {2011}
}

@inproceedings{Paszke2019,
  author    = {Paszke, Adam and others},
  title     = {{PyTorch}: An Imperative Style, High-Performance Deep Learning Library},
  booktitle = {Advances in Neural Information Processing Systems},
  volume    = {32},
  year      = {2019}
}

@article{Haagerup1981,
  author  = {Haagerup, Uffe},
  title   = {The Best Constants in the Khintchine Inequality},
  journal = {Studia Mathematica},
  volume  = {70},
  number  = {3},
  pages   = {231--283},
  year    = {1981}
}

@article{Haussler1995,
  author  = {Haussler, David},
  title   = {Sphere Packing Numbers for Subsets of the Boolean $n$-Cube with Bounded Vapnik--Chervonenkis Dimension},
  journal = {Journal of Combinatorial Theory, Series A},
  volume  = {69},
  number  = {2},
  pages   = {217--232},
  year    = {1995}
}

@misc{HeeringaSpekBrune2025,
  author       = {Heeringa, Tjeerd Jan and Spek, Len and Brune, Christoph},
  title        = {Deep Networks are Reproducing Kernel Chains},
  howpublished = {arXiv preprint arXiv:2501.03697},
  year         = {2025}
}

@article{LiuHuangGongYangLi2020,
  author  = {Liu, Fanghui and Huang, Xiaolin and Gong, Chen and Yang, Jie and Li, Li},
  title   = {Learning Data-Adaptive Non-Parametric Kernels},
  journal = {Journal of Machine Learning Research},
  volume  = {21},
  number  = {208},
  pages   = {1--39},
  year    = {2020}
}

@inproceedings{He2016DeepResidual,
  author    = {Kaiming He and Xiangyu Zhang and Shaoqing Ren and Jian Sun},
  title     = {Deep Residual Learning for Image Recognition},
  booktitle = {Proceedings of the IEEE Conference on Computer Vision and Pattern Recognition},
  pages     = {770--778},
  year      = {2016},
  doi       = {10.1109/CVPR.2016.90}
}

@techreport{Krizhevsky2009CIFAR,
  author      = {Alex Krizhevsky},
  title       = {Learning Multiple Layers of Features from Tiny Images},
  institution = {University of Toronto},
  year        = {2009}
}

@article{Helber2019EuroSAT,
  author  = {Patrick Helber and Benjamin Bischke and Andreas Dengel and Damian Borth},
  title   = {{EuroSAT}: A Novel Dataset and Deep Learning Benchmark for Land Use and Land Cover Classification},
  journal = {IEEE Journal of Selected Topics in Applied Earth Observations and Remote Sensing},
  volume  = {12},
  number  = {7},
  pages   = {2217--2226},
  year    = {2019},
  doi     = {10.1109/JSTARS.2019.2918242}
}

\clearpage
\appendix
\phantomsection
\section*{Appendix Map}
\addcontentsline{toc}{section}{Appendix map}

The appendix separates proofs required by the main paper from optional extensions. \Cref{app:main-proofs} is self-contained and follows the seven retained main results in order, including the pairwise/one-point propagation statement beside its proof. \Cref{app:auxiliary-brownian} collects additional terminal-head statements, each immediately followed by its proof. \Cref{app:experiments} supplies the complete experimental protocols and secondary figures. \Cref{sec:brownian-chains} is an independent recursive-chain and BKL extension and is not used by the main terminal-head theorems.

\begin{table}[H]
\centering
\small
\caption{Role of each appendix block.}
\label{tab:appendix-guide}
\begin{tabular}{@{}>{\raggedright\arraybackslash}p{0.42in}>{\raggedright\arraybackslash}p{2.20in}>{\raggedright\arraybackslash}p{2.48in}@{}}
\toprule
Block & Contents & Connection to the main paper \\
\midrule
A & Proofs and proof tools for canonical geometry, ReLU propagation, the fixed-head theorem, threshold chaos, same-sample bounds, matching trace dependence, and the explicit network corollary & Direct proof destination for every formal result retained in the nine-page paper \\
B & Gram, centered, spectral, covariance, root-mass, width, trace, representer, ridge, finite-realization, vector-valued, generalization, and rescaling results & Optional refinements and computational consequences for the terminal-head analysis \\
C & Exact theorem-aligned diagnostics, learned same-sample families, frozen-feature transfer, digits checks, and reproducibility details & Full evidence and protocols supporting \Cref{sec:experiments} \\
D & Fixed-measure BKL links, recursive Brownian chains, dyadic and affine propagation, and sample compression & Optional independent extension; not an ingredient of the main terminal-head theorems \\
\bottomrule
\end{tabular}
\end{table}

\begin{multicols}{2}
\small
\noindent\textbf{A. Main proofs.} Canonical geometry; Hilbert-valued Khintchine tool; fixed-head activation mass; rectangular ReLU propagation and the architecture-level bound; threshold chaos; uniform same-sample control; matching finite-trace and VC dependence; explicit ReLU specialization.

\columnbreak
\noindent\textbf{B--D. Supplementary material.} Additional terminal-head results; supplementary experiments and reproducibility; optional recursive Brownian-chain and BKL theory.
\end{multicols}

\section{Proofs and Proof Tools for the Main Results}
\label{app:main-proofs}\label{sec:proofs}
\makeatletter
\newcommand{\prooflabelalias}[1]{\ltx@label{{#1}}}
\makeatother
This section is self-contained and follows the order of the retained main-text results. It contains no forward reference to later appendices. All Hilbert spaces are real. Samples have $n\ge1$, head radii satisfy $B\ge0$, and representation families used in a union are nonempty. Because every trace family is a subset of the finite set $2^{[n]}$, every trace supremum is an ordinary maximum. For finite representation families all remaining suprema are measurable. For unrestricted families, every displayed expectation of a representation supremum means outer expectation; the inequalities are first proved for measurable finite subfamilies and pass to the stated outer quantities by monotonicity. The representation family is conditioned on the observed sample and is fixed before any auxiliary Rademacher or Gaussian variable is introduced.

\subsection{Proof of \texorpdfstring{\Cref{prop:main-isometry}}{canonical geometry}}
\label{proof:main-isometry}
\begin{proof}
We first verify that the two displayed kernels are positive definite, so their canonical RKHS feature maps are well defined.  For $s\in\R$, define the signed-threshold feature
\[
 \phi_{\rm B}(s)(\sigma,r)
 =\begin{cases}
   \1_{(0,s]}(r),&\sigma=+1,\ s>0,\\
   \1_{(0,-s]}(r),&\sigma=-1,\ s<0,\\
   0,&\text{otherwise},
  \end{cases}
\]
in $L^2(\{-1,+1\}\times\R_+,\#\otimes\dd r)$.  If $s$ and $t$ have opposite signs, the two features have disjoint sign coordinates and their inner product is zero; this equals
$(|s|+|t|-|s-t|)/2$.  If they have the same sign, their inner product is
$\min\{|s|,|t|\}$, which is again
$(|s|+|t|-|s-t|)/2$.  Thus
\begin{equation}
 \inner{\phi_{\rm B}(s)}{\phi_{\rm B}(t)}
 =\kb(s,t):=\frac{|s|+|t|-|s-t|}{2},
 \qquad \|\phi_{\rm B}(s)\|^2=|s|.
 \label{eq:main-proof-scalar-brownian-feature}
\end{equation}
In particular, $\kb$ is positive definite.

For $q=1$, the identity $\|u\|_1=\sum_{j=1}^m|u_j|$ gives
\[
 K_1(u,v)=\sum_{j=1}^m\kb(u_j,v_j).
\]
Hence $u\mapsto(\phi_{\rm B}(u_j))_{j=1}^m$ is a feature map in the Hilbert direct sum of $m$ copies of the scalar feature space, and $K_1$ is positive definite.

For $q=2$, let $G\sim\mathcal N(0,I_m)$ and put $c_0=\sqrt{\pi/2}$.  Since
$\E|\inner{G}{w}|=\sqrt{2/\pi}\|w\|_2$, applying the scalar Brownian formula to the three vectors $u,v,u-v$ gives
\begin{align}
 c_0\E_G\kb(\inner{G}{u},\inner{G}{v})
 &=\frac12\left(
 c_0\E|\inner{G}{u}|+c_0\E|\inner{G}{v}|
 -c_0\E|\inner{G}{u-v}|
 \right)\\
 &=\frac{\|u\|_2+\|v\|_2-\|u-v\|_2}{2}
 =K_2(u,v).
 \label{eq:main-proof-gaussian-brownian-feature}
\end{align}
The integrand is measurable and is bounded in absolute value by
$\min\{|\inner{G}{u}|,|\inner{G}{v}|\}$, which is integrable.  Therefore
$u\mapsto\sqrt{c_0}\,\phi_{\rm B}(\inner{\cdot}{u})$ is a feature map in the corresponding Gaussian $L^2$ space.  This proves positive definiteness of $K_2$, including singular and zero configurations.

Let $\Phi_q(u)=K_q(\cdot,u)$ be the canonical feature map in the RKHS of $K_q$.  The reproducing identity now yields
\[
 \|\Phi_q(u)\|_{\cH_q}^2=K_q(u,u)=\|u\|_q.
\]
Expanding the squared Hilbert distance and substituting the definition of $K_q$ gives
\begin{align*}
 \|\Phi_q(u)-\Phi_q(v)\|_{\cH_q}^2
 &=K_q(u,u)+K_q(v,v)-2K_q(u,v)\\
 &=\|u-v\|_q.
\end{align*}
Both identities also hold when one or both vectors are zero, because $K_q(0,\cdot)=0$.
\end{proof}

\subsection{A Hilbert-valued Khintchine proof tool}
\label{app:khintchine-tool}
The fixed-head theorem uses the following inequality for a finite vector family in an arbitrary real Hilbert space.

\begin{lemma}[Hilbert-valued Khintchine inequality]
\label{lem:main-khintchine}
For vectors $v_1,\ldots,v_n$ in a real Hilbert space,
\begin{align}
 \frac1{\sqrt2}\left(\sum_i\norm{v_i}^2\right)^{1/2}
 \le \E_\eps\norm{\sum_i\eps_i v_i}
 \le \left(\sum_i\norm{v_i}^2\right)^{1/2}.
 \label{eq:main-khintchine}
\end{align}
Both constants are optimal over all finite Hilbert-space configurations.
\end{lemma}
\phantomsection\label{proof:main-khintchine}
\begin{proof}
Set $S_\eps=\sum_i\eps_i v_i$. If every $v_i$ is zero, all three quantities in \eqref{eq:main-khintchine} vanish, so assume below that $\sum_i\|v_i\|^2>0$.
For the upper bound, Jensen's inequality and independence of the signs give
\begin{align}
  \E_\eps\norm{S_\eps}_H
  &\le \left(\E_\eps\norm{S_\eps}_H^2\right)^{1/2}\\
  &=\left(
    \sum_{i,j}\E(\eps_i\eps_j)\inner{v_i}{v_j}_H
  \right)^{1/2}
  =\left(\sum_i\norm{v_i}_H^2\right)^{1/2}.
\end{align}

For the lower bound, it is enough to work in the finite-dimensional space
$H_0=\operatorname{span}\{v_1,\dots,v_n\}$.
Let $G$ be a standard Gaussian vector in $H_0$.
For every fixed $x\in H_0$,
$\E_G\abs{\inner{G}{x}_H}=\sqrt{2/\pi}\norm{x}_H$; therefore Tonelli's theorem yields
\begin{align}
  \E_\eps\norm{S_\eps}_H
  &=\sqrt{\frac\pi2}\,
    \E_G\E_\eps\abs{\sum_i\eps_i\inner{G}{v_i}_H}.
  \label{eq:hilbert-khin-gaussian-step}
\end{align}
The sharp scalar Khintchine inequality at exponent one~\citep{Haagerup1981} states that
\[
  \E_\eps\abs{\sum_i\eps_i a_i}
  \ge \frac1{\sqrt2}\left(\sum_i a_i^2\right)^{1/2}.
\]
Applying it conditionally on $G$ in \eqref{eq:hilbert-khin-gaussian-step} gives
\begin{align}
  \E_\eps\norm{S_\eps}_H
  \ge \frac{\sqrt\pi}{2}\,
  \E_G\left(\sum_i\inner{G}{v_i}_H^2\right)^{1/2}.
  \label{eq:hilbert-khin-after-scalar}
\end{align}
Let $A=\sum_i v_i\otimes v_i$ on $H_0$, and let
$\lambda_1,\dots,\lambda_r\ge0$ be its nonzero eigenvalues.
If $g_1,\dots,g_r$ are independent standard Gaussians, then
\[
  \sum_i\inner{G}{v_i}_H^2
  \stackrel{d}{=}\sum_{j=1}^r\lambda_j g_j^2,
  \qquad
  \sum_{j=1}^r\lambda_j=\operatorname{tr}(A)=\sum_i\norm{v_i}_H^2.
\]
When the trace is nonzero, put $p_j=\lambda_j/\operatorname{tr}(A)$.
The map
$p\mapsto\E(\sum_jp_jg_j^2)^{1/2}$ is concave on the probability simplex.
Since $p=\sum_jp_je_j$, concavity and
$\E\abs{g_j}=\sqrt{2/\pi}$ imply
\begin{align}
  \E_G\left(\sum_i\inner{G}{v_i}_H^2\right)^{1/2}
  &=\left(\sum_i\norm{v_i}_H^2\right)^{1/2}
    \E\left(\sum_jp_jg_j^2\right)^{1/2}\\
  &\ge
    \sqrt{\frac2\pi}
    \left(\sum_i\norm{v_i}_H^2\right)^{1/2}.
\end{align}
The zero-trace case was separated at the start.
Substitution into \eqref{eq:hilbert-khin-after-scalar} proves the lower bound.

For optimality of $1/\sqrt2$, take $n=2$ and $v_1=v_2\ne0$ in a one-dimensional Hilbert space.
Then $\E\norm{\eps_1v_1+\eps_2v_2}=\norm{v_1}$, whereas
$(\norm{v_1}^2+\norm{v_2}^2)^{1/2}=\sqrt2\norm{v_1}$.
For optimality of the upper constant, take one nonzero vector (or, more generally, pairwise orthogonal vectors); the norm of the signed sum is then deterministic and equal to the right-hand side of
\eqref{eq:main-khintchine}.
\end{proof}

\subsection{Proof of \texorpdfstring{\Cref{thm:main-fixed}}{the fixed-head theorem}}
\label{proof:main-fixed}
\begin{proof}
If $B=0$, the class contains only the zero function and all asserted quantities vanish.  We may therefore keep $B\ge0$ in the formulas without dividing by it.  Put $u_i=z(x_i)$ and $v_i=\Phi_q(u_i)\in\cH_q$.
For a fixed sign vector $\eps$, Hilbert-space duality gives
\begin{align}
 \sup_{\norm{a}_{\cH_q}\le B}
 \sum_{i=1}^n\eps_i\inner{a}{v_i}_{\cH_q}
 &=\sup_{\norm{a}\le B}\inner{a}{\sum_{i=1}^n\eps_i v_i}_{\cH_q}\\
 &=B\norm{\sum_{i=1}^n\eps_i v_i}_{\cH_q}.
\end{align}
Taking expectation and dividing by $n$ proves \eqref{eq:main-exact-rad}.
By \eqref{eq:main-isometry},
\[
 \sum_{i=1}^n\norm{v_i}_{\cH_q}^2
 =\sum_{i=1}^n\norm{u_i}_q
 =n\widehat M_{q,S}(z).
\]
Applying \Cref{lem:main-khintchine} to $(v_i)_{i=1}^n$ and multiplying by $B/n$ gives \eqref{eq:main-two-sided}. If $\widehat M_{q,S}(z)=0$, then every $u_i=0$, hence every $v_i=0$ and the identity gives zero complexity. For $n=1$ and positive mass, the signed norm is deterministic, so only the upper numerical endpoint is attained.

It remains to verify that both constants occur inside the Brownian canonical feature class. Fix a prescribed mass $M>0$ and $n\ge2$. For the lower endpoint, let
$u_1=u_2=(nM/2)e_1$ and $u_i=0$ for $i\ge3$. These states are nonnegative and their empirical activation mass is $M$. Since $\Phi_q(0)=0$,
\[
 \sum_i\eps_i\Phi_q(u_i)
 =(\eps_1+\eps_2)\Phi_q((nM/2)e_1).
\]
Now $\norm{\Phi_q((nM/2)e_1)}=\sqrt{nM/2}$ and
$\E|\eps_1+\eps_2|=1$, so the complexity is
$B\sqrt{M/(2n)}$.

For the upper endpoint, take $u_1=nMe_1$ and $u_i=0$ for $i\ge2$. The signed feature sum has deterministic norm $\sqrt{nM}$, hence the complexity is $B\sqrt{M/n}$. Both configurations have the required mass and attain the stated endpoints.
\end{proof}

\subsection{Rectangular ReLU propagation and the architecture-level bound}
\begin{proposition}[Propagation through rectangular ReLU layers]
\label{prop:main-propagation}
For every $x,x'$ and $q\in\{1,2\}$,
\begin{equation}
 \norm{\Phi_q(z_L(x))-\Phi_q(z_L(x'))}_{\cH_q}^2
 \le \left(\prod_{\ell=1}^L s_{\ell,q}\right)\norm{x-x'}_q.
 \label{eq:main-deep-propagation}
\end{equation}
Moreover,
\begin{equation}
 \norm{z_L(x)}_q
 \le \left(\prod_{\ell=1}^L s_{\ell,q}\right)\norm{x}_q
 +\sum_{r=1}^L\norm{b_r}_q\prod_{\ell=r+1}^L s_{\ell,q}.
 \label{eq:main-affine-envelope}
\end{equation}
\end{proposition}
\phantomsection\label{proof:main-propagation}
\begin{proof}
Fix $x,x'$ and a layer $\ell$.  For real numbers $r,s$,
$|\relu(r)-\relu(s)|\le|r-s|$; summing the coordinate inequalities shows that coordinatewise ReLU is nonexpansive in both $\ell_1$ and $\ell_2$.  With
$u=z_{\ell-1}(x)$ and $v=z_{\ell-1}(x')$, the common bias cancels before this contraction, and therefore
\begin{align}
 \|z_\ell(x)-z_\ell(x')\|_q
 &\le \|W_\ell(u-v)\|_q
 \le s_{\ell,q}\|u-v\|_q.
 \label{eq:one-layer-distance}
\end{align}
Induction on $r$ gives
\[
 \|z_r(x)-z_r(x')\|_q
 \le\left(\prod_{\ell=1}^r s_{\ell,q}\right)\|x-x'\|_q.
\]
At $r=L$, the second identity of \Cref{prop:main-isometry} turns the left side into the squared terminal feature distance and proves
\eqref{eq:main-deep-propagation}.  If a factor $s_{\ell,q}$ is zero, the same induction shows that all later realized differences are zero; no inverse or positive singular-value assumption is used.  For $L=0$, the assertion reduces to the canonical identity with the empty product equal to one.

For the one-point estimate, $\relu(0)=0$ and ReLU nonexpansiveness imply
$\|\relu(w)\|_q\le\|w\|_q$.  Hence
\begin{equation}
 r_\ell:=\|z_\ell(x)\|_q
 \le s_{\ell,q}r_{\ell-1}+\|b_\ell\|_q.
 \label{eq:one-layer-energy}
\end{equation}
We prove the claimed formula by induction.  For $r=1$, it is exactly
\eqref{eq:one-layer-energy}.  If it holds at $r-1$, then
\begin{align*}
 r_r
 &\le s_{r,q}\left[
   \left(\prod_{\ell=1}^{r-1}s_{\ell,q}\right)r_0
   +\sum_{j=1}^{r-1}\|b_j\|_q
     \prod_{\ell=j+1}^{r-1}s_{\ell,q}
 \right]+\|b_r\|_q\\
 &=\left(\prod_{\ell=1}^{r}s_{\ell,q}\right)r_0
 +\sum_{j=1}^{r}\|b_j\|_q
   \prod_{\ell=j+1}^{r}s_{\ell,q}.
\end{align*}
Taking $r=L$ and $r_0=\|x\|_q$ proves
\eqref{eq:main-affine-envelope}, including zero biases, zero matrices, and empty products.
\end{proof}

Averaging \eqref{eq:main-affine-envelope} over $x_1,\ldots,x_n$ bounds
$\widehat M_{q,S}(z_L)$ by the bracket in \eqref{eq:main-architecture-rad2}.
The upper inequality in \Cref{thm:main-fixed} then proves
\Cref{prop:main-propagation2}.

\subsection{Proof of \texorpdfstring{\Cref{lem:main-chaos}}{threshold chaos}}
\label{proof:main-chaos}
\begin{proof}
Because $\varnothing\ne\cT\subseteq2^{[n]}$, the family is finite and every supremum below is a maximum. If $\cT=\{\varnothing\}$, then $Q_T=0$ and $\rho_n(\cT)=0$, so the assertion holds. For $T\subseteq[n]$, put
$S_T(\eps)=\sum_{i\in T}\eps_i$.
The random vector
\[
  X(\eps)=(S_T(\eps))_{T\in\cT}
\]
is a well-defined Rademacher series in the finite-dimensional Banach space
$\ell_\infty(\cT)$, with
$\norm{X(\eps)}_\infty=\sup_{T\in\cT}\abs{S_T(\eps)}$.
Expanding the square gives, for each $T$,
\begin{align}
  S_T(\eps)^2
  =\abs T+2\sum_{1\le i<j\le n}
  \eps_i\eps_j\1_{\{i\in T\}}\1_{\{j\in T\}},
\end{align}
and hence
\begin{align}
  Q_T(\eps)=\frac{S_T(\eps)^2-\abs T}{2n}.
  \label{eq:chaos-square-identity}
\end{align}
Therefore
\begin{align}
  \E_\eps\sup_{T\in\cT}\abs{Q_T(\eps)}
  &\le\frac1{2n}
  \left(
    \E_\eps\sup_{T\in\cT}S_T(\eps)^2
    +\sup_{T\in\cT}\abs T
  \right)\notag\\
  &=\frac1{2n}
  \left(
    \E_\eps\norm{X(\eps)}_\infty^2
    +\sup_{T\in\cT}\abs T
  \right).
  \label{eq:chaos-before-kk}
\end{align}
For every fixed $T$,
$\E_\eps S_T(\eps)^2=\abs T$.
Thus
\[
  \sup_{T\in\cT}\abs T
  =\sup_{T\in\cT}\E S_T^2
  \le\E\sup_{T\in\cT}S_T^2
  =\E\norm{X}_\infty^2.
\]
Substituting this into \eqref{eq:chaos-before-kk} yields
\begin{align}
  \E_\eps\sup_{T\in\cT}\abs{Q_T(\eps)}
  \le\frac1n\E_\eps\norm{X(\eps)}_\infty^2.
  \label{eq:chaos-second-moment}
\end{align}
The $p=1$, $r=2$ Kahane--Khintchine inequality for a finite Rademacher series in an arbitrary Banach space states that a universal $K_{1,2}<\infty$ satisfies~\citep{LedouxTalagrand1991}
\[
  \left(\E\norm{\sum_i\eps_i x_i}^2\right)^{1/2}
  \le K_{1,2}\E\norm{\sum_i\eps_i x_i}
\]
for every finite vector family $(x_i)$. Applying it in $\ell_\infty(\cT)$ to the vectors $x_i=(\1_{\{i\in T\}})_{T\in\cT}$ gives
\[
  \left(\E\norm{X}_\infty^2\right)^{1/2}
  \le K_{1,2}\E\norm{X}_\infty.
\]
By definition,
$\E\norm{X}_\infty=n\rho_n(\cT)$.
Squaring the Kahane--Khintchine estimate and inserting it into
\eqref{eq:chaos-second-moment} proves
\[
  \E\sup_T\abs{Q_T}
  \le K_{1,2}^2 n\rho_n(\cT)^2.
\]
\end{proof}

\subsection{Proof of \texorpdfstring{\Cref{thm:main-adaptive}}{uniform threshold control}}
\label{proof:main-adaptive}
\begin{proof}
The representation family is conditioned on the observed sample before $\eps$ and, in the projected argument, before $G$ are drawn.  The trace systems are nonempty: choosing a threshold above every realized coordinate gives the empty trace in the additive case, while $g=0$ gives the empty trace in the projected case.  They are finite because they are subsets of $2^{[n]}$.  If the relevant envelope is zero, all realized states are zero on the sample and both sides of the corresponding bound vanish; the calculations below cover this case without division.

\medskip\noindent\emph{Additive Brownian head.}
For $z\in\cZ$, abbreviate $z_i=z(x_i)$ and
$K_z(i,j)=K_1(z_i,z_j)$.
Because every $z_i$ is nonnegative, the scalar identity
$\kb(s,t)=\int_0^\infty\1_{\{s\ge r\}}\1_{\{t\ge r\}}\dd r$
on $\R_+$ and additivity of $K_1$ give
\begin{align}
  K_z(i,j)
  =\sum_{r=1}^m\int_0^\infty
  \1_{\{z_r(x_i)\ge t\}}
  \1_{\{z_r(x_j)\ge t\}}\dd t.
  \label{eq:q1-kernel-threshold}
\end{align}
Define
\[
  D_z(\eps)=\frac1n\sum_{1\le i<j\le n}
  \eps_i\eps_jK_z(i,j)
\]
and, for each $(z,r,t)$, let
$T_{z,r,t}=\{i:z_r(x_i)\ge t\}$.
Substituting \eqref{eq:q1-kernel-threshold} and using the definition of
$Q_T$ yields the exact identity
\begin{align}
  D_z(\eps)
  =\sum_{r=1}^m\int_0^\infty Q_{T_{z,r,t}}(\eps)\,\dd t.
  \label{eq:q1-D-layercake}
  \prooflabelalias{eq:main-layer-cake}
\end{align}
For fixed $z$ and $r$, the trace is empty whenever
$t>\max_i z_r(x_i)$.
Hence, for every sign vector,
\begin{align}
  \abs{D_z(\eps)}
  &\le\sum_{r=1}^m
  \max_i z_r(x_i)
  \sup_{T\in\cT_{1,S}(\cZ)}\abs{Q_T(\eps)}\\
  &\le\widehat A_{1,S}(\cZ)
  \sup_{T\in\cT_{1,S}(\cZ)}\abs{Q_T(\eps)}.
\end{align}
Taking first the supremum over $z$ and then expectation, and applying
\Cref{lem:main-chaos}, gives
\begin{align}
  \E_\eps\sup_{z\in\cZ}\abs{D_z(\eps)}
  \le C_{\mathrm{ch}}\widehat A_{1,S}(\cZ)
  n\rho_{1,S}(\cZ)^2.
  \label{eq:q1-chaos-bound}
\end{align}

For each $z$ and $\eps$, the RKHS norm expands as
\begin{align}
\begin{aligned}
  \norm{\sum_{i=1}^n\eps_i\Phi_{1,m}(z_i)}_{\cH_{1,m}}^2
  &=\sum_{i=1}^nK_1(z_i,z_i)
    +2\sum_{i<j}\eps_i\eps_jK_1(z_i,z_j)\\
  &=\sum_{i=1}^n\norm{z_i}_1+2nD_z(\eps).
\end{aligned}
  \label{eq:q1-feature-norm-expansion}
  \prooflabelalias{eq:main-feature-expansion}
\end{align}
Hilbert duality followed by Jensen's inequality now gives
\begin{align}
  \Rad_S(\cF_{\cZ,B}^{(1)})
  &=\frac Bn\E_\eps\sup_{z\in\cZ}
    \norm{\sum_i\eps_i\Phi_{1,m}(z_i)}\\
  &\le\frac Bn
  \left[
    \E_\eps\sup_{z\in\cZ}
    \norm{\sum_i\eps_i\Phi_{1,m}(z_i)}^2
  \right]^{1/2}\\
  &\le B\left[
    \frac{\widehat M^*_{1,S}(\cZ)}{n}
    +\frac2n\E_\eps\sup_{z\in\cZ}\abs{D_z(\eps)}
  \right]^{1/2}.
\end{align}
Substitution of \eqref{eq:q1-chaos-bound} proves the first inequality in
\eqref{eq:main-q1-adaptive} with $C_1=2C_{\mathrm{ch}}$.

It remains to absorb the diagonal term into the threshold term.
Let
\[
  k_*=\max_{T\in\cT_{1,S}(\cZ)}\abs T.
\]
Choose $T_*$ attaining this maximum.
The sharp scalar Khintchine lower bound gives
\begin{align}
  n\rho_{1,S}(\cZ)
  \ge\E_\eps\abs{\sum_{i\in T_*}\eps_i}
  \ge\sqrt{\frac{k_*}{2}},
  \qquad\text{so}\qquad
  k_*\le2n^2\rho_{1,S}(\cZ)^2.
  \label{eq:q1-kstar-rho}
\end{align}
For any fixed $z$, the layer-cake formula gives
\begin{align}
  \frac1n\sum_{i=1}^n\norm{z(x_i)}_1
  &=\frac1n\sum_{r=1}^m\int_0^\infty
    \abs{T_{z,r,t}}\,\dd t\\
  &\le\frac{k_*}{n}\sum_{r=1}^m\max_i z_r(x_i)
  \le\frac{k_*}{n}\widehat A_{1,S}(\cZ).
\end{align}
Taking the supremum over $z$ and using \eqref{eq:q1-kstar-rho} yields
\begin{align}
  \frac{\widehat M^*_{1,S}(\cZ)}{n}
  \le2\widehat A_{1,S}(\cZ)\rho_{1,S}(\cZ)^2.
  \label{eq:q1-diagonal-absorption}
\end{align}
Combining this with the first inequality proves the second one, for example with
$C_1'=\sqrt{2+2C_{\mathrm{ch}}}$.

\medskip\noindent\emph{L\'evy--Brownian head.}
Fix $z\in\cZ$, write $z_i=z(x_i)$, and define
\[
  D_z(\eps)=\frac1n\sum_{1\le i<j\le n}
  \eps_i\eps_jK_2(z_i,z_j).
\]
For $G\in\R^m$, let
$a_i(G,z)=\inner{G}{z_i}$ and define the two projected threshold traces
\[
  T_{G,z,t}^{+}=\{i:a_i(G,z)\ge t\},
  \qquad
  T_{G,z,t}^{-}=\{i:-a_i(G,z)\ge t\},
  \qquad t>0.
\]
Both traces belong to $\cT_{2,S}(\cZ)$.
For every threshold trace $T\subseteq[n]$,
$\abs{Q_T(\eps)}\le n^{-1}\binom{\abs T}{2}\le(n-1)/2$.
Consequently, for fixed $G,z,\eps$,
\begin{align}
  &\int_0^\infty
  \left(
    \abs{Q_{T_{G,z,t}^{+}}(\eps)}
    +\abs{Q_{T_{G,z,t}^{-}}(\eps)}
  \right)\dd t\\
  &\qquad\le(n-1)\max_{i\in[n]}\abs{\inner{G}{z_i}}.
\end{align}
The right-hand side is Gaussian integrable because the sample is finite.
For every $s,t\in\R$, the signed scalar identity
\[
 \kb(s,t)=\int_0^\infty\left(
 \mathbf1_{\{s\ge r\}}\mathbf1_{\{t\ge r\}}+
 \mathbf1_{\{s\le-r\}}\mathbf1_{\{t\le-r\}}
 \right)\dd r
\]
holds as follows. If $st\le0$, both products vanish for every $r>0$ and $\kb(s,t)=0$. If $s,t>0$, the first product integrates to $\min\{s,t\}$ and the second vanishes; if $s,t<0$, the second integrates to $\min\{|s|,|t|\}$ and the first vanishes. Zero endpoints are included in the first case. In addition, the Gaussian projection calculation in \eqref{eq:main-proof-gaussian-brownian-feature} gives
\[
 K_2(u,v)=c_0\E_G\left[\kb(\langle G,u\rangle,\langle G,v\rangle)\right].
\]
These are the Gaussian projection and signed-threshold representations of the L\'evy--Brownian kernel. Together with the preceding integrability estimate, they justify Fubini's theorem and give
\begin{align}
  D_z(\eps)
  =c_0\E_G\int_0^\infty
  \left(
    Q_{T_{G,z,t}^{+}}(\eps)
    +Q_{T_{G,z,t}^{-}}(\eps)
  \right)\dd t.
  \label{eq:q2-D-projection}
\end{align}
For fixed $G,z$, the positive integral vanishes above
$\max_i(a_i)_+$ and the negative integral vanishes above
$\max_i(-a_i)_+$.
Therefore, for every $\eps$,
\begin{align}
  \sup_{z\in\cZ}\abs{D_z(\eps)}
  &\le c_0\E_G\sup_{z\in\cZ}
  \int_0^\infty
  \left(
    \abs{Q_{T_{G,z,t}^{+}}(\eps)}
    +\abs{Q_{T_{G,z,t}^{-}}(\eps)}
  \right)\dd t\\
  &\le 2c_0\E_G\sup_{z\in\cZ,\,i\in[n]}
  \abs{\inner{G}{z(x_i)}}
  \sup_{T\in\cT_{2,S}(\cZ)}\abs{Q_T(\eps)}\\
  &=2\widehat W_{2,S}(\cZ)
  \sup_{T\in\cT_{2,S}(\cZ)}\abs{Q_T(\eps)}.
\end{align}
Taking expectation in $\eps$ and applying
\Cref{lem:main-chaos} gives
\begin{align}
  \E_\eps\sup_{z\in\cZ}\abs{D_z(\eps)}
  \le2C_{\mathrm{ch}}\widehat W_{2,S}(\cZ)
  n\rho_{2,S}(\cZ)^2.
  \label{eq:q2-chaos-bound}
\end{align}

As in \eqref{eq:q1-feature-norm-expansion}, the diagonal identity
$K_2(z_i,z_i)=\norm{z_i}_2$ gives
\[
  \norm{\sum_i\eps_i\Phi_{2,m}(z_i)}_{\cH_{2,m}}^2
  =\sum_i\norm{z_i}_2+2nD_z(\eps).
\]
Hilbert duality and Jensen's inequality therefore imply
\begin{align}
  \Rad_S(\cF_{\cZ,B}^{(2)})
  &\le B\left[
    \frac{\widehat M^*_{2,S}(\cZ)}{n}
    +\frac2n\E_\eps\sup_{z\in\cZ}\abs{D_z(\eps)}
  \right]^{1/2}\\
  &\le B\left[
    \frac{\widehat M^*_{2,S}(\cZ)}{n}
    +4C_{\mathrm{ch}}\widehat W_{2,S}(\cZ)
    \rho_{2,S}(\cZ)^2
  \right]^{1/2},
\end{align}
which proves the first inequality in \eqref{eq:main-q2-adaptive} with
$C_2=4C_{\mathrm{ch}}$.

To absorb the diagonal term, let
$k_*=\max_{T\in\cT_{2,S}(\cZ)}\abs T$.
Choose $T_*\in\cT_{2,S}(\cZ)$ with $|T_*|=k_*$. The sharp scalar Khintchine inequality gives
\[
 n\rho_{2,S}(\cZ)
 \ge\E_\eps\left|\sum_{i\in T_*}\eps_i\right|
 \ge\sqrt{k_*/2}.
\]
Squaring and rearranging yields
\begin{align}
  k_*\le2n^2\rho_{2,S}(\cZ)^2.
  \label{eq:q2-kstar-rho}
\end{align}
The Gaussian identity
$\norm{v}_2=c_0\E_G\abs{\inner{G}{v}}$ gives
\begin{align}
  \widehat M^*_{2,S}(\cZ)
  &=c_0\sup_{z\in\cZ}\E_G
  \frac1n\sum_{i=1}^n\abs{a_i(G,z)}.
\end{align}
For fixed $G,z$, the signed layer-cake formula yields
\begin{align}
  \sum_{i=1}^n\abs{a_i(G,z)}
  &=\int_0^\infty\abs{T_{G,z,t}^{+}}\,\dd t
    +\int_0^\infty\abs{T_{G,z,t}^{-}}\,\dd t\\
  &\le k_*\left(
    \max_i(a_i(G,z))_+
    +\max_i(-a_i(G,z))_+
  \right)\\
  &\le2k_*\max_i\abs{a_i(G,z)}.
\end{align}
Taking the supremum over $z$ outside the Gaussian expectation gives
\begin{align}
  \widehat M^*_{2,S}(\cZ)
  \le\frac{2k_*}{n}\widehat W_{2,S}(\cZ).
  \label{eq:q2-mass-width}
\end{align}
Combining \eqref{eq:q2-kstar-rho} and \eqref{eq:q2-mass-width},
\[
  \frac{\widehat M^*_{2,S}(\cZ)}{n}
  \le4\widehat W_{2,S}(\cZ)\rho_{2,S}(\cZ)^2.
\]
This proves the second inequality in \eqref{eq:main-q2-adaptive}, for example with
$C_2'=2\sqrt{1+C_{\mathrm{ch}}}$.
\end{proof}

\subsection{Proof of \texorpdfstring{\Cref{thm:main-matching}}{matching trace dependence}}
\label{proof:main-matching}
\begin{proof}
All logarithms in this proof may be taken in any fixed base, because changing the base only changes universal constants.

\medskip\noindent\emph{Cardinality upper bound.}
Let $\cT\ne\varnothing$ be a trace family and write
$v_T=(\1_{\{i\in T\}})_{i=1}^n$.  The class
$\{v_T,-v_T:T\in\cT\}$ has cardinality at most $2|\cT|$ and every vector has Euclidean norm at most $\sqrt n$.  The finite-class exponential-moment argument (Massart's lemma) gives
\[
 \E_\eta\max_{T\in\cT}|\inner{\eta}{v_T}|
 \le\sqrt{2n\log(2|\cT|)}.
\]
Indeed, for $\lambda>0$,
$\E e^{\lambda\inner{\eta}{v_T}}\le e^{\lambda^2\|v_T\|_2^2/2}\le e^{\lambda^2n/2}$; applying the log-sum-exp bound to the $2|\cT|$ signed vectors and optimizing at
$\lambda=\sqrt{2\log(2|\cT|)/n}$ gives the display.  Division by $n$, together with the deterministic bound by one, yields
\begin{equation}
 \rho_n(\cT)\le
 \min\left\{1,\sqrt{\frac{2\log(2|\cT|)}{n}}\right\}.
 \label{eq:matching-proof-card-upper}
\end{equation}
When $|\cT|\le N$ and $N\ge2$, $\log(2N)\le2\log N$ after a universal adjustment of the logarithm base.  Applying the absorbed form of \Cref{thm:main-adaptive} with $\widehat A_{1,S}(\cZ)\le A$ for $q=1$, or $\widehat W_{2,S}(\cZ)\le A$ for $q=2$, proves the cardinality upper direction.

\medskip\noindent\emph{VC upper bound.}
Every nonempty realized trace system in the theorem contains $\varnothing$.  For a trace class of VC dimension at most $V$, Haussler's packing theorem in the empirical symmetric-difference metric
$d_n(T,T')=(|T\triangle T'|/n)^{1/2}$ gives universal constants $c_1,c_2$ such that
\[
 \log N(u,\cT,d_n)\le c_1V\log(c_2/u),\qquad 0<u\le1.
\]
The empirical Dudley bound for the centered Rademacher process indexed by the indicator vectors then gives
\[
 \frac1n\E_\eta\sup_{T\in\cT}\sum_{i\in T}\eta_i
 \le\frac{c_3}{\sqrt n}\int_0^1
       \sqrt{V\log(c_2/u)}\,\dd u
 \le c_4\sqrt{V/n}.
\]
The integral is finite because the substitution $u=e^{-s}$ produces an exponentially weighted square-root moment.  Since $\varnothing\in\cT$, the one-sided supremum is nonnegative; symmetry of $\eta$ bounds the absolute supremum by twice its expectation.  Combining this estimate with the deterministic cap gives
\begin{equation}
 \rho_n(\cT)\le C_{\rm VC}\min\{1,\sqrt{V/n}\}.
 \label{eq:matching-proof-vc-upper}
\end{equation}
This is the Haussler--Dudley route recorded in
\citet{Haussler1995,BartlettMendelson2002}.  Substitution into
\Cref{thm:main-adaptive} proves the VC upper direction.

\medskip\noindent\emph{Exact scalar Brownian reduction.}
Let $\varnothing\ne\cT\subseteq2^{[n]}$ and prescribe sample values
$z_T(x_i)=A\1_{\{i\in T\}}$.  The output dimension is one and the values are nonnegative.  For $q=1$ and $q=2$ the scalar kernels coincide with $\kb$, while
$\Phi_q(0)=0$ and $\|\Phi_q(A)\|=\sqrt A$.  For every sign vector, Hilbert duality therefore gives
\begin{align}
 \sup_{T\in\cT}\sup_{\|a\|\le B}
 \frac1n\sum_{i=1}^n\eps_i
 \inner{a}{\Phi_q(z_T(x_i))}
 &=\frac{B\sqrt A}{n}
   \sup_{T\in\cT}\left|\sum_{i\in T}\eps_i\right|.
\end{align}
Taking expectation yields
\begin{equation}
 \Rad_S\!\left(\bigcup_{T\in\cT}\cF_{z_T,B}^{(q)}\right)
 =B\sqrt A\,\rho_n(\cT).
 \label{eq:matching-scalar-realization}
\end{equation}
If $\varnothing\in\cT$ and at least one member is nonempty, the induced positive-threshold trace family is exactly $\cT$: thresholds in $(0,A]$ recover $T$ and thresholds above $A$ give $\varnothing$.  In the projected scalar case the sign and projection scalar recover the same alternatives.  The additive envelope is $A$.  The projected envelope is also $A$, because
$c_0\E|G|=1$.  Thus the lower construction belongs to precisely the admissible class used in the upper theorem.

\medskip\noindent\emph{Finite ReLU realization on pairwise-distinct inputs.}
Assume first $n\ge2$.  The set of $w\in\R^d$ for which
$\inner{w}{x_i-x_j}=0$ for some $i\ne j$ is a finite union of proper hyperplanes, hence does not cover $\R^d$.  Choose $w$ outside it and relabel the distinct knots
$\tau_1<\cdots<\tau_n$, where $\tau_r=\inner{w}{x_{\pi(r)}}$.
For a fixed $T$, put $y_r=A\1_{\{\pi(r)\in T\}}$, define
$s_0=s_n=0$, and for $1\le r<n$ set
$s_r=(y_{r+1}-y_r)/(\tau_{r+1}-\tau_r)$.  Then
\begin{equation}
 g_T(t)=y_1+\sum_{r=1}^n(s_r-s_{r-1})(t-\tau_r)_+
 \label{eq:matching-explicit-hinge}
\end{equation}
interpolates every $(\tau_r,y_r)$, is constant outside the extreme knots, and is affine between consecutive knots.  Because each segment joins two values in $[0,A]$, it satisfies $0\le g_T(t)\le A$ for every $t\in\R$.  The right side is an affine combination of $n$ ReLU hinge units and an output bias.  Appending a scalar terminal ReLU leaves it unchanged globally, so
$x\mapsto\relu(g_T(\inner{w}{x}))$ is a finite ReLU representation with one hinge layer, a scalar terminal ReLU, and the required globally nonnegative sample values.  For $n=1$, the same conclusion follows from the constant scalar representation $z_T\equiv0$ or $z_T\equiv A$.  No network parameter depends on the auxiliary Rademacher signs; it may depend on the observed sample and on the trace index $T$, as permitted by the theorem.

\medskip\noindent\emph{Finite-cardinality lower bound.}
Fix $2\le N\le2^n$ and set $k=\lfloor\log_2N\rfloor$.  Then
$1\le k\le n$.  Put $b=\lfloor n/k\rfloor$; since $n/k\ge1$,
\begin{equation}
 b\ge \frac{n}{2k}.
 \label{eq:matching-block-size}
\end{equation}
Choose $k$ pairwise disjoint blocks $B_1,\ldots,B_k\subseteq[n]$, each of cardinality $b$, and set
\[
 \cT_{\rm blk}=\left\{\bigcup_{r\in J}B_r:J\subseteq[k]\right\}.
\]
Different choices of $J$ give different unions, so
$|\cT_{\rm blk}|=2^k\le N$, and $\varnothing\in\cT_{\rm blk}$.  Write
$S_r=\sum_{i\in B_r}\eps_i$.  Selecting exactly those blocks for which $S_r>0$ gives
\[
 \sup_{T\in\cT_{\rm blk}}\sum_{i\in T}\eps_i
 =\sum_{r=1}^k(S_r)_+.
\]
By symmetry, $\E(S_r)_+=\tfrac12\E|S_r|$, and the sharp scalar Khintchine lower bound gives
$\E|S_r|\ge\sqrt{b/2}$.  Consequently,
\begin{equation}
 \rho_n(\cT_{\rm blk})
 \ge\frac{k\sqrt b}{2\sqrt2\,n}
 \ge \frac14\sqrt{k/n}.
 \label{eq:block-rho-lower}
\end{equation}
For $x=\log_2N\ge1$, $\lfloor x\rfloor\ge x/2$; hence
$k\ge(\log_2N)/2$.  The right side of
\eqref{eq:block-rho-lower} is therefore at least a universal constant times
$\min\{1,\sqrt{\log N/n}\}$.  The scalar ReLU family just constructed has trace cardinality $2^k\le N$, envelope exactly $A$, and complexity
$B\sqrt A\rho_n(\cT_{\rm blk})$, proving the cardinality lower direction, including the saturated regime $k$ of order $n$.

\medskip\noindent\emph{VC lower bound.}
Repeat the block construction with $k=V$; then $b=\lfloor n/V\rfloor\ge1$.  Choosing one representative point from each block shows that all $2^V$ label patterns on those representatives are realized, so the VC dimension is at least $V$.  Since the class has cardinality $2^V$, its VC dimension is at most $\log_2(2^V)=V$.  It is therefore exactly $V$.  The same calculation gives
$\rho_n(\cT_{\rm blk})\ge c\min\{1,\sqrt{V/n}\}$.  Equation
\eqref{eq:matching-scalar-realization} and the finite globally nonnegative ReLU realization prove the VC lower direction.

The upper and lower constants used in the two directions are universal but are not equal.  Thus $\asymp$ in the theorem means matching universal-constant order, not exact equality of constants.
\end{proof}

\subsection{Proof of \texorpdfstring{\Cref{cor:main-explicit}}{the ReLU-family specialization}}
\label{proof:main-explicit}
\begin{proof}
Applying \eqref{eq:main-affine-envelope} with the uniform layer bounds and then taking the supremum over $\theta\in\Theta$ and $i\in[n]$ gives
\begin{equation}
 \sup_{\theta,i}\|z_L^\theta(x_i)\|_q\le R_{\Theta,q}.
 \label{eq:relu-uniform-radius-proof}
\end{equation}
The final states are nonnegative because the last representation layer is a coordinatewise ReLU.

For $q=1$, every coordinate is nonnegative and no coordinate exceeds the $\ell_1$ norm.  Hence
\[
 \sum_{j=1}^m\max_i z_{L,j}^\theta(x_i)
 \le mR_{\Theta,1}.
\]
Also, for each $j$, $\max_i z_{L,j}^\theta(x_i)\le\sum_i z_{L,j}^\theta(x_i)$, so
\[
 \sum_{j=1}^m\max_i z_{L,j}^\theta(x_i)
 \le\sum_{i=1}^n\|z_L^\theta(x_i)\|_1
 \le nR_{\Theta,1}.
\]
Taking the smaller bound and then the supremum over $\theta$ yields
\begin{equation}
 \widehat A_{1,S}(\cZ_\Theta)
 \le\min\{m,n\}R_{\Theta,1}.
 \label{eq:q1-relu-envelope-proof}
\end{equation}

For $q=2$, suppose all realized states lie in a common subspace $E$ of dimension $r$.  For $G\sim\mathcal N(0,I_m)$,
\[
 \sup_{\theta,i}|\inner{G}{z_L^\theta(x_i)}|
 =\sup_{\theta,i}|\inner{P_EG}{z_L^\theta(x_i)}|
 \le R_{\Theta,2}\|P_EG\|_2.
\]
Jensen's inequality and $\E\|P_EG\|_2^2=r$ imply
$\E\|P_EG\|_2\le\sqrt r$.  Therefore
\begin{equation}
 \widehat W_{2,S}(\cZ_\Theta)
 \le c_0R_{\Theta,2}\sqrt r.
 \label{eq:q2-relu-width-proof}
\end{equation}
Taking $E=\R^m$ gives the ambient-width estimate with $r=m$.  The same argument covers $r=0$, in which case every realized state is zero.

It remains to justify one VC bound for both trace systems.  Adjoin to the fixed representation graph the scalar affine output
\[
 x\longmapsto \inner{g}{z_L^\theta(x)}-t,
 \qquad g\in\R^m,\quad t\in\R.
\]
The augmented graph has
$W_\Theta=P+m+1$ trainable weights and biases and longest path at most $L+1$.  Because each noninput unit has a counted bias, its number of computation units is at most a universal multiple of $W_\Theta$.  Every nonoutput activation is ReLU, a two-piece degree-one function, and the new output is the identity.  The main upper-bound theorem of
\citet{BartlettHarvey2019} applies to a fixed directed acyclic computation graph of this form and gives
\[
 \operatorname{VCdim}\left\{
 x\mapsto\1_{\{\inner{g}{z_L^\theta(x)}-t\ge0\}}:
 (\theta,g,t)
 \right\}
 \le C_{\rm net}W_\Theta(L+1)\log(eW_\Theta),
\]
after enlarging the universal constant to include the finitely many small-unit cases.  The additive coordinate-threshold class is a subclass obtained by restricting $g$ to coordinate vectors.  The projected trace class is a subclass as well: the sign $s$ is absorbed into $g$, and allowing every real $t$ only enlarges the class relative to $t>0$.  Thus the displayed $V_\Theta$ bounds both realized trace systems with the exact sample, threshold, coordinate, and projection quantifiers used in the corollary.

The VC trace estimate proved in the preceding theorem gives
$\rho_{q,S}(\cZ_\Theta)\le C\min\{1,\sqrt{V_\Theta/n}\}$.  Substituting this estimate and
\eqref{eq:q1-relu-envelope-proof} into the $q=1$ line of
\Cref{thm:main-adaptive} yields
\[
 \Rad_S(\cF_{\cZ_\Theta,B}^{(1)})
 \le C B\sqrt{\min\{m,n\}R_{\Theta,1}}
       \min\{1,\sqrt{V_\Theta/n}\}.
\]
Substitution of \eqref{eq:q2-relu-width-proof} yields
\[
 \Rad_S(\cF_{\cZ_\Theta,B}^{(2)})
 \le C B\sqrt{R_{\Theta,2}}\,r^{1/4}
       \min\{1,\sqrt{V_\Theta/n}\},
\]
where the fixed factor $c_0^{1/2}$ is absorbed into the universal constant.  Taking $r=m$ gives the stated ambient-width term.

Finally, \eqref{eq:relu-uniform-radius-proof} and
\Cref{prop:main-isometry} give
$\|\Phi_q(z_L^\theta(x_i))\|\le\sqrt{R_{\Theta,q}}$.  For every sign vector,
\[
 \sup_\theta\left\|\sum_{i=1}^n\eps_i
 \Phi_q(z_L^\theta(x_i))\right\|
 \le\sum_{i=1}^n\sup_\theta
 \|\Phi_q(z_L^\theta(x_i))\|
 \le n\sqrt{R_{\Theta,q}}.
\]
Hilbert duality therefore gives the deterministic cap
$\Rad_S\le B\sqrt{R_{\Theta,q}}$, including $R_{\Theta,q}=0$.  Taking the minimum of this cap and the corresponding VC estimate proves
\eqref{eq:main-explicit-q1}--\eqref{eq:main-explicit-q2} and the common-rank replacement.
\end{proof}

\section{Additional Terminal-Head Results}
\label{app:auxiliary-brownian}

This section collects supporting terminal-head refinements, not additional headline contributions. Fixed-representation Hilbert duality and the empirical-process tools used below are standard and are adapted here. Each statement is followed immediately by its proof; topics are Brownian identities, covariance, same-sample tools, finite realization, and consequences.

\subsection{Scalar Brownian identities and kernel construction}

The scalar Brownian kernel on $\R$ is
\begin{equation}
  \kb(s,t)=\frac{\abs{s}+\abs{t}-\abs{s-t}}{2}.
  \label{eq:scalar-brownian}
\end{equation}
It has the signed-threshold representation
\begin{equation}
  \kb(s,t)
  =\int_0^\infty
  \left(
  \1_{\{s\ge r\}}\1_{\{t\ge r\}}
  +
  \1_{\{s\le-r\}}\1_{\{t\le-r\}}
  \right)\dd r.
  \label{eq:signed-threshold-rep}
\end{equation}
On $\R_+$ this reduces to $\kb(s,t)=\min\{s,t\}$, with canonical feature $r\mapsto\1_{[0,s]}(r)$ in $L^2(\R_+)$.  For the additive geometry,
\begin{equation}
  K_1(u,v)=\sum_{j=1}^m\kb(u_j,v_j).
  \label{eq:additive-brownian}
\end{equation}

\begin{proposition}[Positive definiteness and Gaussian projection]
\label{prop:pd-kernels}
The kernels $K_1$ and $K_2$ in \eqref{eq:main-brownian-kernel} are positive definite.  If $G\sim\mathcal N(0,I_m)$ and $c_0=\sqrt{\pi/2}$, then
\begin{equation}
  K_2(u,v)=c_0\,\E_G\!\left[\kb(\inner{G}{u},\inner{G}{v})\right].
  \label{eq:gaussian-projection-kernel}
\end{equation}
\end{proposition}

\begin{proof}
Let
$\Omega:=\{+,-\}\times\R_+$ with counting measure on $\{+,-\}$ and
Lebesgue measure on $\R_+$, and define
\[
 \phi(s)(+,r):=\1_{\{s\ge r\}},
 \qquad
 \phi(s)(-,r):=\1_{\{s\le-r\}}.
\]
For every $s\in\R$, $\phi(s)\in L^2(\Omega)$ and
$\|\phi(s)\|_{L^2(\Omega)}^2=|s|$.  The signed-threshold identity
\eqref{eq:signed-threshold-rep} gives
\[
 \inner{\phi(s)}{\phi(t)}_{L^2(\Omega)}=\kb(s,t).
\]
Thus $\kb$ is positive definite, including at $s=0$ or $t=0$.
Taking the Hilbert direct sum of $m$ copies of $L^2(\Omega)$ shows that
$K_1(u,v)=\sum_{j=1}^m\kb(u_j,v_j)$ is positive definite.

For fixed $g\in\R^m$, the map
$(u,v)\mapsto\kb(\inner{g}{u},\inner{g}{v})$ is the pullback of the
positive-definite kernel $\kb$ through the linear functional
$u\mapsto\inner{g}{u}$.  Moreover,
\[
  0\le \kb(a,b)\le \frac{|a|+|b|}{2},
\]
so the Gaussian integrand is integrable for every fixed $u,v$.  Therefore a
finite Gram quadratic form may be passed through the expectation, and the
expectation is again positive definite.  Finally, for every $w\in\R^m$,
$\inner{G}{w}\sim\mathcal N(0,\|w\|_2^2)$ and hence
$\E|\inner{G}{w}|=\sqrt{2/\pi}\|w\|_2$.  Applying this identity to
$w=u,v,u-v$ gives
\[
 c_0\E\kb(\inner{G}{u},\inner{G}{v})
 =\frac{\norm{u}_2+\norm{v}_2-\norm{u-v}_2}{2}=K_2(u,v).
\]
This proves both positive definiteness and
\eqref{eq:gaussian-projection-kernel}.
\end{proof}

For $u\in\R_+^m$, the additive canonical map has the explicit realization
\begin{equation}
  \Phi_1(u)(j,t)=\1_{[0,u_j]}(t)
  \quad\text{in}\quad
  L^2([m]\times\R_+,\#\otimes\dd t).
  \label{eq:explicit-additive-feature}
\end{equation}
Consequently, an additive Brownian readout can be written as
\begin{equation}
  h_a(u)=\inner{a}{\Phi_1(u)}
  =\sum_{j=1}^m\int_0^{u_j}a_j(t)\dd t
  =\sum_{j=1}^m g_j(u_j),
  \label{eq:additive-head}
\end{equation}
where $g_j(0)=0$, $g_j'=a_j$, and $\norm{h_a}_{\cH_1}^2=\sum_j\norm{g_j'}_{L^2}^2$.

\subsection{Fixed-head refinements and localization}
\label{app:fixed-head-refinements}
The next statements refine \Cref{thm:main-fixed} when the realized Gram matrix, empirical centering, or spectral localization is available.

\begin{proposition}[Gram-sensitive lower endpoint]
\label{prop:main-gram}
Let $K_S=(\langle v_i,v_j\rangle)$ with $v_i=\Phi_q(z(x_i))$, and set
$T_S=\tr K_S$ and $\Delta_S^2=2\sum_{i\ne j}(K_S)_{ij}^2$.
If $T_S>0$, then
\begin{equation}
 \Rad_S(\cF_{z,B}^{(q)})
 \ge \frac{B\sqrt{T_S}}n
 \max\left\{\frac1{\sqrt2},\frac1{\sqrt{1+\Delta_S^2/T_S^2}}\right\}.
 \label{eq:main-gram-lower}
\end{equation}
When $K_S$ is diagonal, the second term equals the trace upper bound.
\end{proposition}

\phantomsection\label{proof:main-gram}
\begin{proof}
Write
$X_\eps=\norm{\sum_i\eps_i v_i}^2=\eps^\top\mathbf K_S\eps$ and
$Y_\eps=\sqrt{X_\eps}$.
Independence of the signs gives
\begin{equation}
  \E X_\eps=T_S,
  \qquad
  \E X_\eps^2=T_S^2+2\sum_{i\ne j}(\mathbf K_S)_{ij}^2
  =T_S^2+\Delta_S^2.
  \label{eq:gram-quadratic-moments}
\end{equation}
Interpolation between $L^1$, $L^2$, and $L^4$ gives
$\norm{Y_\eps}_{L^2}
\le\norm{Y_\eps}_{L^1}^{1/3}\norm{Y_\eps}_{L^4}^{2/3}$, hence
\[
  \E Y_\eps
  \ge\frac{(\E Y_\eps^2)^{3/2}}{(\E Y_\eps^4)^{1/2}}
  =\frac{T_S^{3/2}}{(T_S^2+\Delta_S^2)^{1/2}}.
\]
Multiply by $B/n$ and combine with the universal lower bound in
\eqref{eq:main-two-sided}.
If $\mathbf K_S$ is diagonal, then $\Delta_S=0$ and this lower bound equals
the Jensen upper bound.
The assumption $T_S>0$ is used only to divide by $\sqrt{T_S}$; when
$T_S=0$, positive semidefiniteness forces every $v_i=0$ and the complexity is
zero.
\end{proof}

\begin{corollary}[Centered fixed-head complexity]
\label{cor:main-centered}
Let $v_i=\Phi_q(z(x_i))$ and $\overline v=n^{-1}\sum_i v_i$.  For the centered class
$\cF_{z,B}^{(q),\circ}=\{x\mapsto\langle a,\Phi_q(z(x))-\overline v\rangle:\|a\|\le B\}$, define
\begin{equation}
 \widehat V_{q,S}^{\circ}(z)
 =\frac{1}{2n^2}\sum_{i,j=1}^n\norm{z(x_i)-z(x_j)}_q.
 \label{eq:main-centered-dispersion}
\end{equation}
Then
\begin{equation}
 \frac{B}{\sqrt2}\sqrt{\frac{\widehat V_{q,S}^{\circ}(z)}n}
 \le \Rad_S(\cF_{z,B}^{(q),\circ})
 \le B\sqrt{\frac{\widehat V_{q,S}^{\circ}(z)}n}.
 \label{eq:main-centered-bound}
\end{equation}
For a depth-$L$ extractor, $\widehat V_{q,S}^{\circ}(z_L)$ is at most
$\bigl(\prod_{\ell=1}^Ls_{\ell,q}\bigr)(2n^2)^{-1}\sum_{i,j}\|x_i-x_j\|_q$, with no bias term.
\end{corollary}

\phantomsection\label{proof:main-centered}
\begin{proof}
Let $w_i=\Phi_q(z(x_i))$ and $\overline w=n^{-1}\sum_iw_i$, and set $v_i=w_i-\overline w$. Hilbert-space duality gives
\[
  \Rad_S(\cF_{z,B}^{(q),\circ})
  =\frac Bn\E_\eps\norm{\sum_i\eps_i v_i}_{\cH_q}.
\]
For arbitrary vectors $w_1,\dots,w_n$ in a Hilbert space, expansion around their empirical mean gives
\begin{equation}
  \frac1n\sum_i\norm{w_i-\overline w}^2
  =\frac1{2n^2}\sum_{i,j}\norm{w_i-w_j}^2.
  \label{eq:hilbert-pairwise-variance}
\end{equation}
Applying this identity to the Brownian features and using \eqref{eq:main-isometry} yields
\[
  \frac1n\sum_i\norm{v_i}^2
  =\frac1{2n^2}\sum_{i,j}\norm{z(x_i)-z(x_j)}_q
  =\widehat V_{q,S}^{\circ}(z).
\]
The two inequalities in \eqref{eq:main-centered-bound} now follow from \Cref{lem:main-khintchine}.
Finally, \eqref{eq:main-deep-propagation} bounds every pairwise term for a depth-$L$ extractor; averaging over the ordered pairs gives the displayed bias-free product envelope.
\end{proof}

\begin{proposition}[Localized spectral complexity]
\label{prop:main-localized}
Let $v_i=\Phi_q(z(x_i))$, let $C_S=n^{-1}\sum_i v_i\otimes v_i$ have nonzero eigenvalues $(\lambda_j)_{j=1}^r$, and define
$\cF_{z,B}^{(q)}(\varrho)=\{f\in\cF_{z,B}^{(q)}:n^{-1}\sum_i f(x_i)^2\le\varrho^2\}$.
Then
\begin{equation}
 \Rad_S(\cF_{z,B}^{(q)}(\varrho))
 \le \frac1{\sqrt n}\inf_{t\ge0}
 \left[(B^2+t\varrho^2)\sum_{j=1}^r\frac{\lambda_j}{1+t\lambda_j}\right]^{1/2}.
 \label{eq:main-localized}
\end{equation}
In particular,
\[
 \Rad_S(\cF_{z,B}^{(q)}(\varrho))
 \le \min\left\{B\sqrt{\frac{\widehat M_{q,S}(z)}n},
 \left[\frac2n\sum_j\min\{\varrho^2,B^2\lambda_j\}\right]^{1/2}\right\}.
\]
\end{proposition}

\phantomsection\label{proof:main-localized}
\begin{proof}
Put $g_\eps=\sum_{i=1}^n\eps_i v_i$.  The component of $a$ in
$\ker C_S$ is orthogonal to every $v_i$, so it changes neither the sample
predictions nor $\inner{a}{g_\eps}$ and can be discarded.  On
$(\ker C_S)^\perp$, membership in the localized class is exactly the pair of
constraints
\[
  \norm a^2\le B^2,
  \qquad
  \inner{a}{C_Sa}=\frac1n\sum_{i=1}^n\inner{a}{v_i}^2\le\varrho^2.
\]
For $t\ge0$, the operator $I+tC_S$ is strictly positive and invertible.  Its
Cauchy--Schwarz inequality gives, for every admissible $a$,
\begin{align*}
 \inner{a}{g_\eps}
 &=\inner{(I+tC_S)^{1/2}a}
 {(I+tC_S)^{-1/2}g_\eps}\\
 &\le
 \bigl(\norm a^2+t\inner{a}{C_Sa}\bigr)^{1/2}
 \inner{g_\eps}{(I+tC_S)^{-1}g_\eps}^{1/2}\\
 &\le
 (B^2+t\varrho^2)^{1/2}
 \inner{g_\eps}{(I+tC_S)^{-1}g_\eps}^{1/2}.
\end{align*}
After taking the supremum, expectation, and applying Jensen's inequality,
independence of the signs yields
\begin{align*}
 \frac1n\E_\eps\sup_a\inner{a}{g_\eps}
 &\le \frac{(B^2+t\varrho^2)^{1/2}}n
 \left[\E_\eps\inner{g_\eps}{(I+tC_S)^{-1}g_\eps}\right]^{1/2}\\
 &=\frac1{\sqrt n}
 \left[(B^2+t\varrho^2)
 \tr\{C_S(I+tC_S)^{-1}\}\right]^{1/2},
\end{align*}
because $\E(g_\eps\otimes g_\eps)=\sum_i v_i\otimes v_i=nC_S$.
Diagonalizing $C_S$ gives
$\tr\{C_S(I+tC_S)^{-1}\}=\sum_{j=1}^r
\lambda_j/(1+t\lambda_j)$, and taking the infimum over $t\ge0$ proves
\eqref{eq:main-localized}.

At $t=0$, the eigenvalue sum is $\tr C_S=\widehat M_{q,S}(z)$, which gives
the global trace bound.  If $B>0$ and $\varrho>0$, choose
$t=B^2/\varrho^2$.  Then $B^2+t\varrho^2=2B^2$ and, for each
$\lambda\ge0$,
\[
 \frac{\lambda}{1+(B^2/\varrho^2)\lambda}
 \le \min\left\{\lambda,\frac{\varrho^2}{B^2}\right\}.
\]
Substitution gives
$[2n^{-1}\sum_j\min\{B^2\lambda_j,\varrho^2\}]^{1/2}$.
If $B=0$, the class contains only zero.  If $\varrho=0$, every admissible
$a$ satisfies $C_S^{1/2}a=0$, so all sample predictions and the empirical
Rademacher complexity are zero.  If $r=0$, the same conclusion holds.  Thus
the displayed particular bounds cover all singular and endpoint cases.
\end{proof}

\subsection{Covariance-shaped terminal heads}
\label{app:covariance-heads}
For $\Sigma\succeq0$, write $\|u\|_\Sigma=\|\Sigma^{1/2}u\|_2$ and
$K_\Sigma(u,v)=\tfrac12(\|u\|_\Sigma+\|v\|_\Sigma-\|u-v\|_\Sigma)$.
The first proposition identifies this geometry as a Euclidean Brownian pullback and records the exact same-sample reduction; the subsequent results quantify perturbations and anisotropic refinements.

\begin{proposition}[Covariance pullback and same-sample reduction]
\label{prop:main-covariance}
If $G_\Sigma\sim\mathcal N(0,\Sigma)$ and $c_0=\sqrt{\pi/2}$, then
\begin{equation}
 K_\Sigma(u,v)=K_2(\Sigma^{1/2}u,\Sigma^{1/2}v)
 =c_0\,\E\,\kb(\langle G_\Sigma,u\rangle,\langle G_\Sigma,v\rangle).
 \label{eq:main-covariance}
\end{equation}
Hence $\|\Phi_\Sigma(u)\|^2=\|u\|_\Sigma$ and
$\|\Phi_\Sigma(u)-\Phi_\Sigma(v)\|^2=\|u-v\|_\Sigma$.
If $\Sigma$ is selected on the same sample from a family $\mathfrak S$, the resulting union of heads is isometric to the Euclidean $q=2$ class generated by
$\mathfrak S^{1/2}\cZ=\{x\mapsto\Sigma^{1/2}z(x)\}$, and its projected-threshold traces form a subset of those already generated by arbitrary linear projections of $\cZ$.
\end{proposition}

\phantomsection\label{proof:main-covariance}
\begin{proof}
The pullback identity in \eqref{eq:main-covariance} follows by inserting $\|u\|_\Sigma=\|\Sigma^{1/2}u\|_2$ into the definition of $K_\Sigma$. Positive definiteness is therefore inherited from $K_2$.
For every $w\in\R^m$, the scalar $\langle G_\Sigma,w\rangle$ is centered Gaussian with variance $w^\top\Sigma w$, and hence
\[
  \E\abs{\langle G_\Sigma,w\rangle}
  =\sqrt{2/\pi}\,\|w\|_\Sigma.
\]
Applying this identity to $w=u,v,u-v$ in the scalar Brownian kernel proves the Gaussian-mixture equality in \eqref{eq:main-covariance}. The canonical RKHS identities follow from
$\|\Phi_\Sigma(u)\|^2=K_\Sigma(u,u)$ and from expanding
$\|\Phi_\Sigma(u)-\Phi_\Sigma(v)\|^2$.

For same-sample covariance selection, replace each pair $(\Sigma,z)$ by the transformed representation $x\mapsto\Sigma^{1/2}z(x)$. The preceding pullback identity is an isometry between the corresponding head classes. Moreover,
$\langle g,\Sigma^{1/2}z(x)\rangle=\langle\Sigma^{1/2}g,z(x)\rangle$, so every threshold trace of the transformed class is already generated by an arbitrary linear projection of the original representation family.
\end{proof}

\begin{proposition}[Deterministic stability under covariance perturbations]
\label{prop:covariance-perturbation-stability}
Let $\Sigma,\Gamma\succeq0$ and set
$\eta:=\norm{\Sigma-\Gamma}_{\op}^{1/2}$.  Then, for all $r,u,v\in\R^m$,
\begin{align}
  \abs{\norm{r}_\Sigma-\norm{r}_\Gamma}
  &\le \eta\norm{r}_2,
  \label{eq:covariance-seminorm-stability}\\
  \abs{K_\Sigma(u,v)-K_\Gamma(u,v)}
  &\le\frac\eta2\bigl(\norm{u}_2+\norm{v}_2+\norm{u-v}_2\bigr).
  \label{eq:covariance-kernel-stability}
\end{align}
For a sample $S=(x_i)_{i=1}^n$ and a fixed representation $z$, write
\[
  \widehat M_{\Sigma,S}(z):=\frac1n\sum_{i=1}^n\norm{z(x_i)}_\Sigma.
\]
Then
\begin{equation}
  \abs{\widehat M_{\Sigma,S}(z)-\widehat M_{\Gamma,S}(z)}
  \le \eta\,\frac1n\sum_{i=1}^n\norm{z(x_i)}_2.
  \label{eq:covariance-mass-stability}
\end{equation}
If $\max_i\norm{z(x_i)}_2\le R$ and $\mathbf K_{\Sigma,S},\mathbf K_{\Gamma,S}$ are the two sample Gram matrices, then
\begin{equation}
  \norm{\frac1n(\mathbf K_{\Sigma,S}-\mathbf K_{\Gamma,S})}_{\op}
  \le2R\eta.
  \label{eq:covariance-gram-stability}
\end{equation}
Consequently, if the eigenvalues of the normalized Gram matrices are ordered
nonincreasingly and counted with multiplicity, each corresponding pair differs
by at most $2R\eta$.
\end{proposition}

\begin{proof}
For $a,b\ge0$, $\abs{\sqrt a-\sqrt b}\le\sqrt{\abs{a-b}}$; hence
\[
  \abs{\norm{r}_\Sigma-\norm{r}_\Gamma}
  \le\sqrt{\abs{r^\top(\Sigma-\Gamma)r}}
  \le\eta\norm{r}_2.
\]
The kernel and mass bounds follow by applying this inequality to the defining terms.  Under the radius assumption, every Gram entry changes by at most $2R\eta$, so
$\norm{\mathbf K_{\Sigma,S}-\mathbf K_{\Gamma,S}}_{\op}
\le\norm{\mathbf K_{\Sigma,S}-\mathbf K_{\Gamma,S}}_F\le2nR\eta$.
After division by $n$, Weyl's eigenvalue perturbation inequality for real
symmetric matrices gives the final eigenvalue assertion, including repeated
and zero eigenvalues.
\end{proof}

The explicit norm formula also yields a general finite-dimensional expansion for covariance deformation.

\begin{proposition}[Second-order covariance-deformation expansion]
\label{prop:general-covariance-expansion}
Let $\Sigma,\Gamma\succ0$, put $H:=\Gamma-\Sigma$, and define
\[
  h:=\norm{H}_{\op},
  \qquad
  a_*:=\min\{\lambda_{\min}(\Sigma),\lambda_{\min}(\Gamma)\}>0.
\]
For $r\in\R^m$, set
\begin{equation}
  q_{\Sigma,H}(r)
  :=
  \begin{cases}
    \displaystyle\frac{r^\top Hr}{\norm{r}_\Sigma},&r\ne0,\\[1ex]
    0,&r=0,
  \end{cases}
  \qquad
  \mathcal D_{\Sigma,H}(u,v)
  :=\frac14\bigl[q_{\Sigma,H}(u)+q_{\Sigma,H}(v)-q_{\Sigma,H}(u-v)\bigr].
  \label{eq:general-covariance-derivative}
\end{equation}
Then
\begin{equation}
  K_\Gamma(u,v)
  =K_\Sigma(u,v)+\mathcal D_{\Sigma,H}(u,v)
  +R_{\Sigma,H}(u,v),
  \label{eq:general-covariance-expansion}
\end{equation}
where
\begin{equation}
  \abs{R_{\Sigma,H}(u,v)}
  \le\frac{h^2}{16a_*^{3/2}}
  \bigl(\norm{u}_2+\norm{v}_2+\norm{u-v}_2\bigr).
  \label{eq:general-covariance-remainder}
\end{equation}
Thus the Fr\'echet derivative is explicit and the remainder is quadratic in the covariance perturbation, with no distributional or high-dimensional assumption.
\end{proposition}

\begin{proof}
For fixed $r$, let
$g_r(t)=\sqrt{r^\top(\Sigma+tH)r}$, $t\in[0,1]$.
If $r=0$, then $g_r(t)=0$, $q_{\Sigma,H}(0)=0$, and the expansion and
remainder bound both read $0=0$.  Assume henceforth that $r\ne0$.  Then
\[
  g_r'(0)=\frac{r^\top Hr}{2\norm{r}_\Sigma},
  \qquad
  g_r''(t)=-\frac{(r^\top Hr)^2}
  {4\{r^\top(\Sigma+tH)r\}^{3/2}}.
\]
Because $\Sigma+tH=(1-t)\Sigma+t\Gamma\succeq a_*I_m$ and
$\abs{r^\top Hr}\le h\norm{r}_2^2$, Taylor's theorem gives
\[
  \abs{g_r(1)-g_r(0)-g_r'(0)}
  \le\frac{h^2}{8a_*^{3/2}}\norm{r}_2.
\]
Apply this estimate to $r=u,v,u-v$ and combine the three norm expansions in the definition of $K_\Gamma$.
\end{proof}

\begin{corollary}[Two-sided nonasymptotic spiked Brownian expansion]
\label{cor:spiked-brownian-expansion}
Fix $a>0$, $\delta>-a$, and $w\in\R^m$ with $\norm{w}_2=1$, and set
$\Sigma_{a,\delta,w}=aI_m+\delta ww^\top$ and
$a_\delta:=a+\min\{\delta,0\}>0$.
For
\begin{equation}
  \vartheta_w(r)
  :=
  \begin{cases}
    \displaystyle\frac{\inner{w}{r}^2}{\norm{r}_2},&r\ne0,\\[1ex]
    0,&r=0,
  \end{cases}
  \qquad
  D_w(u,v):=\vartheta_w(u)+\vartheta_w(v)-\vartheta_w(u-v),
  \label{eq:spiked-directional-form}
\end{equation}
we have
\begin{equation}
  K_{\Sigma_{a,\delta,w}}(u,v)
  =\sqrt a\,K_2(u,v)
  +\frac{\delta}{4\sqrt a}D_w(u,v)
  +R_{a,\delta,w}(u,v),
  \label{eq:spiked-brownian-expansion}
\end{equation}
with
\begin{equation}
  \abs{R_{a,\delta,w}(u,v)}
  \le
  \frac{\delta^2}{16a_\delta^{3/2}}
  \bigl(\norm{u}_2+\norm{v}_2+\norm{u-v}_2\bigr).
  \label{eq:spiked-brownian-remainder}
\end{equation}
The same formula therefore covers directional amplification ($\delta>0$) and suppression ($-a<\delta<0$).
\end{corollary}

\begin{proof}
Apply \Cref{prop:general-covariance-expansion} with
$\Sigma=aI_m$, $H=\delta ww^\top$, and
$\Gamma=\Sigma_{a,\delta,w}$.  The condition $\delta>-a$ makes both matrices
positive definite.  Moreover,
\[
 \norm{H}_{\op}=|\delta|,
 \qquad
 \min\{\lambda_{\min}(aI_m),\lambda_{\min}(aI_m+\delta ww^\top)\}
 =a_\delta.
\]
For $r\ne0$,
\[
 q_{aI_m,\delta ww^\top}(r)
 =\frac{\delta\inner{w}{r}^2}{\sqrt a\norm r_2}
 =\frac\delta{\sqrt a}\vartheta_w(r),
\]
and both sides are zero at $r=0$ by definition.  Hence
$\mathcal D_{aI_m,H}(u,v)=\delta D_w(u,v)/(4\sqrt a)$.
The base kernel is $K_{aI_m}=\sqrt a\,K_2$, and the general remainder bound
becomes exactly \eqref{eq:spiked-brownian-remainder} because
$\norm H_{\op}^2=\delta^2$.  This proves the asserted expansion for both
signs of $\delta$.
\end{proof}

\begin{remark}[Scope of anisotropic propagation]
\label{rem:anisotropic-relu-scope}
The feature identities and fixed-representation complexity statements hold for every supplied $\Sigma\succeq0$.  If $\Sigma$ is selected on the same sample, \Cref{prop:main-covariance} reduces the resulting class to the transformed Euclidean representation family $\mathfrak S^{1/2}\mathcal Z$.  Coordinatewise ReLU is automatically nonexpansive for diagonal weighted Euclidean seminorms, but not for a general nondiagonal $\Sigma$.  Accordingly, \Cref{prop:main-propagation} is stated for the additive and isotropic Euclidean geometries; a general anisotropic head is controlled by its realized $\Sigma$-activation mass rather than by an unsupported layerwise contraction.
\end{remark}

\begin{corollary}[Anisotropic activation-mass characterization]
\label{cor:anisotropic-sharp-complexity}
Fix $z:\cX\to\R^m$, $S=(x_i)_{i=1}^n$, $B\ge0$, and $\Sigma\succeq0$. Define
\begin{equation}
  \cF_{z,B}^{(\Sigma)}
  :=\set{x\mapsto\inner{a}{\Phi_\Sigma(z(x))}_{\cH_\Sigma}
  \given \norm{a}_{\cH_\Sigma}\le B},
  \qquad
  \widehat M_{\Sigma,S}(z)
  :=\frac1n\sum_{i=1}^n\norm{z(x_i)}_\Sigma.
  \label{eq:anisotropic-head-class-mass}
\end{equation}
Then
\begin{align}
  \Rad_S(\cF_{z,B}^{(\Sigma)})
  &=\frac Bn\E_\eps
  \norm{\sum_{i=1}^n\eps_i\Phi_\Sigma(z(x_i))}_{\cH_\Sigma},
  \label{eq:anisotropic-exact-rad}\\
  \frac{B}{\sqrt2}\sqrt{\frac{\widehat M_{\Sigma,S}(z)}n}
  &\le\Rad_S(\cF_{z,B}^{(\Sigma)})
  \le B\sqrt{\frac{\widehat M_{\Sigma,S}(z)}n}.
  \label{eq:anisotropic-two-sided-rad}
\end{align}
If $\Sigma\ne0$ and $n\ge2$, both constants are attained by canonical
features along any direction $h$ with $h^\top\Sigma h>0$.  For $n=1$, the
complexity equals the upper endpoint for every nonzero realized feature.  For
$\Sigma=0$, the class is identically zero.  Moreover, for $c>0$, $K_{c\Sigma}=\sqrt c\,K_\Sigma$ and
\begin{equation}
  \norm{f}_{\cH_{c\Sigma}}
  \sqrt{\widehat M_{c\Sigma,S}(z)}
  =\norm{f}_{\cH_\Sigma}
  \sqrt{\widehat M_{\Sigma,S}(z)}
  \label{eq:anisotropic-scale-balance}
\end{equation}
for every function $f$ in the common RKHS function set.
\end{corollary}

\begin{proof}
By \Cref{prop:main-covariance}, the anisotropic feature configuration is the Euclidean Brownian configuration generated by $\Sigma^{1/2}z(x_i)$.  Applying \Cref{thm:main-fixed} gives \eqref{eq:anisotropic-exact-rad}--\eqref{eq:anisotropic-two-sided-rad}.  When $\Sigma\ne0$, choose $h$ with
$\|h\|_\Sigma=(h^\top\Sigma h)^{1/2}>0$ and normalize it to
$\bar h=h/\|h\|_\Sigma$.  For $n\ge2$, the configurations
$z(x_1)=nM\bar h$, $z(x_i)=0$ for $i>1$, and
$z(x_1)=z(x_2)=(nM/2)\bar h$, $z(x_i)=0$ for $i>2$, have anisotropic
activation mass $M$ and reproduce respectively the upper and lower endpoint
computations in \Cref{thm:main-fixed}.  For $n=1$, Hilbert duality gives the
exact value $B\sqrt{M}$.  Finally, multiplying a kernel by $\alpha>0$ leaves its function set unchanged and multiplies squared RKHS norms by $\alpha^{-1}$.  Here $\alpha=\sqrt c$ and $\widehat M_{c\Sigma,S}=\sqrt c\,\widehat M_{\Sigma,S}$, which proves \eqref{eq:anisotropic-scale-balance}.
\end{proof}

\subsection{Auxiliary tools for same-sample selection}
\label{app:same-sample-tools}
These results give a threshold-free cap, realized-width reductions, standard trace upper bounds, and the finite scalar ReLU realization used by the lower constructions.

\begin{proposition}[Same-sample root-mass cap]
\label{prop:main-root-mass}
Let $\varnothing\ne\cZ$ and use outer expectation when the displayed
supremum is not measurable.  For either Brownian geometry,
\begin{equation}
 \Rad_S(\cF_{\cZ,B}^{(q)})
 \le \frac Bn\sup_{z\in\cZ}\sum_{i=1}^n\sqrt{\norm{z(x_i)}_q}
 \le B\sqrt{\widehat M^*_{q,S}(\cZ)}.
 \label{eq:main-root-mass}
\end{equation}
\end{proposition}

\phantomsection\label{proof:main-root-mass}
\begin{proof}
For every realization of the Rademacher signs, Hilbert-space duality gives
\[
 \sup_{z\in\cZ}\sup_{\|a\|\le B}
 \sum_{i=1}^n\eps_i\langle a,\Phi_q(z(x_i))\rangle
 =B\sup_{z\in\cZ}\left\|\sum_{i=1}^n\eps_i\Phi_q(z(x_i))\right\|.
\]
The triangle inequality and the canonical energy identity imply
\[
 \left\|\sum_i\eps_i\Phi_q(z(x_i))\right\|
 \le\sum_i\|\Phi_q(z(x_i))\|
 =\sum_i\sqrt{\|z(x_i)\|_q}.
\]
Taking the supremum, expectation, and the factor $1/n$ proves the first inequality. For every fixed $z$, Cauchy--Schwarz gives
\[
 \frac1n\sum_{i=1}^n\sqrt{\|z(x_i)\|_q}
 \le\left(\frac1n\sum_{i=1}^n\|z(x_i)\|_q\right)^{1/2}.
\]
Taking the supremum over $z\in\cZ$ and using monotonicity of the square
root gives $\sqrt{\widehat M^*_{q,S}(\cZ)}$, which proves the second
inequality.
\end{proof}

\begin{proposition}[Projected width on realized states]
\label{prop:main-intrinsic-width}
Let $\varnothing\ne\cV_S(\cZ)=\{z(x_i):z\in\cZ,\ i\in[n]\}$ and assume
$\sup_{v\in\cV_S(\cZ)}\|v\|_2\le R_z$.  If $\cV_S(\cZ)$ lies in a common
$r$-dimensional subspace, then
$\widehat W_{2,S}(\cZ)\le c_0R_z\sqrt r$.  If
$|\cV_S(\cZ)|\le M$ for an integer $M\ge1$, then
$\widehat W_{2,S}(\cZ)\le c_0R_z\sqrt{2\log(2M)}$.
\end{proposition}

\phantomsection\label{proof:main-intrinsic-width}
\begin{proof}
Let $E$ be a common $r$-dimensional subspace and let $P_E$ denote the orthogonal projection onto $E$.
For every $v\in E$,
$\inner{G}{v}=\inner{P_EG}{v}$, so Cauchy--Schwarz gives
\[
  \sup_{v\in\cV_S(\cZ)}\abs{\inner{G}{v}}
  \le R_z\norm{P_EG}_2.
\]
Because $P_EG$ is a standard Gaussian vector in $E$,
\[
  \E\norm{P_EG}_2
  \le\left(\E\norm{P_EG}_2^2\right)^{1/2}
  =\sqrt r.
\]
Multiplication by $c_0$ proves the common-subspace assertion.

For the finite-set bound, let
$M_0:=|\cV_S(\cZ)|\le M$ and enumerate its distinct points as
$v_1,\dots,v_{M_0}$.  For $\lambda>0$, Jensen's inequality followed by the
Gaussian moment-generating function gives
\begin{align*}
 \exp\left\{\lambda\E\max_{j\le M_0}|\inner{G}{v_j}|\right\}
 &\le \E\exp\left\{\lambda\max_{j\le M_0}|\inner{G}{v_j}|\right\}\\
 &\le \sum_{j=1}^{M_0}
 \left[\E e^{\lambda\inner{G}{v_j}}
       +\E e^{-\lambda\inner{G}{v_j}}\right]\\
 &\le 2M_0\exp\{\lambda^2R_z^2/2\}.
\end{align*}
Thus
\[
 \E\max_{j\le M_0}|\inner{G}{v_j}|
 \le \frac{\log(2M_0)}\lambda+\frac{\lambda R_z^2}{2}.
\]
If $R_z=0$, the left side is zero.  Otherwise, minimizing at
$\lambda=\sqrt{2\log(2M_0)}/R_z$ yields
$R_z\sqrt{2\log(2M_0)}\le R_z\sqrt{2\log(2M)}$.  Multiplication by $c_0$
proves the assertion.
\end{proof}

\begin{corollary}[Finite-trace and VC upper bounds]
\label{cor:main-trace-entropy}
For every nonempty trace family $\varnothing\ne\cT\subseteq2^{[n]}$,
\begin{equation}
 \rho_n(\cT)\le
 \min\left\{1,\sqrt{\frac{2\log(2|\cT|)}n}\right\}.
 \label{eq:main-cardinality-rho}
\end{equation}
If $\emptyset\in\cT$ and $\vc(\cT)\le V$, then
$\rho_n(\cT)\le C_{\rm VC}\min\{1,\sqrt{V/n}\}$.
\end{corollary}

\phantomsection\label{proof:main-trace-entropy}
\begin{proof}
Associate each $T\in\cT$ with its indicator vector
$v_T=(\1_{\{i\in T\}})_{i=1}^n\in\{0,1\}^n$.
Let
$\mathcal V:=\{v_T:T\in\cT\}\cup\{-v_T:T\in\cT\}$; it is nonempty,
has at most $2|\cT|$ elements, and every $v\in\mathcal V$ satisfies
$\|v\|_2\le\sqrt n$.  For $\lambda>0$, Jensen's inequality and
$\E e^{\lambda\inner{\eta}{v}}=\prod_i\cosh(\lambda v_i)
\le e^{\lambda^2\|v\|_2^2/2}$ imply
\begin{align*}
 \E\max_{v\in\mathcal V}\inner{\eta}{v}
 &\le \frac1\lambda
 \log\sum_{v\in\mathcal V}\E e^{\lambda\inner{\eta}{v}}\\
 &\le \frac{\log(2|\cT|)}\lambda+\frac{\lambda n}{2}.
\end{align*}
Minimizing over $\lambda$ gives
$\E\sup_{T\in\cT}|\inner{\eta}{v_T}|
\le\sqrt{2n\log(2|\cT|)}$.  Division by $n$ proves the logarithmic bound.
The deterministic estimate
$\sup_T\abs{\sum_{i\in T}\eta_i}\le n$ gives the additional cap by one.

For the VC bound, define the one-sided empirical Rademacher average
\[
  \mathfrak r_n^+(\cT)
  =\frac1n\E_\eta\sup_{T\in\cT}\sum_{i\in T}\eta_i.
\]
Because $\emptyset\in\cT$, the supremum is nonnegative.  By symmetry of
$\eta$ and the elementary inequality $\max\{a,b\}\le a+b$ for
$a,b\ge0$,
\begin{align*}
  \rho_n(\cT)
  &=\frac1n\E\max\left\{
    \sup_T\sum_{i\in T}\eta_i,
    \sup_T\sum_{i\in T}(-\eta_i)
  \right\}\\
  &\le2\mathfrak r_n^+(\cT).
\end{align*}
If $V=0$, the presence of $\emptyset$ forces $\cT=\{\emptyset\}$: any
nonempty $T$ together with $\emptyset$ would shatter a singleton.  Hence
$\rho_n(\cT)=0$.  Assume now $V\ge1$.

Equip $\cT$ with the empirical $L^2$ pseudometric
\[
  d_n(T,T')
  =\left(\frac{\abs{T\triangle T'}}{n}\right)^{1/2}.
\]
Haussler's packing theorem for a binary class of VC dimension at most $V$
implies, after the packing-to-covering conversion, that for $0<u\le1$,
\[
  \log N\bigl(u,\cT,d_n\bigr)
  \le C_0V\log\left(\frac{C_1}{u}\right)
\]
with universal constants $C_0,C_1$~\citep{Haussler1995}.  The finite process
$X_T=n^{-1}\sum_i\eta_i\1_{\{i\in T\}}$ has increments satisfying
\[
 \E_\eta\exp\{\lambda(X_T-X_{T'})\}
 \le \exp\left\{\frac{\lambda^2d_n(T,T')^2}{2n}\right\}
 \qquad(\lambda\in\R),
\]
by $\cosh x\le e^{x^2/2}$.  Dudley's entropy bound for a finite centered
subgaussian process, applied with increment metric $d_n/\sqrt n$, therefore
gives
\begin{align*}
  \mathfrak r_n^+(\cT)
  &=\E\sup_{T\in\cT}(X_T-X_\emptyset)\\
  &\le\frac{C}{\sqrt n}
  \int_0^1\sqrt{\log N(u,\cT,d_n)}\,\dd u\\
  &\le C'\sqrt{\frac Vn}
  \int_0^1\sqrt{\log(C_1/u)}\,\dd u
  \le C''\sqrt{\frac Vn},
\end{align*}
where the last integral is finite; this is the empirical-process form used in
\citet{BartlettMendelson2002}.  Combining the result with
$\rho_n\le1$ proves the VC assertion.
\end{proof}

\begin{proposition}[Scalar ReLU realization of threshold traces]
\label{prop:main-indicator}
Let $x_1,\ldots,x_n$ be distinct and let
$\varnothing\ne\cT\subseteq2^{[n]}$.  For every $A>0$ there is a family of
globally nonnegative scalar one-hidden-layer ReLU representations satisfying
$z_T(x_i)=A\mathbf 1_{\{i\in T\}}$ and
\begin{equation}
 \Rad_S(\cF_{\{z_T:T\in\cT\},B}^{(q)})
 =B\sqrt A\,\rho_n(\cT),\qquad q\in\{1,2\}.
 \label{eq:main-indicator}
\end{equation}
\end{proposition}

\phantomsection\label{proof:main-indicator}
\begin{proof}
Define the sample values by
$z_T(x_i)=A\1_{\{i\in T\}}$.
For either scalar Brownian geometry,
$K_q(0,\cdot)=0$,
$\Phi_{q,1}(0)=0$, and
$\norm{\Phi_{q,1}(A)}_{\cH_{q,1}}=\sqrt A$.
For a fixed sign vector, Hilbert duality therefore gives
\begin{align*}
  \sup_{T\in\cT}\sup_{\norm{a}\le B}
  \frac1n\sum_{i=1}^n\eps_i
  \inner{a}{\Phi_{q,1}(z_T(x_i))}
  &=\frac Bn\sup_{T\in\cT}
  \norm{\sum_{i\in T}\eps_i\Phi_{q,1}(A)}\\
  &=\frac{B\sqrt A}{n}
  \sup_{T\in\cT}\abs{\sum_{i\in T}\eps_i}.
\end{align*}
Taking expectation proves \eqref{eq:main-indicator}.

For $q=1$, a positive scalar threshold of $z_T$ is equal to $T$ whenever
$0<t\le A$ and is empty when $t>A$.
For $q=2$, a signed scalar projection has the form
$s g A\1_T$; when $sg>0$, its positive thresholds are again $T$ or $\emptyset$, and when $sg\le0$ they are empty.
Thus both sample trace families are $\cT\cup\{\emptyset\}$. This is exactly $\cT$ when $\emptyset\in\cT$, and adjoining $\emptyset$ does not change $\rho_n(\cT)$. If $\cT$ contains a nonempty trace, the additive envelope equals $A$ and, because $c_0\E\abs G=1$, the Euclidean projected-width envelope also equals $A$. If $\cT=\{\emptyset\}$, equation~\eqref{eq:main-indicator} remains valid trivially, with both sides equal to zero.

It remains to construct the ReLU extensions.
The set of vectors $w\in\R^d$ for which
$\inner{w}{x_i}=\inner{w}{x_j}$ for some $i\ne j$ is a finite union of proper hyperplanes.
Choose $w$ outside that union and write
$t_i=\inner{w}{x_i}$; these values are pairwise distinct.
Relabel the sample so that
$t_1<\cdots<t_n$, set $y_r=A\1_{\{i_r\in T\}}$, and define slopes
\[
 d_0:=0,
 \qquad
 d_r:=\frac{y_{r+1}-y_r}{t_{r+1}-t_r}\quad(1\le r<n),
 \qquad
 d_n:=0.
\]
The explicit hinge formula
\begin{equation}
 g_T(t):=y_1+\sum_{r=1}^n(d_r-d_{r-1})(t-t_r)_+
 \label{eq:indicator-hinge-formula}
\end{equation}
has slope zero on $(-\infty,t_1)$, slope $d_r$ on
$(t_r,t_{r+1})$, and slope zero on $(t_n,\infty)$.  It therefore interpolates
$(t_r,y_r)$ and is constant outside $[t_1,t_n]$.  On each interval its value
is a convex combination of the endpoint values $y_r,y_{r+1}\in\{0,A\}$;
hence $0\le g_T(t)\le A$ for every $t\in\R$.  Formula
\eqref{eq:indicator-hinge-formula} is exactly a one-hidden-layer ReLU network
with $n$ hinge units, scalar hidden weights and biases
$(w,-t_r)$, output weights $d_r-d_{r-1}$, and output bias $y_1$.
Thus $z_T(x):=g_T(\inner{w}{x})$ is globally nonnegative and has the required
sample values.

Finally, for the full power set, let
$I_+(\eps)=\{i:\eps_i=1\}$ and
$I_-(\eps)=\{i:\eps_i=-1\}$.
Both belong to $2^{[n]}$, and
\[
  \sup_{T\subseteq[n]}\abs{\sum_{i\in T}\eps_i}
  \ge\max\{\abs{I_+(\eps)},\abs{I_-(\eps)}\}
  \ge\frac n2.
\]
Thus $\rho_n(2^{[n]})\ge1/2$, which gives the high-capacity lower statement quoted after \Cref{prop:main-indicator}.
\end{proof}

\subsection{Optimization and finite realization}
\label{sec:main-computation}
For a fixed representation, the Brownian head is an ordinary kernel problem. The results below record a singular-safe representer form, the ridge-capacity path, finite additive-profile approximation, and a vector-valued extension.

\begin{proposition}[Representer form for a terminal Brownian head]
\label{prop:main-representer}
Assume $t\mapsto\ell(y,t)$ is lower semicontinuous and bounded below, and let $\lambda>0$. The problem
\begin{equation}
 \min_{a\in\cH_q}\frac1n\sum_{i=1}^n
 \ell\bigl(y_i,\langle a,\Phi_q(z(x_i))\rangle\bigr)+\lambda\|a\|_{\cH_q}^2
 \label{eq:main-head-optimization}
\end{equation}
attains its minimum, and every minimizer lies in
$\operatorname{span}\{\Phi_q(z(x_i)):i\in[n]\}$.
For squared loss,
\begin{equation}
 \widehat f(x)=\sum_{i=1}^n\alpha_iK_q(z(x_i),z(x)),
 \qquad \alpha=(K+n\lambda I)^{-1}y.
 \label{eq:main-representer}
\end{equation}
\end{proposition}

\phantomsection\label{proof:main-representer}
\begin{proof}
Write $v_i:=\Phi_q(z(x_i))\in\cH_q$ and let
$H_S:=\operatorname{span}\{v_1,\ldots,v_n\}$.  This subspace is
finite-dimensional and therefore closed.  Every $a\in\cH_q$ has a unique
orthogonal decomposition $a=a_S+a_\perp$ with $a_S\in H_S$ and
$a_\perp\perp H_S$.  For every training point,
\[
 \inner{a}{v_i}=\inner{a_S}{v_i},
 \qquad
 \norm a^2=\norm{a_S}^2+\norm{a_\perp}^2.
\]
Thus replacing $a$ by $a_S$ leaves all empirical losses unchanged and
reduces the regularizer by $\lambda\|a_\perp\|^2$.  Every minimizer must
therefore lie in $H_S$.

We next prove existence without assuming convexity of the loss.  Define
$m_0:=n^{-1}\sum_{i=1}^n\inf_{t\in\R}\ell(y_i,t)>-\infty$.  The objective is
at least $m_0+\lambda\|a\|^2$, so every
minimizing sequence is norm bounded.  A real Hilbert space is reflexive;
hence a bounded sequence has a weakly convergent subsequence
$a_k\rightharpoonup a$.  Each coordinate map
$a\mapsto\inner{a}{v_i}$ is weakly continuous, so
$\inner{a_k}{v_i}\to\inner{a}{v_i}$.  Lower semicontinuity of each scalar
loss and weak lower semicontinuity of the squared norm give
\[
 \frac1n\sum_i\ell(y_i,\inner{a}{v_i})+\lambda\|a\|^2
 \le \liminf_k\left[
 \frac1n\sum_i\ell(y_i,\inner{a_k}{v_i})+\lambda\|a_k\|^2
 \right].
\]
The weak limit is therefore a minimizer, and the preceding orthogonal
argument places every minimizer in $H_S$.  Writing
$a=\sum_i\alpha_i v_i$ and using
$\inner{v_i}{\Phi_q(z(x))}=K_q(z(x_i),z(x))$ gives the representer form.

For squared loss, define the bounded sampling operator
$T:\cH_q\to\R^n$ by $(Ta)_i=\inner{a}{v_i}$.  Its adjoint is
$T^*c=\sum_i c_i v_i$ and $TT^*=K$.  Differentiating the quadratic objective
gives the normal equation
\[
 (T^*T+n\lambda I)a=T^*y.
\]
Because $\lambda>0$, $K+n\lambda I_n$ is positive definite even when $K$ is
singular.  Set $\alpha=(K+n\lambda I_n)^{-1}y$ and $a=T^*\alpha$.  The
identity
$(T^*T+n\lambda I)T^*=T^*(TT^*+n\lambda I_n)$ shows that this $a$ satisfies
the normal equation.  The quadratic objective is strictly convex because of
$\lambda\|a\|^2$, so this solution is unique and yields
\eqref{eq:main-representer}.
\end{proof}

\begin{proposition}[Ridge-capacity path]
\label{prop:main-ridge}
Let $\lambda>0$, let
$K=U\operatorname{diag}(\kappa_1,\ldots,\kappa_n)U^\top\succeq0$, set
$\widetilde y=U^\top y$, and let
$\alpha_\lambda=(K+n\lambda I)^{-1}y$. Then
\begin{align}
 \|\widehat f_\lambda\|_{\cH_K}^2
 &=\sum_{j=1}^n\frac{\kappa_j\widetilde y_j^2}{(\kappa_j+n\lambda)^2},\notag\\
 (-1)^r\frac{\dd^r}{\dd\lambda^r}\|\widehat f_\lambda\|_{\cH_K}^2
 &=(r+1)!n^r\sum_{j=1}^n\frac{\kappa_j\widetilde y_j^2}{(\kappa_j+n\lambda)^{r+2}}\ge0,
 \qquad r\ge0.
 \label{eq:main-ridge-path}
\end{align}
Hence the fitted Brownian capacity
$\|\widehat f_\lambda\|_{\cH_K}\sqrt{\widehat M_{q,S}(z)/n}$ is nonincreasing in $\lambda$ for a fixed representation. No monotonicity of test risk is implied.
\end{proposition}

\phantomsection\label{proof:main-ridge}
\begin{proof}
By \Cref{prop:main-representer},
$\widehat f_\lambda=\sum_i(\alpha_\lambda)_iK(x_i,\cdot)$ and hence
\[
 \|\widehat f_\lambda\|_{\cH_K}^2
 =\alpha_\lambda^\top K\alpha_\lambda
 =\sum_{j=1}^n
 \frac{\kappa_j\widetilde y_j^2}{(\kappa_j+n\lambda)^2}.
\]
For every integer $r\ge0$ and every $j$,
\[
 \frac{\dd^r}{\dd\lambda^r}
 \frac{\kappa_j\widetilde y_j^2}{(\kappa_j+n\lambda)^2}
 =(-1)^r(r+1)!n^r
 \frac{\kappa_j\widetilde y_j^2}{(\kappa_j+n\lambda)^{r+2}}.
\]
The sum is finite, so termwise differentiation is valid and proves
\eqref{eq:main-ridge-path}, including zero eigenvalues.  The case $r=1$
shows that the squared norm is nonincreasing in $\lambda$; taking square
roots preserves this order.  Multiplication by the fixed nonnegative factor
$\sqrt{\widehat M_{q,S}(z)/n}$ proves the capacity claim.  No statement about
test risk follows from this algebraic monotonicity.
\end{proof}

\begin{proposition}[Finite additive-profile realization]
\label{prop:main-spline}
Suppose $A_j>0$, $0\le u_j\le A_j$, and
$h(u)=\sum_jg_j(u_j)$, where each $g_j\in H^1(0,A_j)$ is represented by its
absolutely continuous version, $g_j(0)=0$, and
$\sum_j\|g_j'\|_{L^2(0,A_j)}^2\le B^2$.  If $I_jg_j$ is the piecewise-linear
interpolant on a finite grid containing $0$ and $A_j$, with maximum spacing
$h_j$, then
\begin{equation}
 \sup_{u\in\prod_j[0,A_j]}
 \left|h(u)-\sum_j(I_jg_j)(u_j)\right|
 \le \frac B2\left(\sum_jh_j\right)^{1/2}.
 \label{eq:main-spline}
\end{equation}
On a uniform $G_j$-cell grid with integer $G_j\ge1$,
$h_j=A_j/G_j$.
\end{proposition}

\phantomsection\label{proof:main-spline}
\begin{proof}
Fix $j\in[m]$ and one grid cell
$[a,b]=[t_{j,r-1},t_{j,r}]$.
For $t\in[a,b]$, the affine interpolation formula and absolute
continuity give
\begin{align}
  g_j(t)-(I_jg_j)(t)
  &=g_j(t)-g_j(a)
    -\frac{t-a}{b-a}\bigl(g_j(b)-g_j(a)\bigr)\notag\\
  &=\int_a^b q_t(s)g_j'(s)\dd s,
  \label{eq:spline-error-integral}
\end{align}
where
\begin{equation}
  q_t(s):=\1_{[a,t]}(s)-\frac{t-a}{b-a},
  \qquad s\in[a,b].
  \label{eq:spline-bridge-representer}
\end{equation}
Writing $\alpha=(t-a)/(b-a)$ and separating the intervals
$[a,t]$ and $(t,b]$ yields
\begin{align}
  \norm{q_t}_{L^2(a,b)}^2
  &=(t-a)(1-\alpha)^2+(b-t)\alpha^2\notag\\
  &=\frac{(t-a)(b-t)}{b-a}
  \le\frac{b-a}{4}
  \le\frac{h_j}{4}.
  \label{eq:spline-bridge-norm}
\end{align}
Cauchy--Schwarz in \eqref{eq:spline-error-integral} therefore gives
\begin{equation}
  \abs{g_j(t)-(I_jg_j)(t)}
  \le
  \frac{\sqrt{h_j}}{2}\norm{g_j'}_{L^2(a,b)}
  \le
  \frac{\sqrt{h_j}}{2}\norm{g_j'}_{L^2(0,A_j)}.
  \label{eq:spline-coordinate-error}
\end{equation}
Because the cell and $t$ were arbitrary,
\eqref{eq:spline-coordinate-error} holds uniformly on $[0,A_j]$.
Consequently, for every
$u\in\prod_{j=1}^m[0,A_j]$,
\begin{align}
  \abs{h(u)-\widetilde h(u)}
  &\le\sum_{j=1}^m
    \abs{g_j(u_j)-(I_jg_j)(u_j)}\notag\\
  &\le\frac12\sum_{j=1}^m
    \sqrt{h_j}\,\norm{g_j'}_{L^2(0,A_j)}\notag\\
  &\le\frac12
    \left(\sum_{j=1}^m h_j\right)^{1/2}
    \left(\sum_{j=1}^m\norm{g_j'}_{L^2(0,A_j)}^2\right)^{1/2}\notag\\
  &\le\frac B2\left(\sum_{j=1}^m h_j\right)^{1/2},
\end{align}
where the last two steps use Cauchy--Schwarz across coordinates and the Sobolev budget assumed in \Cref{prop:main-spline}.
Taking the supremum over $u$ proves \eqref{eq:main-spline}; the uniform-grid statement follows by substitution.
\end{proof}

\begin{proposition}[Vector-valued fixed head]
\label{prop:main-vector}
\label{prop:vector-hs-extension}
For an integer $C\ge1$, define
$\cF_{z,B}^{(q,C)}=\{x\mapsto A\Phi_q(z(x)):A:\cH_q\to\R^C,\ \|A\|_{\HS}\le B\}$ and
\[
 \Rad_S^{\mathrm{vec}}(\mathcal G)
 =\frac1n\E_\sigma\sup_{g\in\mathcal G}\sum_{i=1}^n\langle\sigma_i,g(x_i)\rangle,
\]
where $\sigma_i\in\{-1,+1\}^C$ have mutually independent coordinates. Then
\begin{equation}
 \Rad_S^{\mathrm{vec}}(\cF_{z,B}^{(q,C)})
 =\frac Bn\E_\sigma\left\|\sum_{i=1}^n\sigma_i\otimes\Phi_q(z(x_i))\right\|_{\HS}
 \le B\sqrt{\frac{C\widehat M_{q,S}(z)}n}.
 \label{eq:main-vector}
\end{equation}
\end{proposition}

\phantomsection\label{proof:main-vector}
\begin{proof}
Write $v_i=\Phi_q(z(x_i))\in\cH_q$.  The Hilbert--Schmidt inner product on
operators from $\cH_q$ to $\R^C$ satisfies
\begin{equation}
  \inner{\sigma_i}{Av_i}_{\R^C}
  =\inner{A}{\sigma_i\otimes v_i}_{\HS}.
  \label{eq:vector-rank-one-duality}
\end{equation}
Hence Hilbert-space duality gives, for every realization of the signs,
\begin{align}
  \sup_{\norm{A}_{\HS}\le B}
  \sum_{i=1}^n\inner{\sigma_i}{Av_i}_{\R^C}
  &=\sup_{\norm{A}_{\HS}\le B}
    \inner{A}{\sum_{i=1}^n\sigma_i\otimes v_i}_{\HS}\notag\\
  &=B\norm{\sum_{i=1}^n\sigma_i\otimes v_i}_{\HS}.
  \label{eq:vector-hs-duality}
\end{align}
Taking expectation and dividing by $n$ proves
\eqref{eq:main-vector}.

For the upper bound, Jensen's inequality and independence imply
\begin{align}
  \E_\sigma
  \norm{\sum_{i=1}^n\sigma_i\otimes v_i}_{\HS}
  &\le
  \left(
  \E_\sigma
  \norm{\sum_{i=1}^n\sigma_i\otimes v_i}_{\HS}^2
  \right)^{1/2}\notag\\
  &=\left(
  \sum_{i,j=1}^n
  \E\inner{\sigma_i}{\sigma_j}_{\R^C}
  \inner{v_i}{v_j}_{\cH_q}
  \right)^{1/2}\notag\\
  &=\left(C\sum_{i=1}^n\norm{v_i}_{\cH_q}^2\right)^{1/2}.
  \label{eq:vector-second-moment}
\end{align}
The feature-energy identity \eqref{eq:main-isometry} gives
$\norm{v_i}^2=\norm{z(x_i)}_q$; consequently,
$\sum_i\norm{v_i}^2=n\widehat M_{q,S}(z)$.  Substitution into the exact
duality identity and division by $n$ proves the stated upper bound, including
$B=0$, $C=1$, and zero feature configurations.
\end{proof}

\subsection{Consequences for propagation, generalization, and rescaling}

\begin{corollary}[Brownian H\"older stability]
\label{cor:holder-stability}
Let $q\in\{1,2\}$ and $f_{a}(x)=\inner{a}{\Phi_q(z_L(x))}$ with $\norm a\le B$. Then
\begin{equation}
  \abs{f_a(x)-f_a(x')}
  \le B\left(\prod_{\ell=1}^Ls_{\ell,q}\right)^{1/2}
  \norm{x-x'}_q^{1/2}.
  \label{eq:holder-stability}
\end{equation}
\end{corollary}

\begin{proof}
For every $x,x'$, Cauchy--Schwarz and the canonical feature isometry give
\begin{align*}
 |f_a(x)-f_a(x')|
 &\le\|a\|\,\|\Phi_q(z_L(x))-\Phi_q(z_L(x'))\|\\
 &\le B\,\|z_L(x)-z_L(x')\|_q^{1/2}.
\end{align*}
By \eqref{eq:main-deep-propagation},
$\|z_L(x)-z_L(x')\|_q\le
(\prod_{\ell=1}^Ls_{\ell,q})\|x-x'\|_q$.
Taking the square root proves \eqref{eq:holder-stability}; if a factor or the
input distance is zero, both sides remain well defined and the same argument
applies.
\end{proof}

\begin{corollary}[Architecture-level norm-product envelope]
\label{cor:product-envelope}
Let $\widehat M_{q,S}^{(0)}=n^{-1}\sum_i\norm{x_i}_q$.  For the fixed depth-$L$ representation in \Cref{sec:main-geometry},
\begin{equation}
  \Rad_S(\cF_{z_L,B}^{(q)})
  \le\frac{B}{\sqrt n}
  \left[
    \left(\prod_{\ell=1}^Ls_{\ell,q}\right)\widehat M_{q,S}^{(0)}
    +\sum_{r=1}^L\norm{b_r}_q\prod_{\ell=r+1}^Ls_{\ell,q}
  \right]^{1/2}.
  \label{eq:bias-product-rad}
\end{equation}
In the bias-free case,
\begin{equation}
  \Rad_S(\cF_{z_L,B}^{(q)})
  \le B\sqrt{\frac{\widehat M_{q,S}^{(0)}}n}
  \prod_{\ell=1}^Ls_{\ell,q}^{1/2}.
  \label{eq:biasfree-product-rad}
\end{equation}
\end{corollary}

\begin{proof}
Average \eqref{eq:main-affine-envelope} over the sample and substitute the resulting upper bound on $\widehat M_{q,S}(z_L)$ into the upper endpoint of \eqref{eq:main-two-sided}.  When every bias vanishes, only the first term remains.
\end{proof}

The realized activation mass in \Cref{thm:main-fixed} can be much smaller than this architecture-level envelope.  The following construction shows that no converse based only on a product of layer norms is possible.

\begin{proposition}[No lower characterization by layer-norm products alone]
\label{prop:no-product-lower}
Fix $q\in\{1,2\}$ and consider a bias-free ReLU network with input dimension $d_0\ge1$ and at least two affine--ReLU layers.
Assume that the first hidden width satisfies $d_1\ge2$ and that the second layer has at least one output coordinate.
For every prescribed $s_1,s_2>0$, there exist matrices
$W_1\in\R^{d_1\times d_0}$ and
$W_2\in\R^{d_2\times d_1}$ such that
\[
  \norm{W_1}_{q\to q}=s_1,
  \qquad
  \norm{W_2}_{q\to q}=s_2,
  \qquad
  \relu\bigl(W_2\relu(W_1x)\bigr)=0
  \quad\text{for every }x\in\R^{d_0}.
\]
Thus a network may have an arbitrary positive product
$\norm{W_1}_{q\to q}\norm{W_2}_{q\to q}=s_1s_2$
while its final representation, and hence every anchored Brownian-head class built on it, is identically zero.
Consequently, no strictly positive lower bound on Brownian-head complexity can depend only on a product of layer operator norms.
\end{proposition}

\begin{proof}
Let $q^*$ be the H\"older conjugate of $q$ and choose
$v\in\R^{d_0}$ with $\norm{v}_{q^*}=1$.
Set
\[
  W_1=s_1e_1v^\top.
\]
The induced norm of a rank-one operator satisfies
$\norm{uv^\top}_{q\to q}=\norm{u}_q\norm{v}_{q^*}$; hence
$\norm{W_1}_{q\to q}=s_1$.
Moreover, for every $x$,
$W_1x$ lies in $\operatorname{span}\{e_1\}$, and therefore
$\relu(W_1x)$ lies in the nonnegative ray $\R_+e_1$.

Because $d_1\ge2$, define
\[
  W_2=s_2e_1e_2^\top,
\]
where the first $e_1$ is an output basis vector in $\R^{d_2}$ and
$e_2$ is the second input basis vector in $\R^{d_1}$.
Again the rank-one norm formula gives
$\norm{W_2}_{q\to q}=s_2$, while $W_2e_1=0$.
It follows that
$W_2\relu(W_1x)=0$ for every $x$, and applying the second ReLU leaves zero unchanged.
Any additional bias-free layers also map zero to zero, regardless of their positive operator norms.
Finally, $\Phi_{q,m}(0)=K_q(\cdot,0)=0$, so the corresponding anchored Brownian-head class contains only the zero function.
\end{proof}

\begin{corollary}[Generalization for a fixed or independently learned representation]
\label{cor:fixed-generalization}
Let $S=((X_i,Y_i))_{i=1}^n\sim P^n$, let $\delta\in(0,1)$, and suppose $z$ is deterministic or random independently of $S$.  If $\ell:\mathcal Y\times\R\to[0,1]$ is $L_\ell$-Lipschitz in its second argument, then, with probability at least $1-\delta$, simultaneously for all $f\in\cF_{z,B}^{(q)}$,
\begin{equation}
  R(f)\le\widehat R_S(f)
  +2L_\ell B\sqrt{\frac{\widehat M_{q,S}(z)}n}
  +3\sqrt{\frac{\log(2/\delta)}{2n}}.
  \label{eq:fixed-gen-bound}
\end{equation}
Here $R(f)=\E_P\ell(Y,f(X))$ and $\widehat R_S(f)=n^{-1}\sum_i\ell(Y_i,f(X_i))$.
\end{corollary}

\begin{proof}
First suppose $z$ is deterministic.  The empirical-Rademacher deviation
theorem for a fixed class $\mathcal G$ of functions with values in $[0,1]$
states that, with probability at least $1-\delta$, simultaneously for all
$g\in\mathcal G$,
\[
 \E g\le\frac1n\sum_{i=1}^ng(X_i,Y_i)
 +2\widehat{\mathfrak R}_S(\mathcal G)
 +3\sqrt{\frac{\log(2/\delta)}{2n}};
\]
we use the empirical form recorded in \citet{Mohri2018}.  Apply it to
$\mathcal G=\{(x,y)\mapsto\ell(y,f(x)):f\in\cF_{z,B}^{(q)}\}$.
For the realized sample define
$\psi_i(t)=\ell(Y_i,t)-\ell(Y_i,0)$.  Then $\psi_i(0)=0$ and each $\psi_i$ is
$L_\ell$-Lipschitz.  The scalar contraction inequality gives
\begin{align*}
 \widehat{\mathfrak R}_S(\mathcal G)
 &=\frac1n\E_\eps\sup_f\sum_i\eps_i\psi_i(f(X_i))\\
 &\le L_\ell\Rad_S(\cF_{z,B}^{(q)}),
\end{align*}
because the removed term $\sum_i\eps_i\ell(Y_i,0)$ is independent of $f$
and has expectation zero.  The upper endpoint of
\eqref{eq:main-two-sided} is
$\Rad_S(\cF_{z,B}^{(q)})\le
B\sqrt{\widehat M_{q,S}(z)/n}$, which proves
\eqref{eq:fixed-gen-bound}.

If $z$ is random independently of $S$, condition on its value.  The preceding
probability statement holds with conditional probability at least
$1-\delta$ for every fixed value of $z$.  Integrating this conditional bound
over the law of $z$ gives the same unconditional probability, without moving
a sample-dependent representation choice through a Rademacher expectation.
\end{proof}

\begin{proposition}[Scale-balanced terminal rescaling]
\label{prop:scale-balance}
Fix $q\in\{1,2\}$. Let $z:\cX\to\R^m$, let $c>0$, and define
\begin{align*}
  H_z&=\overline{\operatorname{span}}\{\Phi_{q,m}(z(x)):x\in\cX\},\\
  H_{cz}&=\overline{\operatorname{span}}\{\Phi_{q,m}(cz(x)):x\in\cX\}.
\end{align*}
Because $K_q(cu,cv)=cK_q(u,v)$, there is a unique surjective isometry
$J_c:H_z\to H_{cz}$ satisfying
\begin{equation}
  J_c\Phi_{q,m}(z(x))=c^{-1/2}\Phi_{q,m}(cz(x))
  \qquad(x\in\cX).
  \label{eq:scale-isometry}
\end{equation}
For every $a\in H_z$, define $a_c=c^{-1/2}J_ca$.
Then
\begin{equation}
  \inner{a_c}{\Phi_{q,m}(cz(x))}
  =\inner{a}{\Phi_{q,m}(z(x))}
  \quad\text{for all }x,
  \qquad
  \norm{a_c}=c^{-1/2}\norm{a}.
  \label{eq:scale-compensation}
\end{equation}
Consequently, under this compensating terminal rescaling,
\begin{equation}
  \norm{a_c}^2\widehat M_{q,S}(cz)
  =\norm{a}^2\widehat M_{q,S}(z).
  \label{eq:scale-invariant-product}
\end{equation}
\end{proposition}

\begin{proof}
The one-homogeneity of the norm gives
$K_q(cu,cv)=cK_q(u,v)$.
Hence, for every $x,x'\in\cX$,
\begin{align*}
  &\inner{
    c^{-1/2}\Phi_{q,m}(cz(x))
  }{
    c^{-1/2}\Phi_{q,m}(cz(x'))
  }_{\cH_{q,m}}\\
  &\qquad=c^{-1}K_q(cz(x),cz(x'))
  =K_q(z(x),z(x'))
  =\inner{\Phi_{q,m}(z(x))}{\Phi_{q,m}(z(x'))}.
\end{align*}
Therefore the rule in \eqref{eq:scale-isometry} preserves inner products on the algebraic span.
It extends uniquely by continuity to an isometry $J_c:H_z\to H_{cz}$, and its range contains a dense spanning set of $H_{cz}$, so the range is all of $H_{cz}$.
Since \eqref{eq:scale-isometry} is equivalent to
$\Phi_{q,m}(cz(x))=\sqrt c\,J_c\Phi_{q,m}(z(x))$,
\begin{align*}
  \inner{a_c}{\Phi_{q,m}(cz(x))}
  &=\inner{c^{-1/2}J_ca}{\sqrt c\,J_c\Phi_{q,m}(z(x))}\\
  &=\inner{a}{\Phi_{q,m}(z(x))}.
\end{align*}
The norm identity follows from the isometry.
Finally,
$\widehat M_{q,S}(cz)=c\widehat M_{q,S}(z)$, and substitution proves
\eqref{eq:scale-invariant-product}.
\end{proof}

% Flexible appendix floats; the boundary before Appendix D is retained.
\section{Supplementary Experiments and Reproducibility}
\label{app:experiments}\label{app:extended-diagnostics}

This section expands the controlled evidence summarized in \Cref{sec:experiments}. The three principal protocols are parallel: exact finite-sample constructions, same-sample representation-family experiments, and frozen-transfer experiments. The earlier digits study is retained only as an auxiliary mechanism diagnostic for the fixed-head quantities. None of these protocols is used to claim an end-to-end state-of-the-art system.

The experiments are designed to check the quantities appearing in the theory rather than to establish a new benchmark. The protocols and results are documented below; experimental code and machine-readable outputs are separate from this manuscript project.
We use the handwritten-digits data distributed with scikit-learn~\citep{Pedregosa2011}, normalize pixels to $[0,1]$, and form disjoint stratified source/probe/validation/test splits containing approximately $50/25/12.5/12.5$ percent of the data.
For each of three seeds, a $64$--$128$--$64$--$32$ ReLU feature extractor with a linear $10$-class head is trained for $60$ epochs on the source split only, using Adam in PyTorch~\citep{Paszke2019}.
We then freeze the $32$-dimensional representation and use only the disjoint target splits for all downstream fits.
The logistic linear probe is fitted on the union of the probe and validation splits.
For each additive or L\'evy--Brownian kernel-ridge head, the probe split is used for fitting candidate models, the validation split selects the regularization parameter, and the selected model is refitted on probe plus validation data.
The test split is untouched until final evaluation.

\subsection{Sharpness of the trace characterization}

On $400$ frozen downstream-fit features per seed, sampled without replacement from the union of the probe and validation representations, we estimate
\[
  \frac1n\E_\eps\sqrt{\eps^\top K\eps}
\]
using $600$ Rademacher draws and compare it with the universal lower and upper bounds from \Cref{thm:main-fixed} for $B=1$.
As shown in \Cref{fig:rad-tightness}, the Monte Carlo value lies well inside the guaranteed interval and is close to the trace upper bound for both kernels.
The mean ratios of Monte Carlo complexity to the upper bound are approximately $0.93$ for $q=1$ and $0.92$ for $q=2$.
The Gram-sensitive lower bound from \Cref{prop:main-gram} is also reported; it adapts to the realized off-diagonal Gram energy and strictly improves the worst-case universal lower endpoint and moves toward the Monte Carlo value.

\begin{figure}[!htbp]
  \centering
  \includegraphics[width=0.524615\linewidth]{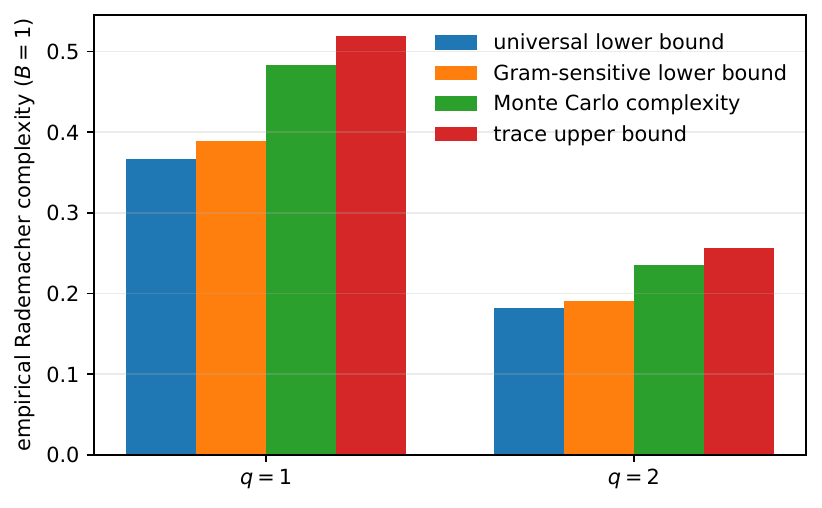}
  \caption{Monte Carlo empirical Rademacher complexity of frozen Brownian feature heads compared with the universal lower endpoint, the Gram-sensitive refinement in \Cref{prop:main-gram}, and the trace upper endpoint. Bars show means over three seeds.}
  \label{fig:rad-tightness}
\end{figure}

\subsection{Activation mass versus layer-norm envelopes}

At every fifth epoch, we record the empirical final-layer activation mass and the bias-free product envelope from \eqref{eq:biasfree-product-rad}.
The network contains biases, so this product is displayed only as a diagnostic proxy; the exact affine envelope is given by \eqref{eq:bias-product-rad}.
At the final epoch, the $\ell_1$ product proxy is on average about $110$ times the observed $\ell_1$ activation mass, while the spectral proxy is about $4.0$ times the observed $\ell_2$ mass.
The centered Euclidean dispersion is about $38\%$ of the uncentered mass, and its bias-free product envelope is approximately $4.1$ times the observed dispersion.
\Cref{fig:mass-l1,fig:mass-l2,fig:centered-l2} illustrate that the empirical quantities characterized by \Cref{thm:main-fixed} and \Cref{cor:main-centered} can be substantially more informative than worst-case layerwise propagation estimates.

\begin{figure}[!htbp]
  \centering
  \begin{subfigure}[t]{0.46\linewidth}
    \centering
    \includegraphics[width=0.846154\linewidth]{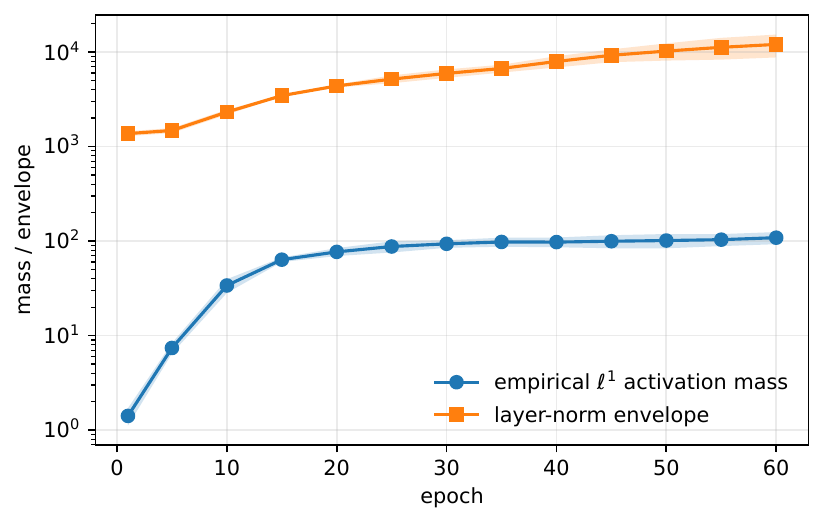}
    \caption{$\ell_1$ activation mass and induced-$\ell_1$ proxy.}
    \label{fig:mass-l1}
  \end{subfigure}\hfill
  \begin{subfigure}[t]{0.46\linewidth}
    \centering
    \includegraphics[width=0.846154\linewidth]{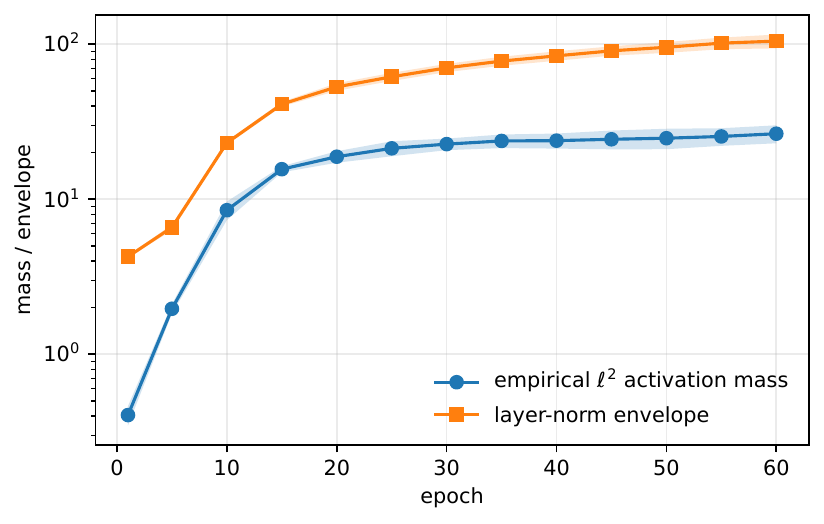}
    \caption{$\ell_2$ activation mass and spectral proxy.}
    \label{fig:mass-l2}
  \end{subfigure}
  \caption{Empirical activation mass and layer-norm product proxies during training. Both vertical axes are logarithmic; ribbons show one standard deviation over three seeds.}
  \label{fig:mass-envelope-pair}
\end{figure}

\begin{figure}[!htbp]
  \centering
  \includegraphics[width=0.440000\linewidth]{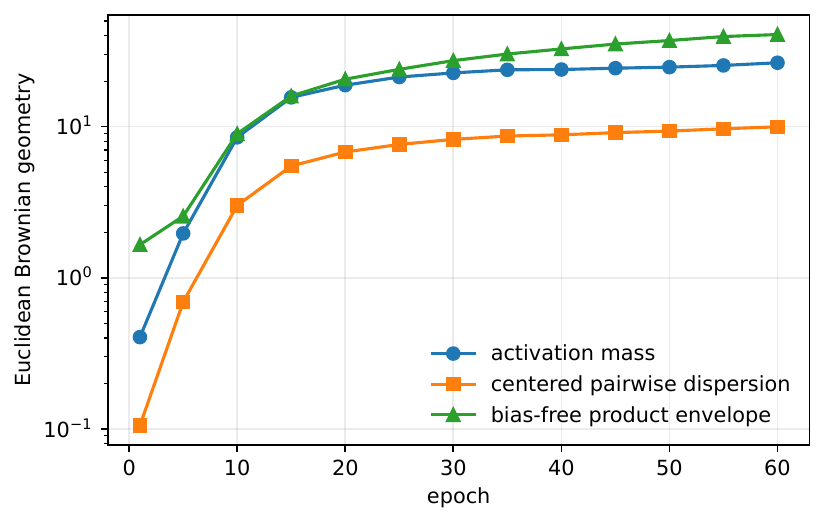}
  \caption{Euclidean activation mass, centered pairwise Brownian dispersion, and the bias-free centered product envelope from \Cref{cor:main-centered}. The feature extractor is evaluated on a probe split independent of its training data.}
  \label{fig:centered-l2}
\end{figure}

\subsection{Predictive performance of Brownian heads}

For each fitted $C$-output kernel head, let $\widehat{\mathsf A}_q$ be its
Hilbert--Schmidt output operator, let $S_{\mathrm{fit}}$ be the union of the
probe and validation splits, let $n_{\mathrm{fit}}=\abs{S_{\mathrm{fit}}}$, and set
\begin{equation}
  \widehat M_{q,\mathrm{fit}}
  :=\frac1{n_{\mathrm{fit}}}
  \sum_{x\in S_{\mathrm{fit}}}\norm{z(x)}_q,
  \qquad
  \Pi_q
  :=\norm{\widehat{\mathsf A}_q}_{\HS}
  \sqrt{\frac{\widehat M_{q,\mathrm{fit}}}{n_{\mathrm{fit}}}}.
  \label{eq:experimental-head-scale}
\end{equation}

\begin{table}[!htbp]
  \centering
  \caption{Digits test accuracy and the fitted scalarized head-scale diagnostic
  $\Pi_q$ from \eqref{eq:experimental-head-scale}, reported as mean $\pm$ standard deviation over three seeds.
  The diagnostic is shown only for the kernel heads and is not directly comparable across different output parameterizations.
  Because the experiment has $C=10$ outputs, the vector-valued upper bound in
  \Cref{prop:vector-hs-extension} contains the additional fixed factor $\sqrt{10}$; $\Pi_q$ is therefore not itself a certified multiclass bound.}
  \label{tab:digits}
  \begin{tabular}{lcc}
    \toprule
    Method & Test accuracy (\%) & $\Pi_q$ \\
    \midrule
    Source-trained neural head & $92.74\pm2.68$ & -- \\
    Frozen-feature linear probe & $92.89\pm3.36$ & -- \\
    Additive Brownian head ($q=1$) & $94.22\pm2.35$ & $0.662\pm0.162$ \\
    L\'evy--Brownian head ($q=2$) & $94.81\pm2.82$ & $0.964\pm0.246$ \\
    \bottomrule
  \end{tabular}
\end{table}

The L\'evy--Brownian head obtains the highest mean accuracy in this small experiment, with the additive head also exceeding the two linear-head baselines on average.
These results do not establish superiority on large-scale tasks; they show that the proposed terminal hypothesis spaces are computationally usable and need not sacrifice predictive performance in a frozen-feature setting.

\subsection{Regularization paths and spectral tails}
\label{subsec:empirical-capacity-path}

For each frozen representation and each $q\in\{1,2\}$, we refit the multiclass
Brownian kernel-ridge head over the predetermined regularization grid and
record the Hilbert--Schmidt capacity
\begin{equation}
  \mathfrak C_{q,S}(\lambda)
  :=\norm{\widehat{\mathsf A}_{q,\lambda}}_{\HS}
  \sqrt{\frac{\widehat M_{q,S}}n}.
  \label{eq:empirical-capacity-path}
\end{equation}
The probe and validation observations are used for fitting, and the held-out
test split is used only to evaluate the displayed path.
As guaranteed by \Cref{prop:main-ridge}, the capacity is
nonincreasing in $\lambda$ for every seed and both Brownian geometries.
\Cref{fig:capacity-path} plots test error against this exact complexity
coordinate rather than parameter count; the curves are diagnostic and are not
claimed to satisfy a universal deterministic-equivalent law.

The same fitted Gram matrices provide a direct spectral diagnostic for
\Cref{prop:main-localized} and
\Cref{prop:spectral-tail-chain-compression}.
\Cref{fig:spectral-tail} reports the normalized residual trace
$\sum_{j>r}\lambda_j/\sum_j\lambda_j$.
Its decay quantifies the approximation available to a rank-$r$ feature
projection; the experiment concerns terminal Brownian features and does not
simulate the distinct recursive-chain architecture.

\begin{figure}[!htbp]
  \centering
  \begin{subfigure}[t]{0.485\linewidth}
    \centering
    \includegraphics[width=0.846154\linewidth]{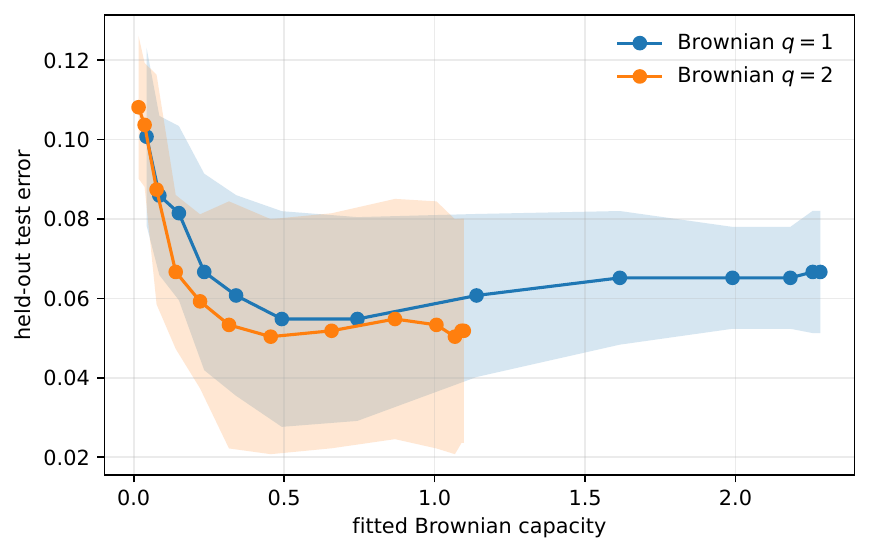}
    \caption{Held-out error along the fitted-capacity path.}
    \label{fig:capacity-path}
  \end{subfigure}\hfill
  \begin{subfigure}[t]{0.485\linewidth}
    \centering
    \includegraphics[width=0.846154\linewidth]{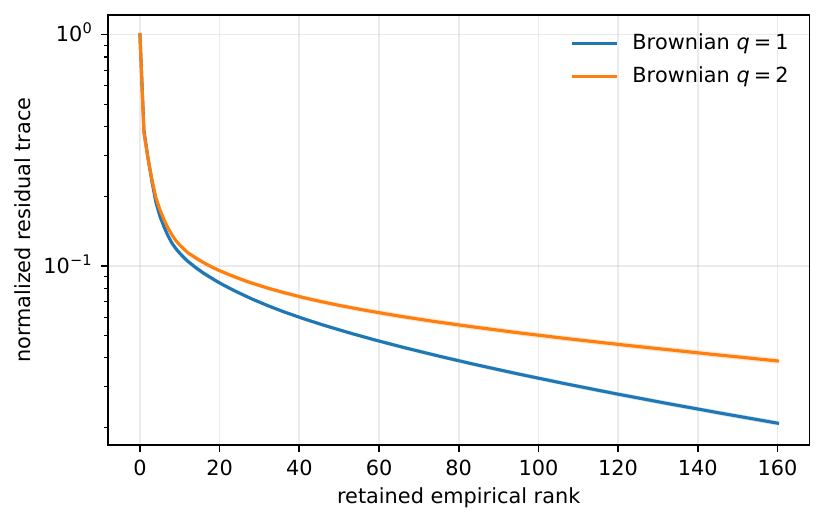}
    \caption{Normalized empirical Brownian Gram tails.}
    \label{fig:spectral-tail}
  \end{subfigure}
  \caption{Regularization-path and spectral diagnostics. In (a), points correspond to the predetermined regularization grid; markers show means over three seeds and ribbons show one standard deviation. In (b), faster residual-trace decay yields smaller localized complexity and smaller rank-truncation error.}
  \label{fig:path-and-tail}
\end{figure}

\subsection{Digits reproducibility details}
\label{app:digits-reproducibility}

The reproducible script \path{experiments/run_digits.py} uses only
\path{sklearn.datasets.load_digits}.  A $64\to128\to64\to32\to10$ ReLU network is trained for $60$ source-half epochs with Adam (learning rate $10^{-3}$, weight decay $10^{-4}$, batch size $128$).  The disjoint target half is split $1/2$, $1/4$, $1/4$ into probe, validation, and test sets.  Brownian kernel-ridge regularization is selected on the validation split from $\{10^{-5},10^{-4.5},\ldots,10^1\}$ and refit on probe plus validation; test observations never enter fitting or selection.  Seeds $0,1,2$, raw results, spectral diagnostics, figures, histories, and exact package versions are included in the archive.

\subsection{Exact finite-sample constructions}
\label{app:theory-core-experiments}

All quantities in this subsection are deterministic evaluations of proved finite-sample formulas. No optimizer, train/test split, or selected random seed enters the calculations. For $S_m=\sum_{i=1}^m\varepsilon_i$, the block-union class with $k$ equal blocks satisfies
\begin{equation}
 \rho_n(\mathcal T_k)
 =\frac{k\,\mathbb E|S_{n/k}|+\mathbb E|S_n|}{2n},
 \qquad |\mathcal T_k|=2^k,
 \qquad \operatorname{VCdim}(\mathcal T_k)=k.
 \label{eq:app-block-union-exact}
\end{equation}
The common term $\mathbb E|S_n|/(2n)$ does not depend on $k$; subtracting it isolates the representation-selection growth.
\begin{table}[!htbp]
\centering
\small
\begin{tabular}{lccc}
\toprule
quantity & fitted slope & $R^2$ & deterministic check \\
\midrule
finite traces vs. $\sqrt{\log N/n}$ & 0.9785 & 0.99982 & slope in $[0.95,1.02]$ \\
VC traces vs. $\sqrt{V/n}$ & 0.9804 & 0.99981 & slope in $[0.95,1.02]$ \\
\bottomrule
\end{tabular}
\caption{Deterministic theorem-diagnostic summary.}
\label{tab:app-phase01-acceptance}
\end{table}

For the mass-only counterexample, $\mathcal T$ is the class of all $n/2$-subsets. Every representation $z_T(x_i)=A\mathbf 1_{\{i\in T\}}$ has mass $A/2$, while for every Rademacher vector one can choose a subset containing $n/2$ signs of one majority value. Hence the family complexity is exactly $B\sqrt A/2$ for every even $n$, whereas a fixed member has trace upper $B\sqrt{A/(2n)}$. The ratio is $\sqrt{n/2}$.

\begin{figure}[!htbp]
\centering
\includegraphics[width=0.829231\linewidth]{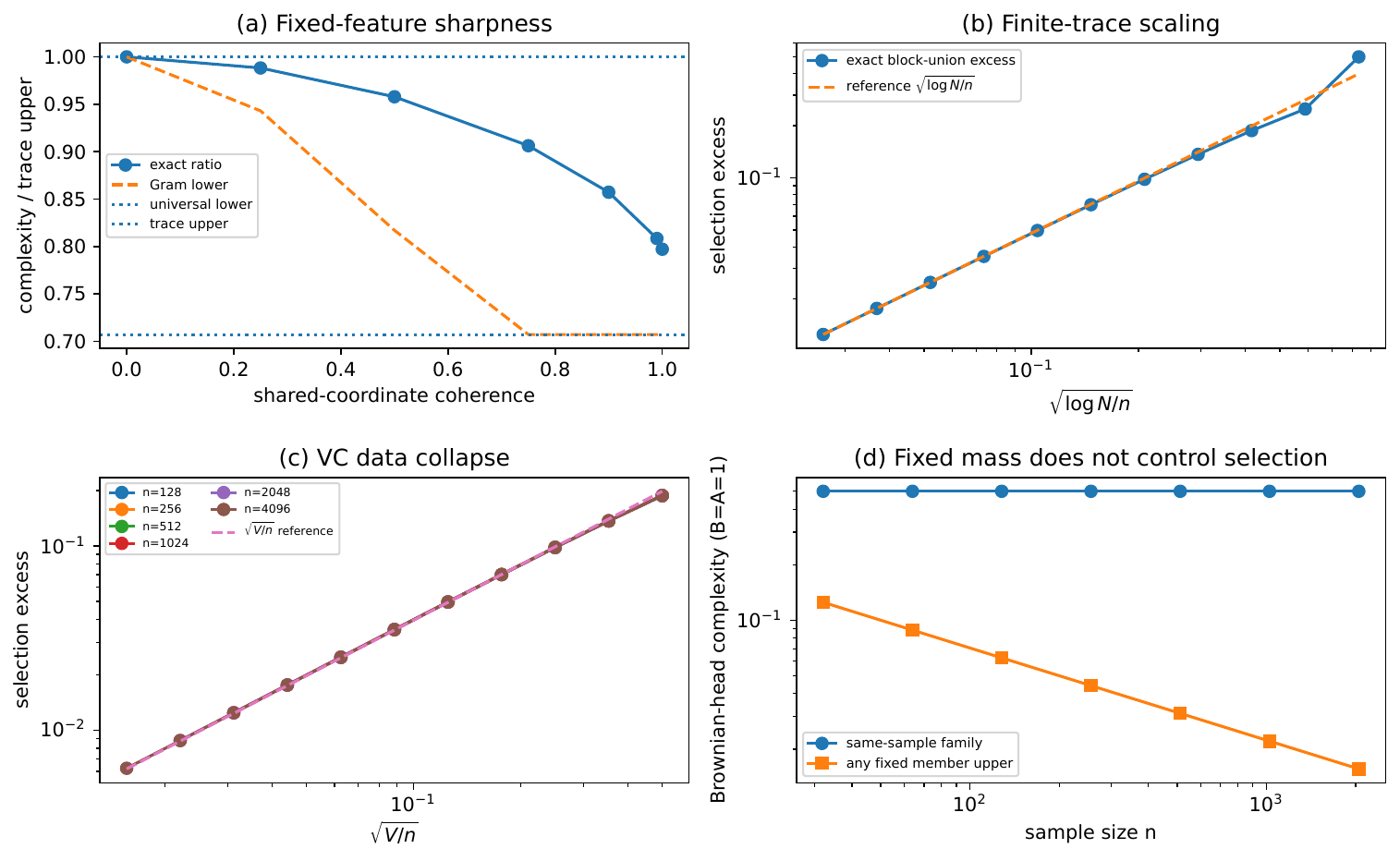}
\caption{Complete deterministic diagnostics. The panels show the optimal fixed-feature endpoints, finite-trace $\sqrt{\log N/n}$ scaling, VC $\sqrt{V/n}$ scaling, and the equal-mass failure under same-sample selection.}
\label{fig:app-theory-core-summary}
\end{figure}

\begin{figure}[!htbp]
\centering
\includegraphics[width=0.575385\linewidth]{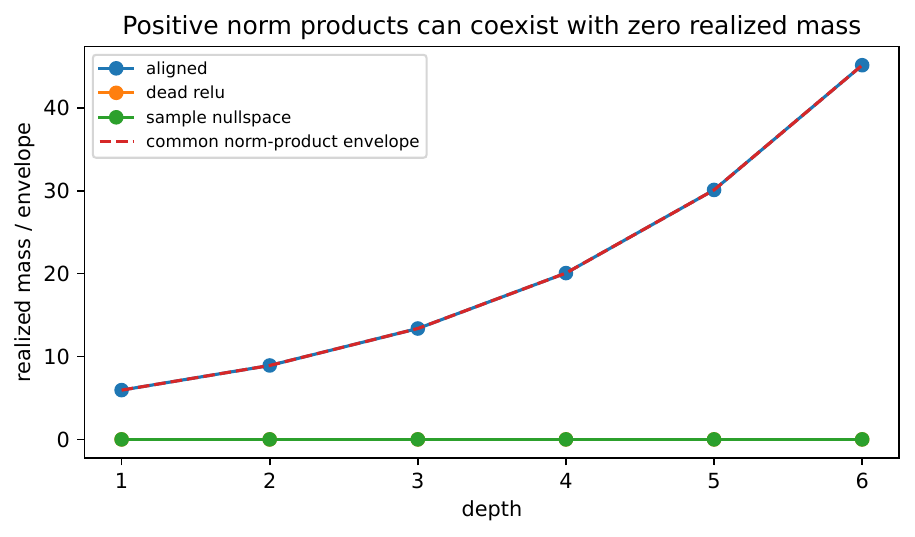}
\caption{Rank-deficiency construction. The aligned, sample-nullspace, and dead-ReLU networks have the same positive product of operator norms, while the latter two have zero realized sample mass. Thus the norm product is a worst-case envelope and has no universal realized converse.}
\label{fig:app-rank-deficiency}
\end{figure}

The exact summaries and vector figures are reported here; the underlying machine-readable tables are separate from this manuscript project.

\subsection{Same-sample representation-family experiments}
\label{app:same-sample-learned}

This controlled experiment complements the exact trace constructions with label-dependent representation families on frozen ImageNet-pretrained ResNet-18 features. Each CIFAR context uses the official training/test split; each EuroSAT seed uses a stratified 80/20 split. From the training pool we draw a balanced sample of exactly $n=500$, evaluate on a balanced held-out subset of 2000 examples, corrupt 20\% of the training labels only, and never expose clean held-out labels to adapter generation or family selection. Training features are standardized coordinatewise using training means and standard deviations (with a deterministic variance floor), and every sample is then normalized to unit Euclidean norm; the same training transformation is applied to held-out features.

For candidate $c$, a class-contrast vector $v_c\in\mathbb R^C$ is sampled from a standard Gaussian, centered, and normalized. With corrupted labels $\widetilde y_i$, the scalar targets are $t_{c,i}=(v_c)_{\widetilde y_i}$, centered over the training sample, and the label-dependent direction is
\[s_c=\frac{X^\top t_c}{\|X^\top t_c\|_2}.\]
A width-12 adapter is then
\[z_c(x)=\operatorname{ReLU}((W_{0,c}+\eta_c a_cs_c^\top)x+b_c),\]
where the entries of $W_{0,c}$ are independent $\mathcal N(0,1/d)$ variables, $b_c\sim\mathcal N(0,0.08^2I)$, and $a_c$ is an independent normalized Gaussian width-12 direction. The predeclared strengths cycle through $\eta_c\in\{0.5, 1, 2, 4\}$ in candidate order. The 32 candidates are generated in that deterministic order from the context seed, and the nested families consist of the first $K\in\{1, 2, 4, 8, 16, 32\}$ candidates. Positive homogeneity rescales every adapter so that its empirical additive-Brownian activation mass is exactly one before any complexity comparison.

For each dataset-seed context, the same 128 Rademacher sign vectors are reused for every candidate and every nested family. The union complexity is the draw-average of the largest additive-Brownian canonical-feature norm over the first $K$ candidates, divided by $n$; the strongest fixed-member complexity is the largest corresponding candidatewise draw-average. Their difference is the selection excess. The threshold statistic is the draw-average of the largest absolute coordinate-threshold Rademacher sum over candidates and thresholds, divided by $n$. Scalar Brownian layer-cake quantities are evaluated exactly; only the secondary predictive head discretizes each coordinate into 6 equal-width threshold cells. Each candidate head chooses ridge regularization by training-set GCV over $\{0.0001, 0.001, 0.01, 0.1\}$, while family selection chooses the candidate with the smallest corrupted-training MSE among the first $K$ candidates. All random seeds, split hashes, and corruption hashes are stored in the machine-readable records.

The largest family has mean union-to-fixed ratio 1.396 and minimum ratio 1.338. At fixed $n=500$, the stored selection excess and $\sqrt{\log K}$ predictor are each divided separately by their own contextwise maximum; pooling these normalized curves gives a descriptive zero-intercept slope 0.8621 and $R^2=0.8795$. This is a curve-shape comparison only: it neither tests the dependence on $n$ nor estimates a universal theorem constant. The global union--threshold Spearman value is 0.977, but the within-context correlation equals one in all 15 nested contexts because both quantities are monotone in $K$; it is therefore retained only as a confounded diagnostic and not as independent evidence beyond family size.

\begin{figure}[!htbp]
\centering
\includegraphics[width=0.609231\linewidth]{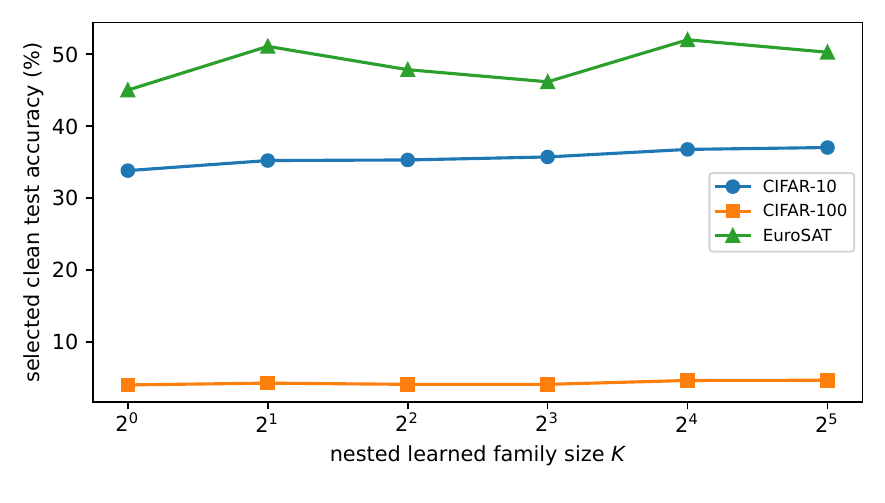}
\caption{Secondary predictive diagnostic: clean held-out accuracy of the candidate selected by corrupted-training MSE within each nested family. Clean test labels are used only for final evaluation.}
\label{fig:learned-selection-predictive}
\end{figure}

Complete candidate-level, family-level, and within-context monotonicity records are separate from this manuscript project.

\subsection{Frozen-transfer experiments}
\label{app:frozen-transfer}

We freeze the penultimate 512-dimensional representation of an ImageNet-pretrained ResNet-18 \citep{He2016DeepResidual} and fit five heads on identical balanced labeled subsets: linear ridge, RBF and Laplace kernel ridge, and the additive and L\'evy--Brownian heads. CIFAR-10 and CIFAR-100 \citep{Krizhevsky2009CIFAR} use their official test sets. EuroSAT \citep{Helber2019EuroSAT} uses a new stratified 80/20 split for each seed, with all methods paired on that split. The label budgets 500 and 1000 are nested, the representation never uses target labels, ridge regularization is selected by training-set generalized cross-validation, and clean test labels are accessed only for final evaluation. The $q=2$ head is the rotation-invariant Euclidean Brownian geometry emphasized by the projected-threshold theory; $q=1$ is retained as the coordinatewise alternative. \Cref{fig:main-frozen-transfer,tab:main-frozen-transfer} in the main text summarize paired accuracy differences across all six settings and the $n_{\rm lab}=1000$ results, respectively.

\begin{figure}[!htbp]
\centering
\includegraphics[width=0.829231\linewidth]{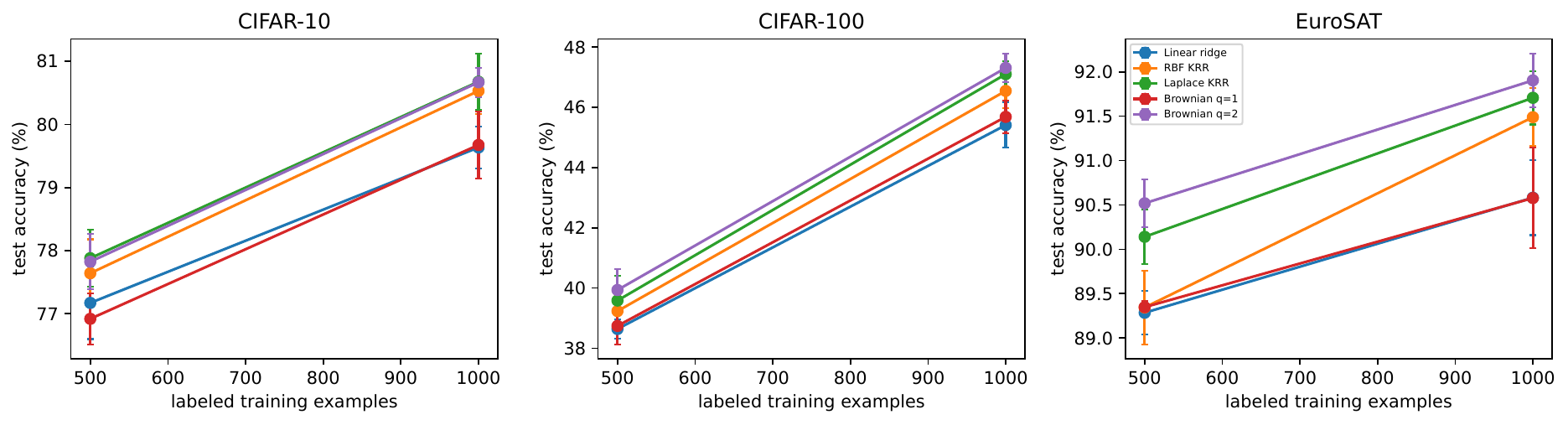}
\caption{Frozen-feature learning curves. Markers show five-seed means and error bars show one standard deviation.}
\label{fig:app-frozen-transfer-curves}
\end{figure}

\begin{table}[!htbp]
\caption{Complete frozen-transfer results for both labeled-data budgets. Values are mean $\pm$ standard deviation over five paired seeds.}
\label{tab:app-frozen-transfer-full}
\centering
\scriptsize
\begin{tabular}{llrrrrr}
\toprule
Dataset & Method & $n_{\rm lab}$ & Accuracy (\%) & Macro-F1 (\%) & Fit (s) & Predict (s) \\
\midrule
CIFAR-10 & Brownian $q=1$ & 500 & $76.92\pm0.40$ & $76.68\pm0.39$ & 0.01 & 0.01 \\
CIFAR-10 & Brownian $q=2$ & 500 & $77.82\pm0.44$ & $77.62\pm0.44$ & 0.01 & 0.01 \\
CIFAR-10 & Laplace KRR & 500 & $77.88\pm0.45$ & $77.65\pm0.45$ & 0.01 & 0.01 \\
CIFAR-10 & Linear ridge & 500 & $77.17\pm0.57$ & $76.84\pm0.57$ & 0.01 & 0.01 \\
CIFAR-10 & RBF KRR & 500 & $77.64\pm0.54$ & $77.47\pm0.54$ & 0.01 & 0.02 \\
CIFAR-10 & Brownian $q=1$ & 1000 & $79.67\pm0.53$ & $79.49\pm0.54$ & 0.05 & 0.03 \\
CIFAR-10 & Brownian $q=2$ & 1000 & $80.67\pm0.23$ & $80.54\pm0.23$ & 0.05 & 0.02 \\
CIFAR-10 & Laplace KRR & 1000 & $80.68\pm0.45$ & $80.53\pm0.46$ & 0.05 & 0.03 \\
CIFAR-10 & Linear ridge & 1000 & $79.63\pm0.33$ & $79.43\pm0.34$ & 0.04 & 0.01 \\
CIFAR-10 & RBF KRR & 1000 & $80.53\pm0.37$ & $80.38\pm0.38$ & 0.05 & 0.03 \\
CIFAR-100 & Brownian $q=1$ & 500 & $38.74\pm0.61$ & $36.06\pm0.62$ & 0.01 & 0.02 \\
CIFAR-100 & Brownian $q=2$ & 500 & $39.93\pm0.70$ & $37.56\pm0.65$ & 0.01 & 0.02 \\
CIFAR-100 & Laplace KRR & 500 & $39.58\pm0.83$ & $37.55\pm0.79$ & 0.01 & 0.02 \\
CIFAR-100 & Linear ridge & 500 & $38.64\pm0.33$ & $35.49\pm0.46$ & 0.01 & 0.01 \\
CIFAR-100 & RBF KRR & 500 & $39.23\pm0.50$ & $36.52\pm0.51$ & 0.01 & 0.02 \\
CIFAR-100 & Brownian $q=1$ & 1000 & $45.68\pm0.54$ & $43.20\pm0.60$ & 0.05 & 0.03 \\
CIFAR-100 & Brownian $q=2$ & 1000 & $47.31\pm0.48$ & $45.22\pm0.51$ & 0.05 & 0.03 \\
CIFAR-100 & Laplace KRR & 1000 & $47.10\pm0.43$ & $45.23\pm0.50$ & 0.05 & 0.03 \\
CIFAR-100 & Linear ridge & 1000 & $45.41\pm0.75$ & $42.57\pm0.92$ & 0.04 & 0.01 \\
CIFAR-100 & RBF KRR & 1000 & $46.55\pm0.56$ & $44.09\pm0.68$ & 0.05 & 0.03 \\
EuroSAT & Brownian $q=1$ & 500 & $89.35\pm0.07$ & $88.82\pm0.07$ & 0.01 & 0.01 \\
EuroSAT & Brownian $q=2$ & 500 & $90.52\pm0.27$ & $90.09\pm0.29$ & 0.01 & 0.01 \\
EuroSAT & Laplace KRR & 500 & $90.14\pm0.31$ & $89.68\pm0.34$ & 0.01 & 0.01 \\
EuroSAT & Linear ridge & 500 & $89.29\pm0.25$ & $88.83\pm0.29$ & 0.01 & 0.00 \\
EuroSAT & RBF KRR & 500 & $89.34\pm0.42$ & $88.83\pm0.46$ & 0.01 & 0.01 \\
EuroSAT & Brownian $q=1$ & 1000 & $90.58\pm0.57$ & $90.15\pm0.58$ & 0.05 & 0.01 \\
EuroSAT & Brownian $q=2$ & 1000 & $91.90\pm0.30$ & $91.56\pm0.32$ & 0.05 & 0.01 \\
EuroSAT & Laplace KRR & 1000 & $91.71\pm0.30$ & $91.34\pm0.32$ & 0.05 & 0.01 \\
EuroSAT & Linear ridge & 1000 & $90.58\pm0.42$ & $90.20\pm0.42$ & 0.04 & 0.00 \\
EuroSAT & RBF KRR & 1000 & $91.49\pm0.33$ & $91.12\pm0.34$ & 0.05 & 0.02 \\
\bottomrule
\end{tabular}
\end{table}

\begin{figure}[!htbp]
\centering
\includegraphics[width=0.693846\linewidth]{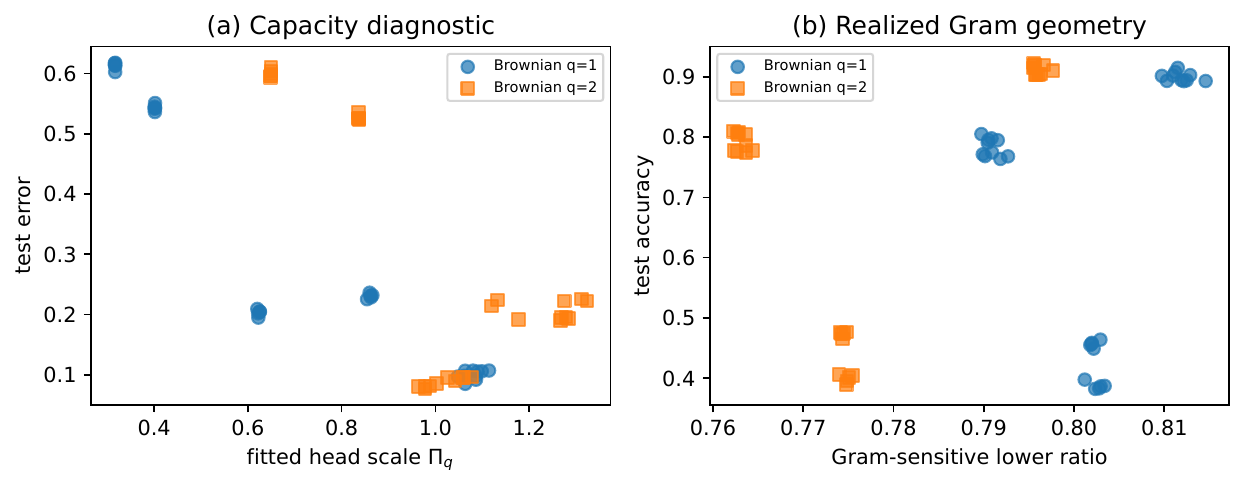}
\caption{Brownian fitted-head and realized-Gram diagnostics. These are finite-sample mechanism diagnostics, not deterministic-equivalent test-risk curves.}
\label{fig:app-frozen-transfer-capacity}
\end{figure}

All predefined completeness, split-pairing, finite-value, and Brownian positive-semidefiniteness checks passed. The fixed L\'evy--Brownian head has average rank 1.333, mean gap to the best head 0.011 percentage points, and a positive mean gain over linear ridge in all six settings. Complete records include split hashes, selected ridge parameters, accuracy, macro-F1, balanced accuracy, fit time, and prediction time.

\paragraph{Paired differences.}
Intervals below are descriptive paired percentile-bootstrap summaries over five seeds; no multiplicity-corrected significance claim is made.
\begin{table}[!htbp]
\caption{Paired gains of the Brownian heads over linear ridge. Intervals are descriptive paired percentile-bootstrap summaries over five seeds; no multiplicity-corrected significance claim is made.}
\label{tab:app-brownian-linear-paired}
\centering
\scriptsize
\begin{tabular}{llrrr}
\toprule
Dataset & Brownian head & $n_{\rm lab}$ & Gain over linear (pp) & Interval (pp) \\
\midrule
CIFAR-10 & Brownian $q=1$ & 500 & -0.25 & [-0.48, -0.05] \\
CIFAR-10 & Brownian $q=2$ & 500 & 0.65 & [0.49, 0.82] \\
CIFAR-10 & Brownian $q=1$ & 1000 & 0.04 & [-0.20, 0.25] \\
CIFAR-10 & Brownian $q=2$ & 1000 & 1.03 & [0.92, 1.13] \\
CIFAR-100 & Brownian $q=1$ & 500 & 0.10 & [-0.26, 0.54] \\
CIFAR-100 & Brownian $q=2$ & 500 & 1.29 & [0.90, 1.69] \\
CIFAR-100 & Brownian $q=1$ & 1000 & 0.27 & [0.02, 0.53] \\
CIFAR-100 & Brownian $q=2$ & 1000 & 1.90 & [1.53, 2.18] \\
EuroSAT & Brownian $q=1$ & 500 & 0.06 & [-0.13, 0.30] \\
EuroSAT & Brownian $q=2$ & 500 & 1.23 & [1.01, 1.44] \\
EuroSAT & Brownian $q=1$ & 1000 & -0.00 & [-0.27, 0.30] \\
EuroSAT & Brownian $q=2$ & 1000 & 1.32 & [1.12, 1.53] \\
\bottomrule
\end{tabular}
\end{table}

\begin{table}[!htbp]
\centering
\scriptsize
\caption{L\'evy--Brownian $q=2$ paired differences against the two principal non-Brownian baselines.}
\label{tab:app-q2-pairwise}
\begin{tabular}{llrrr}
\toprule
Dataset & Comparator & $n_{\rm lab}$ & Difference (pp) & Interval (pp) \\
\midrule
CIFAR-10 & Laplace KRR & 500 & -0.06 & [-0.13, 0.06] \\
CIFAR-10 & Linear ridge & 500 & 0.65 & [0.49, 0.82] \\
CIFAR-10 & Laplace KRR & 1000 & -0.01 & [-0.25, 0.22] \\
CIFAR-10 & Linear ridge & 1000 & 1.03 & [0.92, 1.13] \\
CIFAR-100 & Laplace KRR & 500 & 0.35 & [0.18, 0.61] \\
CIFAR-100 & Linear ridge & 500 & 1.29 & [0.90, 1.69] \\
CIFAR-100 & Laplace KRR & 1000 & 0.21 & [0.05, 0.36] \\
CIFAR-100 & Linear ridge & 1000 & 1.90 & [1.53, 2.18] \\
EuroSAT & Laplace KRR & 500 & 0.38 & [0.30, 0.46] \\
EuroSAT & Linear ridge & 500 & 1.23 & [1.01, 1.44] \\
EuroSAT & Laplace KRR & 1000 & 0.20 & [0.10, 0.27] \\
EuroSAT & Linear ridge & 1000 & 1.32 & [1.12, 1.53] \\
\bottomrule
\end{tabular}
\end{table}

% Keep supplementary-experiment floats before the kernel-chain appendix.
\FloatBarrier

% Appendix-local display spacing; keep short derivations together.
\allowdisplaybreaks[1]
\setlength{\jot}{2pt}
\setlength{\abovedisplayskip}{5pt plus 2pt minus 2pt}
\setlength{\belowdisplayskip}{5pt plus 2pt minus 2pt}
\setlength{\belowdisplayshortskip}{3pt plus 2pt minus 2pt}
\section{Supporting Recursive Brownian-Chain Material}
\label{sec:brownian-chains}

This appendix records material substantially related to BKL, not a separate contribution. It is not used in the terminal-head theorem proofs. We relate fixed-measure BKL recursions to Hilbertian kernel chains and record recursive Brownian lifts, dyadic propagation, affine energy, and sample compression in the present notation.

\subsection{Fixed-measure BKL recursions are Hilbertian kernel chains}

Let $\mathcal G$ be an RKHS on $\cX$ with kernel $k$, and write
$k_x:=k(\cdot,x)\in\mathcal G$.
Let $\mathcal U\subseteq\mathcal G$ be Borel and let $\mu$ be a finite positive Borel measure on $\mathcal U$ satisfying
\begin{equation}
  \int_{\mathcal U}\norm{u}_{\mathcal G}\dd\mu(u)<\infty.
  \label{eq:chain-link-first-moment}
\end{equation}
Define the Brownian link kernel on $\mathcal G$ by
\begin{equation}
  \widetilde k_{\mu}(h,h')
  :=
  \int_{\mathcal U}
  \kb\!\left(
    \inner{u}{h}_{\mathcal G},
    \inner{u}{h'}_{\mathcal G}
  \right)
  \dd\mu(u),
  \qquad h,h'\in\mathcal G,
  \label{eq:brownian-link-kernel}
\end{equation}
and pull it back to $\cX$ through the canonical feature map $x\mapsto k_x$:
\begin{equation}
  k^{+}_{\mu}(x,x')
  :=\widetilde k_{\mu}(k_x,k_{x'}).
  \label{eq:chain-kernel-recursion}
\end{equation}
This is the Hilbert-space specialization of the kernel-chaining operation of
\citet{HeeringaSpekBrune2025}.

\begin{proposition}[Fixed-measure BKL--kernel-chain correspondence]
\label{prop:bkl-chain-correspondence}
Under \eqref{eq:chain-link-first-moment}, the integral in
\eqref{eq:brownian-link-kernel} is finite for every $h,h'\in\mathcal G$, and
$\widetilde k_\mu$ is a positive-definite kernel on $\mathcal G$.
Consequently, $k^+_\mu$ is a positive-definite kernel on $\cX$ and satisfies
\begin{equation}
  k^{+}_{\mu}(x,x')
  =
  \int_{\mathcal U}
  \kb\bigl(u(x),u(x')\bigr)
  \dd\mu(u).
  \label{eq:bkl-recursion-from-chain}
\end{equation}
Thus every fixed admissible sequence of BKL measures generates, link by link, a Hilbertian reproducing-kernel chain with Brownian link kernels.
In the indexing of an $L$-level BKL, the initial RKHS is followed by $L-1$ such links.
\end{proposition}

\begin{proof}
The scalar Brownian kernel obeys
\begin{equation}
  0\le \kb(s,t)\le \min\{\abs{s},\abs{t}\}
  \qquad(s,t\in\R).
  \label{eq:brownian-min-bound}
\end{equation}
If $st\le0$, then $\abs{s-t}=\abs{s}+\abs{t}$ and hence
$\kb(s,t)=0$.
If $st>0$, then
$\abs{s-t}=\bigl|\abs{s}-\abs{t}\bigr|$, so
$\kb(s,t)=\min\{\abs{s},\abs{t}\}$.
By Cauchy--Schwarz and \eqref{eq:brownian-min-bound},
\begin{align}
  0
  &\le
  \kb\!\left(
    \inner{u}{h}_{\mathcal G},
    \inner{u}{h'}_{\mathcal G}
  \right) \\
  &\le
  \norm{u}_{\mathcal G}
  \min\{\norm{h}_{\mathcal G},\norm{h'}_{\mathcal G}\}.
  \label{eq:link-integrability-bound}
\end{align}
The right-hand side is $\mu$-integrable by
\eqref{eq:chain-link-first-moment}, so the link kernel is finite.
The integrand is Borel measurable because
$u\mapsto\inner{u}{h}_{\mathcal G}$ is continuous and $\kb$ is continuous.

To prove positive definiteness, fix
$h_1,\ldots,h_N\in\mathcal G$ and $c_1,\ldots,c_N\in\R$.
The finite sum below is integrable by \eqref{eq:link-integrability-bound}, and therefore
\begin{align*}
  \sum_{r,s=1}^N c_rc_s\widetilde k_\mu(h_r,h_s)
  &=
  \int_{\mathcal U}
  \sum_{r,s=1}^N c_rc_s
  \kb\!\left(
    \inner{u}{h_r}_{\mathcal G},
    \inner{u}{h_s}_{\mathcal G}
  \right)
  \dd\mu(u)\\
  &\ge0.
\end{align*}
For each fixed $u$, the inner sum is nonnegative because it is a Gram quadratic form for the pullback of the positive-definite kernel $\kb$ through the bounded functional
$h\mapsto\inner{u}{h}_{\mathcal G}$.
Hence $\widetilde k_\mu$ is positive definite.
Its pullback $k^+_\mu$ through $x\mapsto k_x$ is therefore positive definite as well.
Finally, the reproducing property gives
\[
  \inner{u}{k_x}_{\mathcal G}=u(x),
  \qquad
  \inner{u}{k_{x'}}_{\mathcal G}=u(x'),
\]
and substitution into \eqref{eq:chain-kernel-recursion} proves
\eqref{eq:bkl-recursion-from-chain}.
Iteration proves the final assertion.
\end{proof}

\begin{proposition}[Directional-moment propagation for fixed BKL links]
\label{prop:bkl-directional-moment}
Under the assumptions of \Cref{prop:bkl-chain-correspondence}, define the
$L^1(\mu)$ analysis norm
\begin{equation}
  \chi_{\mathcal G}(\mu)
  :=\sup_{\norm{h}_{\mathcal G}\le1}
  \int_{\mathcal U}\abs{\inner{u}{h}_{\mathcal G}}\dd\mu(u).
  \label{eq:bkl-directional-moment}
\end{equation}
Then
\begin{equation}
  \chi_{\mathcal G}(\mu)
  \le\int_{\mathcal U}\norm{u}_{\mathcal G}\dd\mu(u).
  \label{eq:bkl-directional-first-moment}
\end{equation}
Let
$d_0(x,x'):=\norm{k_x-k_{x'}}_{\mathcal G}$ and let $d_\mu^+$ be the
canonical feature distance of $k_\mu^+$.
Then
\begin{equation}
  d_\mu^+(x,x')^2
  =\int_{\mathcal U}\abs{u(x)-u(x')}\dd\mu(u)
  \le\chi_{\mathcal G}(\mu)d_0(x,x').
  \label{eq:bkl-directional-link-bound}
\end{equation}
Consequently, after $r$ fixed BKL links with directional moments
$\chi_1,\ldots,\chi_r$, the induced distances satisfy
\begin{equation}
  d_r(x,x')
  \le d_0(x,x')^{2^{-r}}
  \prod_{\ell=1}^r
  \chi_\ell^{2^{-(r-\ell+1)}}.
  \label{eq:bkl-directional-chain-bound}
\end{equation}
The directional factor can be strictly smaller than the usual first moment:
if $\mu$ is uniform on $s$ orthonormal unit vectors, then the right-hand side
of \eqref{eq:bkl-directional-first-moment} equals one whereas
$\chi_{\mathcal G}(\mu)=s^{-1/2}$.
\end{proposition}

\begin{proof}
For every $h$ with $\|h\|_{\mathcal G}\le1$,
$|\inner{u}{h}|\le\|u\|_{\mathcal G}$ pointwise.  Integrating and then
taking the supremum over such $h$ proves
\eqref{eq:bkl-directional-first-moment}, including the case in which the
right-hand side is zero.
Using the Brownian identity
$\kb(a,a)+\kb(b,b)-2\kb(a,b)=\abs{a-b}$ and the reproducing property,
\begin{align*}
  d_\mu^+(x,x')^2
  &=\int_{\mathcal U}\abs{u(x)-u(x')}\dd\mu(u)\\
  &=\int_{\mathcal U}
    \abs{\inner{u}{k_x-k_{x'}}_{\mathcal G}}\dd\mu(u)\\
  &\le\chi_{\mathcal G}(\mu)\norm{k_x-k_{x'}}_{\mathcal G},
\end{align*}
which proves \eqref{eq:bkl-directional-link-bound}.
At link $\ell$, the same argument gives
$d_\ell^2\le\chi_\ell d_{\ell-1}$.  We prove the closed form by induction
without dividing by any distance or directional moment.  For $r=1$ it reads
$d_1\le d_0^{1/2}\chi_1^{1/2}$.  If it holds at level $r-1$, then
\begin{align*}
 d_r
 &\le \chi_r^{1/2}d_{r-1}^{1/2}\\
 &\le d_0^{2^{-r}}
 \prod_{\ell=1}^{r-1}\chi_\ell^{2^{-(r-\ell+1)}}
 \chi_r^{1/2},
\end{align*}
which is \eqref{eq:bkl-directional-chain-bound}.  If $d_0=0$ or some
$\chi_\ell=0$, the one-step inequality forces all subsequent relevant
distances to be zero, so the same formula remains valid at the endpoints.
For the final claim, if $\mu=s^{-1}\sum_{j=1}^s\delta_{u_j}$ with
$(u_j)_{j=1}^s$ orthonormal, then
\[
  \chi_{\mathcal G}(\mu)
  =\frac1s\sup_{\norm h\le1}\sum_{j=1}^s\abs{\inner{u_j}{h}}
  =\frac1{\sqrt s},
\]
by Cauchy--Schwarz, with equality at
$h=s^{-1/2}\sum_{j=1}^su_j$.
\end{proof}

\begin{corollary}[Gaussian covariance links and dimension-free directional control]
\label{cor:gaussian-bkl-directional}
Let $\Sigma\succeq0$ on $\R^m$ and let
$\mu_\Sigma=\mathcal N(0,\Sigma)$.  With $c_0=\sqrt{\pi/2}$,
\begin{equation}
  c_0^{-1}K_\Sigma(h,h')
  =\int_{\R^m}\kb\bigl(\inner{g}{h},\inner{g}{h'}\bigr)
  \dd\mu_\Sigma(g),
  \label{eq:gaussian-bkl-link-normalization}
\end{equation}
and
\begin{equation}
  \chi_{\R^m}(\mu_\Sigma)
  =\sqrt{\frac2\pi}\,\norm{\Sigma^{1/2}}_{2\to2}.
  \label{eq:gaussian-bkl-directional-factor}
\end{equation}
Consequently, after restoring the normalization $c_0$, the one-link metric
factor is exactly $\norm{\Sigma^{1/2}}_{2\to2}$:
\begin{equation}
  \norm{\Phi_\Sigma(h)-\Phi_\Sigma(h')}_{\cH_\Sigma}^2
  =\norm{h-h'}_\Sigma
  \le\norm{\Sigma^{1/2}}_{2\to2}\norm{h-h'}_2.
  \label{eq:gaussian-bkl-normalized-metric}
\end{equation}
For a sequence of normalized Gaussian covariance links
$\Sigma_1,\ldots,\Sigma_r$, the corresponding chain therefore satisfies
\begin{equation}
  d_r
  \le d_0^{2^{-r}}
  \prod_{\ell=1}^r
  \norm{\Sigma_\ell^{1/2}}_{2\to2}^{2^{-(r-\ell+1)}}.
  \label{eq:gaussian-bkl-chain-factor}
\end{equation}
For $\Sigma=I_m$ and $G\sim\mathcal N(0,I_m)$, the normalized directional
factor equals one, whereas the normalized first-moment factor is
$c_0\E\norm{G}_2>1$ for $m\ge2$.
Thus the directional sharpening removes an otherwise dimension-dependent
Gaussian first-moment loss.
\end{corollary}

\begin{proof}
Equation~\eqref{eq:gaussian-bkl-link-normalization} is
\eqref{eq:main-covariance}.  Moreover,
\begin{align*}
  \chi_{\R^m}(\mu_\Sigma)
  &=\sup_{\norm h_2\le1}\E\abs{\inner{G_\Sigma}{h}}\\
  &=\sqrt{\frac2\pi}
  \sup_{\norm h_2\le1}\sqrt{h^\top\Sigma h}
  =\sqrt{\frac2\pi}\norm{\Sigma^{1/2}}_{2\to2},
\end{align*}
which proves \eqref{eq:gaussian-bkl-directional-factor}.
Multiplication of the link kernel by $c_0$ multiplies its squared canonical
feature distance by $c_0$; hence \eqref{eq:gaussian-bkl-normalized-metric}
follows from \Cref{prop:bkl-directional-moment} and is also the direct identity
in \Cref{prop:main-covariance}.
Iteration proves \eqref{eq:gaussian-bkl-chain-factor}.
For $\Sigma=I_m$, strict inequality
$\E\norm{G}_2>\E\abs{G_1}=\sqrt{2/\pi}$ holds for $m\ge2$.
\end{proof}

\begin{remark}[One fixed chain versus the canonical BKL space]
\label{rem:one-chain-versus-bkl}
\Cref{prop:bkl-chain-correspondence} identifies a fixed BKL measure sequence with one chain of RKHSs.
The canonical BKL space of \citet{MohammadigohariBKL2026} is instead the union of the terminal RKHSs over all normalized admissible ladders, and its complexity is the infimum of the corresponding terminal RKHS norms over all ladders that realize the same function.
It is therefore a variational union of Brownian kernel-chain RKHSs, not one fixed RKHS.
Conversely, the terminal-head model studied in \Cref{sec:main-geometry,sec:main-fixed} is an ordinary finite ReLU representation followed by only one terminal Brownian lift.
\phantomsection\label{rem:terminal-versus-recursive-lift}
These three objects---a terminal Brownian head, a fixed Brownian kernel chain, and the BKL variational union over chains---should not be conflated.
\end{remark}

\subsection{Recursive Brownian canonical-feature chains}
\label{sec:main-chain}

The correspondence above suggests a separate architecture that feeds a Brownian canonical feature into the next link.
Fix a depth $L\ge1$, an index $q\in\{1,2\}$, a real Hilbert space
$\mathcal E_0$, and an input feature map
$\Psi_0:\cX\to\mathcal E_0$.
For clarity, write $K_{q,m}$ for the kernel $K_q$ on $\R^m$ and retain the notation
$\Phi_{q,m}$ for its canonical feature map.
For $\ell=1,\ldots,L$, fix a width $m_\ell\ge1$ and let
\[
  T_\ell:\mathcal E_{\ell-1}\to\R^{m_\ell}
\]
be a bounded linear operator when $\R^{m_\ell}$ carries its $\ell_q$ norm, let
$b_\ell\in\R^{m_\ell}$, and set
$\mathcal E_\ell:=\cH_{q,m_\ell}$.
Define recursively
\begin{equation}
  u_\ell(x)
  :=\relu\bigl(T_\ell\Psi_{\ell-1}(x)+b_\ell\bigr),
  \qquad
  \Psi_\ell(x)
  :=\Phi_{q,m_\ell}\bigl(u_\ell(x)\bigr).
  \label{eq:recursive-brownian-feature-chain}
\end{equation}
For later use, introduce the induced kernels
\begin{equation}
  \kappa_0(x,x'):=\inner{\Psi_0(x)}{\Psi_0(x')}_{\mathcal E_0},
  \qquad
  \kappa_\ell(x,x')
  :=K_{q,m_\ell}\bigl(u_\ell(x),u_\ell(x')\bigr).
  \label{eq:recursive-induced-kernels}
\end{equation}

\begin{proposition}[The recursive lift is a Hilbert kernel chain]
\label{prop:recursive-lift-is-chain}
For every $\ell\in\{0,\ldots,L\}$, $\kappa_\ell$ is positive definite and
$\Psi_\ell$ is a feature map for $\kappa_\ell$.
Let $\mathcal G_\ell$ be the RKHS of $\kappa_\ell$.
There is a unique isometry
\begin{equation}
  J_\ell:\mathcal G_\ell
  \longrightarrow
  \overline{\operatorname{span}}\{\Psi_\ell(x):x\in\cX\}
  \subseteq\mathcal E_\ell
  \quad\text{such that}\quad
  J_\ell\kappa_{\ell,x}=\Psi_\ell(x).
  \label{eq:chain-feature-isometry}
\end{equation}
For $\ell\ge1$, the function
\begin{equation}
  \widetilde\kappa_\ell(g,g')
  :=
  K_{q,m_\ell}\!\left(
    \relu(T_\ell J_{\ell-1}g+b_\ell),
    \relu(T_\ell J_{\ell-1}g'+b_\ell)
  \right)
  \label{eq:recursive-link-kernel}
\end{equation}
forms a positive-definite link kernel on $\mathcal G_{\ell-1}$ and satisfies
\begin{equation}
  \kappa_\ell(x,x')
  =\widetilde\kappa_\ell
  \bigl(\kappa_{\ell-1,x},\kappa_{\ell-1,x'}\bigr).
  \label{eq:recursive-chain-identity}
\end{equation}
Hence \eqref{eq:recursive-brownian-feature-chain} is a fixed Hilbertian reproducing-kernel chain.
\end{proposition}

\begin{proof}
The kernel $\kappa_0$ is positive definite by its feature representation.
For $\ell\ge1$, \eqref{eq:recursive-induced-kernels} is the pullback of the positive-definite kernel $K_{q,m_\ell}$ through $u_\ell$, and
\[
  \kappa_\ell(x,x')
  =\inner{\Psi_\ell(x)}{\Psi_\ell(x')}_{\mathcal E_\ell}.
\]
Define first on the algebraic span of kernel sections
\[
 U_\ell\left(\sum_{i=1}^Nc_i\kappa_{\ell,x_i}\right)
 :=\sum_{i=1}^Nc_i\Psi_\ell(x_i).
\]
For every finite coefficient family,
\begin{align*}
 \left\|\sum_i c_i\Psi_\ell(x_i)\right\|_{\mathcal E_\ell}^2
 &=\sum_{i,j}c_ic_j
 \inner{\Psi_\ell(x_i)}{\Psi_\ell(x_j)}_{\mathcal E_\ell}\\
 &=\sum_{i,j}c_ic_j\kappa_\ell(x_i,x_j)\\
 &=\left\|\sum_i c_i\kappa_{\ell,x_i}\right\|_{\mathcal G_\ell}^2.
\end{align*}
Consequently the rule is well defined: any two section representations of
the same RKHS vector have difference of norm zero and therefore images with
difference of norm zero.  It is an isometry on a dense subspace of
$\mathcal G_\ell$, so it extends uniquely by continuity to an isometry
$J_\ell$.  Its range is the closure of the span of the vectors
$\Psi_\ell(x)$, exactly as displayed in \eqref{eq:chain-feature-isometry}.

For fixed $\ell\ge1$, the map
$g\mapsto\relu(T_\ell J_{\ell-1}g+b_\ell)$ is a well-defined map from
$\mathcal G_{\ell-1}$ to $\R^{m_\ell}$.
Thus \eqref{eq:recursive-link-kernel} is a pullback of $K_{q,m_\ell}$ and is positive definite.
Using
$J_{\ell-1}\kappa_{\ell-1,x}=\Psi_{\ell-1}(x)$ and
\eqref{eq:recursive-brownian-feature-chain} gives
\begin{align*}
  \widetilde\kappa_\ell
  \bigl(\kappa_{\ell-1,x},\kappa_{\ell-1,x'}\bigr)
  &=K_{q,m_\ell}\bigl(u_\ell(x),u_\ell(x')\bigr)\\
  &=\kappa_\ell(x,x'),
\end{align*}
which proves \eqref{eq:recursive-chain-identity}.
\end{proof}

Write
\begin{equation}
  \tau_\ell
  :=\norm{T_\ell}_{\mathcal E_{\ell-1}\to\ell_q},
  \qquad
  \beta_\ell:=\norm{b_\ell}_q,
  \label{eq:chain-link-parameters}
\end{equation}
and define
\begin{equation}
  d_\ell(x,x')
  :=\norm{\Psi_\ell(x)-\Psi_\ell(x')}_{\mathcal E_\ell},
  \qquad
  e_\ell(x):=\norm{\Psi_\ell(x)}_{\mathcal E_\ell}^2.
  \label{eq:chain-distance-energy}
\end{equation}

The first result gives the metric law for the recursive architecture. It is distinct from the one-shot terminal Brownian head studied in the main paper.

\begin{proposition}[Dyadic metric propagation in a recursive Brownian chain]
\label{prop:main-chain}
For a depth-$L$ recursive Brownian canonical-feature chain,
\begin{equation}
 d_L(x,x')\le d_0(x,x')^{2^{-L}}
 \prod_{\ell=1}^L\tau_\ell^{2^{-(L-\ell+1)}}.
 \label{eq:main-chain-bound}
\end{equation}
Scalar-width chains attain equality at every level, so the displayed exponents are forced among uniform monomial bounds over all positive scales.
\end{proposition}

\phantomsection\label{proof:main-chain}
\begin{proof}
By the Brownian feature identity \eqref{eq:main-isometry}, coordinatewise nonexpansiveness of ReLU in $\ell_q$, and the definition of $\tau_\ell$,
\begin{align}
  d_\ell(x,x')^2
  &=\|u_\ell(x)-u_\ell(x')\|_q\notag\\
  &\le\|T_\ell(\Psi_{\ell-1}(x)-\Psi_{\ell-1}(x'))\|_q\notag\\
  &\le\tau_\ell d_{\ell-1}(x,x').
  \label{eq:main-chain-one-step}
\end{align}
The common bias cancels before the ReLU contraction.  If
$d_{\ell-1}(x,x')=0$ or $\tau_\ell=0$, the right-hand side of
\eqref{eq:main-chain-one-step} is zero and hence so is $d_\ell(x,x')$.
In all cases, taking the nonnegative square root gives
\[
  d_\ell(x,x')\le\tau_\ell^{1/2}d_{\ell-1}(x,x')^{1/2}.
\]
Induction on $r$ now yields
\[
  d_r(x,x')\le d_0(x,x')^{2^{-r}}
  \prod_{\ell=1}^r\tau_\ell^{2^{-(r-\ell+1)}},
\]
and the case $r=L$ is \eqref{eq:main-chain-bound}.

It remains to verify attainment and necessity of the exponents. Take scalar width at every level. Suppose $\Psi_{\ell-1}(x)\ne\Psi_{\ell-1}(x')$ and set
\[
  v_{\ell-1}=\frac{\Psi_{\ell-1}(x)-\Psi_{\ell-1}(x')}{d_{\ell-1}(x,x')},
  \qquad
  T_\ell h=\tau_\ell\langle v_{\ell-1},h\rangle.
\]
Writing
$a=T_\ell\Psi_{\ell-1}(x)$ and
$a'=T_\ell\Psi_{\ell-1}(x')$, choose explicitly
$b_\ell:=\max\{0,-a,-a'\}$.  Then both scalar preactivations are
nonnegative. ReLU is then the identity at those two values, and
\[
  |u_\ell(x)-u_\ell(x')|
  =\tau_\ell d_{\ell-1}(x,x').
\]
The Brownian feature identity therefore makes \eqref{eq:main-chain-one-step} an equality. Starting with two base features at prescribed positive distance $d_0$ and
iterating this construction gives equality in
\eqref{eq:main-chain-bound} for every prescribed positive
$\tau_1,\ldots,\tau_L$.  At zero base distance or a zero link norm, both the
bound and every subsequent distance are zero, but those endpoint cases are
not used to identify exponents by scaling.

Finally, suppose a uniform monomial estimate
\[
 d_L\le C d_0^\gamma\prod_{\ell=1}^L\tau_\ell^{\alpha_\ell}
\]
holds at all positive scales. Substituting the equality construction, fixing all $\tau_\ell=1$, and sending $d_0$ first to infinity and then to zero forces $\gamma=2^{-L}$. Next fix $d_0=1$ and all link norms except $\tau_j$ equal to one. Letting $\tau_j$ tend to infinity and then to zero forces
$\alpha_j=2^{-(L-j+1)}$. Thus every exponent displayed in \eqref{eq:main-chain-bound} is necessary in the stated uniform monomial sense.
\end{proof}

For $r\ge1$, set
\begin{equation}
  A_r
  :=\prod_{\ell=1}^r
  \tau_\ell^{2^{-(r-\ell+1)}},
  \qquad A_0:=1.
  \label{eq:dyadic-operator-factor}
\end{equation}
Equivalently, $A_r=(\tau_rA_{r-1})^{1/2}$.
For $r\ge1$, also define the accumulated affine contribution
\begin{equation}
  C_r
  :=\sum_{s=1}^r
  \beta_s^{2^{-(r-s)}}
  \prod_{\ell=s+1}^r
  \tau_\ell^{2^{-(r-\ell)}},
  \qquad C_0:=0,
  \label{eq:dyadic-bias-factor}
\end{equation}
where an empty product equals one.
For $t\ge0$, define the nested affine envelope
\begin{equation}
  \mathcal R_0(t):=t,
  \qquad
  \mathcal R_r(t):=\tau_r\mathcal R_{r-1}(t)^{1/2}+\beta_r
  \quad(r\ge1).
  \label{eq:nested-affine-envelope}
\end{equation}

\begin{proposition}[Affine energy recursion for recursive Brownian chains]
\label{prop:chain-affine-energy}
For every $x\in\cX$ and $\ell\in[L]$,
\begin{equation}
  e_\ell(x)\le\tau_\ell e_{\ell-1}(x)^{1/2}+\beta_\ell.
  \label{eq:chain-one-step-energy}
\end{equation}
Consequently,
\begin{equation}
  e_L(x)\le\mathcal R_L(e_0(x))
  \le A_L^2e_0(x)^{2^{-L}}+C_L.
  \label{eq:dyadic-chain-affine-energy}
\end{equation}
If every bias vanishes, then $\mathcal R_L(t)=A_L^2t^{2^{-L}}$.  Conversely, for every $t\ge0$, every $\tau_1,\ldots,\tau_L>0$, and every $\beta_1,\ldots,\beta_L\ge0$, there is a one-point scalar-width chain satisfying $e_0=t$, $\norm{T_\ell}=\tau_\ell$, $\norm{b_\ell}_q=\beta_\ell$, and
\begin{equation}
  e_\ell=\mathcal R_\ell(t),\qquad \ell=1,\ldots,L.
  \label{eq:nested-affine-envelope-equality}
\end{equation}
Thus the nested affine envelope is exact in the worst case for the prescribed base energy and link parameters.
\end{proposition}

\begin{proof}
The feature-energy identity and $\norm{\relu(v)}_q\le\norm v_q$ give
\[
 e_\ell(x)=\norm{u_\ell(x)}_q
 \le\norm{T_\ell\Psi_{\ell-1}(x)+b_\ell}_q
 \le\tau_\ell e_{\ell-1}(x)^{1/2}+\beta_\ell,
\]
which is \eqref{eq:chain-one-step-energy}.  Since every $\mathcal R_\ell$ is nondecreasing, induction gives $e_\ell(x)\le\mathcal R_\ell(e_0(x))$.

For the closed form, the case $r=0$ is equality.  Assume inductively that
$\mathcal R_{r-1}(t)\le
A_{r-1}^2t^{2^{-(r-1)}}+C_{r-1}$.  Every term is
nonnegative, so repeated use of $\sqrt{x+y}\le\sqrt x+\sqrt y$ and the
explicit sum defining $C_{r-1}$ yields
\begin{align*}
 \mathcal R_r(t)
 &\le\tau_rA_{r-1}t^{2^{-r}}
 +\tau_r\sum_{s=1}^{r-1}
   \beta_s^{2^{-(r-s)}}
   \prod_{\ell=s+1}^{r-1}\tau_\ell^{2^{-(r-\ell)}}
 +\beta_r\\
 &=A_r^2t^{2^{-r}}+C_r.
\end{align*}
This proves \eqref{eq:dyadic-chain-affine-energy}; when all $\beta_r$ vanish, the recursion gives equality in the displayed bias-free formula.

For the one-point construction, take $\mathcal E_0=\R$ and
$\Psi_0(x)=\sqrt t$.  At level $\ell$, if
$\Psi_{\ell-1}(x)\ne0$, let
$v_{\ell-1}=\Psi_{\ell-1}(x)/\|\Psi_{\ell-1}(x)\|$ and set
$T_\ell h=\tau_\ell\inner{v_{\ell-1}}{h}$.  If the current feature is zero,
choose any unit vector $v_{\ell-1}$ in $\mathcal E_{\ell-1}$ and use the
same formula; then $T_\ell\Psi_{\ell-1}(x)=0$.  In both cases
$\|T_\ell\|=\tau_\ell$.  Take scalar width and bias
$b_\ell=\beta_\ell\ge0$.  The preactivation is exactly
$\tau_\ell e_{\ell-1}^{1/2}+\beta_\ell\ge0$, so ReLU does not alter it and
the Brownian energy identity gives
$e_\ell=\tau_\ell e_{\ell-1}^{1/2}+\beta_\ell
=\mathcal R_\ell(t)$.  Induction proves
\eqref{eq:nested-affine-envelope-equality}, including $t=0$ and zero biases.
\end{proof}

\begin{corollary}[Fixed-chain complexity and dyadic envelopes]
\label{cor:dyadic-chain-rademacher}
Fix a head radius $B\ge0$, an input sample $S=(x_1,\ldots,x_n)$ with $n\ge1$, and the recursive chain
\eqref{eq:recursive-brownian-feature-chain}.
Define
\begin{equation}
  \cF_{\mathrm{chain},B}
  :=
  \set{
    x\mapsto\inner{a}{\Psi_L(x)}_{\mathcal E_L}
    \given \norm{a}_{\mathcal E_L}\le B
  }
  \label{eq:chain-head-class}
\end{equation}
and its final activation mass
\begin{equation}
  \widehat M^{\mathrm{chain}}_{L,S}
  :=\frac1n\sum_{i=1}^n\norm{u_L(x_i)}_q
  =\frac1n\sum_{i=1}^n e_L(x_i).
  \label{eq:chain-final-mass}
\end{equation}
Then
\begin{equation}
  \Rad_S(\cF_{\mathrm{chain},B})
  =\frac Bn\E_\eps
  \norm{\sum_{i=1}^n\eps_i\Psi_L(x_i)}_{\mathcal E_L},
  \label{eq:chain-exact-rademacher}
\end{equation}
and
\begin{equation}
  \frac{B}{\sqrt2}
  \sqrt{\frac{\widehat M^{\mathrm{chain}}_{L,S}}{n}}
  \le
  \Rad_S(\cF_{\mathrm{chain},B})
  \le
  B\sqrt{\frac{\widehat M^{\mathrm{chain}}_{L,S}}{n}}.
  \label{eq:chain-sharp-final-mass}
\end{equation}
The numerical constants $1/\sqrt2$ and $1$ are best possible uniformly over all sample sizes and Brownian canonical feature configurations.
For arbitrary biases,
\begin{align}
  \Rad_S(\cF_{\mathrm{chain},B})
  &\le
  \frac B{\sqrt n}
  \left[
    \frac1n\sum_{i=1}^n
    \mathcal R_L\!\left(\norm{\Psi_0(x_i)}_{\mathcal E_0}^2\right)
  \right]^{1/2}
  \label{eq:nested-affine-chain-rademacher}\\
  &\le
  \frac B{\sqrt n}
  \left[
    A_L^2\left(
      \frac1n\sum_{i=1}^n
      \norm{\Psi_0(x_i)}_{\mathcal E_0}^{2^{1-L}}
    \right)
    +C_L
  \right]^{1/2}.
  \label{eq:affine-dyadic-chain-rademacher}
\end{align}
If $b_\ell=0$ for every $\ell$, then $C_L=0$ and
\begin{equation}
  \Rad_S(\cF_{\mathrm{chain},B})
  \le
  \frac{BA_L}{\sqrt n}
  \left(
    \frac1n\sum_{i=1}^n
    \norm{\Psi_0(x_i)}_{\mathcal E_0}^{2^{1-L}}
  \right)^{1/2}.
  \label{eq:dyadic-chain-rademacher}
\end{equation}

For the empirically centered class
\begin{equation}
  \cF_{\mathrm{chain},B,S}^{\circ}
  :=
  \set{
    x\mapsto
    \inner{a}{\Psi_L(x)-\overline\Psi_{L,S}}_{\mathcal E_L}
    \given \norm{a}_{\mathcal E_L}\le B
  },
  \qquad
  \overline\Psi_{L,S}:=\frac1n\sum_{i=1}^n\Psi_L(x_i),
  \label{eq:chain-centered-class}
\end{equation}
define
\begin{equation}
  \widehat V^{\mathrm{chain}}_{L,S}
  :=\frac1{2n^2}\sum_{i,j=1}^n d_L(x_i,x_j)^2.
  \label{eq:chain-centered-dispersion}
\end{equation}
Then, without any restriction on the biases,
\begin{align}
  \frac{B}{\sqrt2}
  \sqrt{\frac{\widehat V^{\mathrm{chain}}_{L,S}}{n}}
  &\le
  \Rad_S(\cF_{\mathrm{chain},B,S}^{\circ})
  \le
  B\sqrt{\frac{\widehat V^{\mathrm{chain}}_{L,S}}{n}},
  \label{eq:chain-centered-two-sided}\\
  \Rad_S(\cF_{\mathrm{chain},B,S}^{\circ})
  &\le
  \frac{BA_L}{\sqrt n}
  \left(
    \frac1{2n^2}\sum_{i,j=1}^n
    d_0(x_i,x_j)^{2^{1-L}}
  \right)^{1/2}.
  \label{eq:chain-centered-dyadic}
\end{align}
Under $0\le\tau_\ell\le\Lambda$, the factor $A_L$ in
\eqref{eq:dyadic-chain-rademacher} and
\eqref{eq:chain-centered-dyadic} is at most
$\Lambda^{1-2^{-L}}$.
\end{corollary}

\begin{proof}
Apply \Cref{thm:main-fixed} to the fixed representation
$x\mapsto u_L(x)$ and its canonical feature map
$\Phi_{q,m_L}(u_L(x))\allowbreak=\Psi_L(x)$.
Hilbert duality gives \eqref{eq:chain-exact-rademacher}, and
\Cref{lem:main-khintchine} with
$\sum_i\|\Psi_L(x_i)\|^2=n\widehat M^{\mathrm{chain}}_{L,S}$ gives
\eqref{eq:chain-sharp-final-mass}.  If $B=0$ or the final mass is zero, all
terms vanish.  For $n=1$ and nonzero final mass, the signed feature norm is
deterministic and the upper endpoint is the exact value.  Uniform optimality
of the upper constant follows already from this one-point case, while
uniform optimality of $1/\sqrt2$ follows from the two-equal-feature
configuration at $n=2$ in \Cref{thm:main-fixed}.
Average the two estimates in \Cref{prop:chain-affine-energy} and substitute them into the upper bound in
\eqref{eq:chain-sharp-final-mass}:
\begin{align*}
  \Rad_S(\cF_{\mathrm{chain},B})
  &\le
  \frac B{\sqrt n}
  \left(\frac1n\sum_i e_L(x_i)\right)^{1/2}
  \le
  \frac B{\sqrt n}
  \left[\frac1n\sum_i\mathcal R_L(e_0(x_i))\right]^{1/2}\\
  &\le
  \frac B{\sqrt n}
  \left[
    A_L^2\left(\frac1n\sum_i e_0(x_i)^{2^{-L}}\right)
    +C_L
  \right]^{1/2}.
\end{align*}
Since $e_0(x)=\norm{\Psi_0(x)}_{\mathcal E_0}^2$, this proves
\eqref{eq:nested-affine-chain-rademacher}--\eqref{eq:affine-dyadic-chain-rademacher}.
When all biases vanish, \Cref{prop:chain-affine-energy} gives the envelope $A_L^2e_0^{2^{-L}}$, giving \eqref{eq:dyadic-chain-rademacher}.

Next apply \Cref{cor:main-centered} to the same final feature configuration.
The Hilbert variance identity and \eqref{eq:chain-distance-energy} give exactly
\eqref{eq:chain-centered-dispersion}, and hence
\eqref{eq:chain-centered-two-sided}.
Squaring \eqref{eq:main-chain-bound} yields
\[
  d_L(x_i,x_j)^2
  \le
  A_L^2d_0(x_i,x_j)^{2^{1-L}}.
\]
Average over the ordered pairs and insert the result into the upper bound in
\eqref{eq:chain-centered-two-sided} to obtain
\eqref{eq:chain-centered-dyadic}.
The final assertion follows because the exponents in $A_L$ sum to $1-2^{-L}$.
\end{proof}

\begin{remark}[Empirical centering]
The centered statements in
\eqref{eq:chain-centered-two-sided}--\eqref{eq:chain-centered-dyadic}
are exact empirical complexity statements and expose the bias-free pairwise geometry of the chain.
For an i.i.d.\ risk bound, the center should be fixed, learned independently, or handled together with its data dependence, as discussed after \Cref{cor:main-centered}.
\end{remark}

\begin{proposition}[Finite-sample compression without link-norm inflation]
\label{prop:chain-finite-sample-representer}
Fix the chain \eqref{eq:recursive-brownian-feature-chain} and a finite input
sample $S=(x_1,\ldots,x_n)$ with $n\ge1$.
For $\ell=0,\ldots,L-1$, let
\begin{equation}
  V_\ell
  :=\operatorname{span}\{\Psi_\ell(x_i):i\in[n]\}
  \subseteq\mathcal E_\ell,
  \qquad
  P_\ell:\mathcal E_\ell\to V_\ell
  \label{eq:chain-sample-spans}
\end{equation}
be the orthogonal projection.
Choose an index set $I_\ell\subseteq[n]$ such that
$\{\Psi_\ell(x_i):i\in I_\ell\}$ is a basis of $V_\ell$; when
$V_\ell=\{0\}$, take $I_\ell=\varnothing$.
Define
\begin{equation}
  \widetilde T_{\ell+1}:=T_{\ell+1}P_\ell.
  \label{eq:projected-chain-links}
\end{equation}
Let $(\widetilde u_\ell,\widetilde\Psi_\ell)_{\ell=1}^L$ be the recursive chain obtained by replacing every $T_\ell$ with $\widetilde T_\ell$, while keeping the biases and the base map unchanged; thus
$\widetilde\Psi_0:=\Psi_0$.
Then, for every $i\in[n]$,
\begin{align}
  \widetilde\Psi_0(x_i)&=\Psi_0(x_i),
  \label{eq:chain-sample-exactness-base}\\
  \widetilde u_\ell(x_i)&=u_\ell(x_i),
  \qquad
  \widetilde\Psi_\ell(x_i)=\Psi_\ell(x_i),
  \quad \ell=1,\ldots,L.
  \label{eq:chain-sample-exactness}
\end{align}
Moreover,
\begin{equation}
  \norm{\widetilde T_\ell}_{\mathcal E_{\ell-1}\to\ell_q}
  \le\tau_\ell,
  \qquad \ell=1,\ldots,L,
  \label{eq:chain-link-norm-preservation}
\end{equation}
and, for every $\ell\in[L]$, there exist vectors
$c_{\ell,i}\in\R^{m_\ell}$, $i\in I_{\ell-1}$, such that
\begin{equation}
  \widetilde T_\ell h
  =\sum_{i\in I_{\ell-1}}
  c_{\ell,i}
  \inner{\Psi_{\ell-1}(x_i)}{h}_{\mathcal E_{\ell-1}}
  \qquad(h\in\mathcal E_{\ell-1}).
  \label{eq:chain-finite-link-expansion}
\end{equation}
Thus every link has an exact representation using
$\abs{I_{\ell-1}}=\dim V_{\ell-1}\le n$ linearly independent sample features, without increasing its operator norm.
Equivalently, the projected chain can be evaluated recursively through finite kernel expansions,
\begin{equation}
  \widetilde u_\ell(x)
  =\relu\!\left(
    \sum_{i\in I_{\ell-1}} c_{\ell,i}
    \kappa^S_{\ell-1}(x_i,x)
    +b_\ell
  \right),
  \label{eq:chain-finite-kernel-recursion}
\end{equation}
where, for $r=0,\ldots,L$,
\begin{equation}
  \kappa^S_r(x,x')
  :=\inner{\widetilde\Psi_r(x)}
  {\widetilde\Psi_r(x')}_{\mathcal E_r}.
  \label{eq:compressed-chain-kernel}
\end{equation}
On the sample, these kernels preserve every original Gram matrix:
\begin{equation}
  \kappa^S_\ell(x_i,x_j)=\kappa_\ell(x_i,x_j)
  \qquad(i,j\in[n],\ \ell=0,\ldots,L).
  \label{eq:compressed-chain-gram-exactness}
\end{equation}
\end{proposition}

\begin{proof}
Because $P_\ell$ is an orthogonal projection,
$\norm{P_\ell}_{\mathcal E_\ell\to\mathcal E_\ell}\le1$.
Therefore
\[
  \norm{\widetilde T_{\ell+1}}_{\mathcal E_\ell\to\ell_q}
  =\norm{T_{\ell+1}P_\ell}_{\mathcal E_\ell\to\ell_q}
  \le\norm{T_{\ell+1}}_{\mathcal E_\ell\to\ell_q},
\]
which is \eqref{eq:chain-link-norm-preservation}.

Equations~\eqref{eq:chain-sample-exactness-base}--\eqref{eq:chain-sample-exactness} follow by induction on $\ell$.
The base identity holds because the base feature map is unchanged.
Assume
$\widetilde\Psi_{\ell-1}(x_i)=\Psi_{\ell-1}(x_i)$ for all $i$.
Since each sample feature belongs to $V_{\ell-1}$,
$P_{\ell-1}\Psi_{\ell-1}(x_i)=\Psi_{\ell-1}(x_i)$, and hence
\begin{align*}
  \widetilde T_\ell\widetilde\Psi_{\ell-1}(x_i)
  &=T_\ell P_{\ell-1}\Psi_{\ell-1}(x_i)\\
  &=T_\ell\Psi_{\ell-1}(x_i).
\end{align*}
Adding the same bias, applying ReLU, and then applying the same canonical Brownian feature map proves equality of both
$\widetilde u_\ell(x_i)$ and $\widetilde\Psi_\ell(x_i)$.

For the finite expansion, fix $\ell$ and a coordinate
$j\in[m_\ell]$.
The map
$h\mapsto[\widetilde T_\ell h]_j$ is a bounded linear functional on the Hilbert space $\mathcal E_{\ell-1}$, so there is a unique representer
$v_{\ell,j}\in\mathcal E_{\ell-1}$ such that
$[\widetilde T_\ell h]_j=\inner{v_{\ell,j}}{h}$.
Because
$\widetilde T_\ell=\widetilde T_\ell P_{\ell-1}$,
\[
  \inner{v_{\ell,j}}{h}
  =\inner{v_{\ell,j}}{P_{\ell-1}h}
  =\inner{P_{\ell-1}v_{\ell,j}}{h}
  \qquad(h\in\mathcal E_{\ell-1}),
\]
so uniqueness of the Riesz representer implies
$v_{\ell,j}=P_{\ell-1}v_{\ell,j}\in V_{\ell-1}$.
Because the features indexed by $I_{\ell-1}$ form a basis of
$V_{\ell-1}$, there are unique coefficients
$\gamma_{\ell,i,j}$, $i\in I_{\ell-1}$, such that
\[
  v_{\ell,j}
  =\sum_{i\in I_{\ell-1}}
  \gamma_{\ell,i,j}\Psi_{\ell-1}(x_i).
\]
Set
$c_{\ell,i}:=(\gamma_{\ell,i,1},\ldots,\gamma_{\ell,i,m_\ell})$.
Collecting the coordinate identities yields
\eqref{eq:chain-finite-link-expansion}.
Substituting
$h=\widetilde\Psi_{\ell-1}(x)$ and using sample exactness at the centers gives
\eqref{eq:chain-finite-kernel-recursion}.
Finally, \eqref{eq:chain-sample-exactness-base}--\eqref{eq:chain-sample-exactness} and the feature representations of
$\kappa_\ell$ and $\kappa^S_\ell$ imply
\eqref{eq:compressed-chain-gram-exactness}.
\end{proof}

\begin{proposition}[Spectral-tail rank compression with recursive error control]
\label{prop:spectral-tail-chain-compression}
Fix the recursive chain \eqref{eq:recursive-brownian-feature-chain} and a
sample $S=(x_1,\ldots,x_n)$. For $\ell=0,\ldots,L-1$, define the empirical
feature covariance
\begin{equation}
  C_{\ell,S}:=\frac1n\sum_{i=1}^n
  \Psi_\ell(x_i)\otimes\Psi_\ell(x_i)
  \quad\text{on }\mathcal E_\ell,
  \label{eq:chain-empirical-covariance}
\end{equation}
let
$\lambda_{\ell,1}\ge\lambda_{\ell,2}\ge\cdots\ge0$ be its eigenvalues,
and let $P_{\ell,r}$ project onto a top-$r$ eigenspace.
Set
\begin{equation}
  \delta_{\ell,r}
  :=\left(\sum_{j>r}\lambda_{\ell,j}\right)^{1/2}
  =\left[
    \frac1n\sum_{i=1}^n
    \norm{(I-P_{\ell,r})\Psi_\ell(x_i)}_{\mathcal E_\ell}^2
  \right]^{1/2}.
  \label{eq:chain-spectral-tail}
\end{equation}
Choose integer ranks $r_0,\ldots,r_{L-1}$ with $0\le r_\ell\le n$; in
the presence of an eigenvalue tie, any orthogonal top-$r_\ell$ eigenspace may
be used.  Define the rank-compressed chain by
$\widetilde\Psi_0=\Psi_0$ and
\begin{equation}
  \widetilde u_\ell(x)
  :=\relu\bigl(T_\ell P_{\ell-1,r_{\ell-1}}
  \widetilde\Psi_{\ell-1}(x)+b_\ell\bigr),
  \qquad
  \widetilde\Psi_\ell(x)
  :=\Phi_{q,m_\ell}(\widetilde u_\ell(x)).
  \label{eq:rank-compressed-chain}
\end{equation}
Then every compressed link factors through a space of dimension at most
$r_{\ell-1}$ and
\begin{equation}
  \norm{T_\ell P_{\ell-1,r_{\ell-1}}}
  \le\tau_\ell.
  \label{eq:rank-compressed-link-norm}
\end{equation}
If
\begin{equation}
  E_0:=0,
  \qquad
  E_\ell:=\frac1n\sum_{i=1}^n
  \norm{\Psi_\ell(x_i)-\widetilde\Psi_\ell(x_i)}_{\mathcal E_\ell}^2,
  \label{eq:chain-compression-error}
\end{equation}
then the errors obey the recursive spectral-tail bound
\begin{equation}
  E_\ell
  \le\tau_\ell
  \left(\delta_{\ell-1,r_{\ell-1}}+\sqrt{E_{\ell-1}}\right),
  \qquad \ell=1,\ldots,L.
  \label{eq:chain-compression-error-recursion}
\end{equation}
Writing $a_\ell:=\tau_\ell\delta_{\ell-1,r_{\ell-1}}$, this gives the explicit
closed envelope
\begin{equation}
  E_L
  \le
  \sum_{s=1}^L
  a_s^{\,2^{-(L-s)}}
  \prod_{\ell=s+1}^L
  \tau_\ell^{\,2^{-(L-\ell)}}.
  \label{eq:chain-compression-closed-form}
\end{equation}
For every $a\in\mathcal E_L$ with $\norm{a}\le B$, the original and compressed
heads satisfy
\begin{equation}
  \left[
    \frac1n\sum_{i=1}^n
    \abs{\inner{a}{\Psi_L(x_i)-\widetilde\Psi_L(x_i)}}^2
  \right]^{1/2}
  \le B\sqrt{E_L}.
  \label{eq:chain-compression-prediction-error}
\end{equation}
If $\widetilde\cF_{\mathrm{chain},B}$ denotes the radius-$B$ head class generated by $\widetilde\Psi_L$, then
\begin{equation}
  \abs{
  \Rad_S(\cF_{\mathrm{chain},B})
  -\Rad_S(\widetilde\cF_{\mathrm{chain},B})}
  \le B\sqrt{\frac{E_L}{n}}.
  \label{eq:chain-compression-rademacher-stability}
\end{equation}
If only level $s\in\{0,\ldots,L-1\}$ is truncated and all other projections
contain the corresponding full sample spans, then
\begin{equation}
  E_L
  \le
  (\tau_{s+1}\delta_{s,r_s})^{2^{-(L-s-1)}}
  \prod_{\ell=s+2}^L
  \tau_\ell^{2^{-(L-\ell)}}.
  \label{eq:single-level-spectral-chain-error}
\end{equation}
In particular, zero spectral tails recover the exact sample compression of
\Cref{prop:chain-finite-sample-representer}.
\end{proposition}

\begin{proof}
Orthogonal projection is contractive, which proves
\eqref{eq:rank-compressed-link-norm} and the rank-factorization statement.
For each sample point, the Brownian feature identity, ReLU contraction, and
the operator norm give
\begin{align*}
  &\norm{\Psi_\ell(x_i)-\widetilde\Psi_\ell(x_i)}^2\\
  &\quad=\norm{u_\ell(x_i)-\widetilde u_\ell(x_i)}_q\\
  &\quad\le\tau_\ell
  \norm{\Psi_{\ell-1}(x_i)
  -P_{\ell-1,r_{\ell-1}}\widetilde\Psi_{\ell-1}(x_i)}\\
  &\quad\le\tau_\ell\left[
    \norm{(I-P_{\ell-1,r_{\ell-1}})\Psi_{\ell-1}(x_i)}
    +\norm{\Psi_{\ell-1}(x_i)-\widetilde\Psi_{\ell-1}(x_i)}
  \right].
\end{align*}
Average over $i$ and apply Cauchy--Schwarz separately to the two sums.
The first becomes $\delta_{\ell-1,r_{\ell-1}}$ by the PCA residual identity,
and the second is at most $\sqrt{E_{\ell-1}}$, proving
\eqref{eq:chain-compression-error-recursion}.
Equation~\eqref{eq:chain-compression-prediction-error} is Cauchy--Schwarz in
$\mathcal E_L$ followed by averaging.  By the reverse triangle inequality and Jensen,
\begin{align*}
  \abs{\Rad_S(\cF_{\mathrm{chain},B})
  -\Rad_S(\widetilde\cF_{\mathrm{chain},B})}
  &\le\frac Bn\E_\eps
  \norm{\sum_i\eps_i(\Psi_L(x_i)-\widetilde\Psi_L(x_i))}\\
  &\le\frac Bn\left(\sum_i
  \norm{\Psi_L(x_i)-\widetilde\Psi_L(x_i)}^2\right)^{1/2},
\end{align*}
which is \eqref{eq:chain-compression-rademacher-stability}.
For a single truncation, exact projection at the earlier levels gives
$E_s=0$ and hence
$E_{s+1}\le\tau_{s+1}\delta_{s,r_s}$.
For the general multi-level statement, write the recurrence as
$E_\ell\le a_\ell+\tau_\ell\sqrt{E_{\ell-1}}$.  We prove
\eqref{eq:chain-compression-closed-form} by induction on $\ell$.  It is
$E_1\le a_1$ at the first level.  If
\[
 E_{\ell-1}\le
 \sum_{s=1}^{\ell-1}
 a_s^{2^{-(\ell-1-s)}}
 \prod_{j=s+1}^{\ell-1}
 \tau_j^{2^{-(\ell-1-j)}},
\]
then square-root subadditivity gives
\begin{align*}
 E_\ell
 &\le a_\ell+\tau_\ell
 \sum_{s=1}^{\ell-1}
 a_s^{2^{-(\ell-s)}}
 \prod_{j=s+1}^{\ell-1}
 \tau_j^{2^{-(\ell-j)}}\\
 &=\sum_{s=1}^{\ell}
 a_s^{2^{-(\ell-s)}}
 \prod_{j=s+1}^{\ell}
 \tau_j^{2^{-(\ell-j)}},
\end{align*}
where the $s=\ell$ term is $a_\ell$ and empty products equal one.  Taking
$\ell=L$ proves \eqref{eq:chain-compression-closed-form}.
At every later level after a single truncation the tail term vanishes and
$E_\ell\le\tau_\ell\sqrt{E_{\ell-1}}$; iteration yields
\eqref{eq:single-level-spectral-chain-error}.
\end{proof}

% Keep this short closing remark together.
\begin{samepage}
\begin{remark}[Compression-localized factors and scope]
\label{rem:compression-localized-link-factors}
\label{rem:spectral-tail-chain-compression}
\label{rem:chain-finite-realization-scope}
Let $Q^{\rm cmp}_{\ell-1,S}$ project onto the span of
$\Psi_{\ell-1}(x_i)-P_{\ell-1,r_{\ell-1}}\widetilde\Psi_{\ell-1}(x_i)$,
$i\in[n]$, and set
$\tau^{\rm cmp}_{\ell,S}:=\norm{T_\ell Q^{\rm cmp}_{\ell-1,S}}_{\mathcal E_{\ell-1}\to\ell_q}$.
Because the compression proof uses $T_\ell$ only on this discrepancy span,
all bounds remain valid after the recursive replacement
$\tau_\ell\mapsto\tau^{\rm cmp}_{\ell,S}\le\tau_\ell$.
This localizes the low-rank spectral-tail control suggested by adaptive-kernel
work~\citep{LiuHuangGongYangLi2020}: unlike a training-only Hadamard
deformation, the projection acts on a globally defined canonical chain,
does not increase link norms, and supplies an explicit sample-dependent
approximation error. It does not assert an off-sample Euclidean realization.
\end{remark}
\end{samepage}

\begin{corollary}[Sample-adaptive link and secant sharpening]
\label{cor:sample-adaptive-chain-norms}
In the setting of \Cref{prop:chain-finite-sample-representer}, let
$Q_{\ell-1,S}$ project onto the span of sample-feature differences and define
\begin{align}
  \tau_{\ell,S}
  &:=\norm{T_\ell P_{\ell-1}}_{\mathcal E_{\ell-1}\to\ell_q},
  &A_{L,S}&:=\prod_{\ell=1}^L\tau_{\ell,S}^{2^{-(L-\ell+1)}},
  \label{eq:sample-adaptive-chain-factor}\\
  D_{\ell-1,S}
  &:=\operatorname{span}\{\Psi_{\ell-1}(x_i)-\Psi_{\ell-1}(x_j):i,j\in[n]\},
  \label{eq:sample-difference-span}\\
  \tau_{\ell,S}^{\Delta}
  &:=\norm{T_\ell Q_{\ell-1,S}}_{\mathcal E_{\ell-1}\to\ell_q},
  \notag\\
  \sigma_{\ell,S}
  &:=\max_{\substack{i,j\in[n]\\d_{\ell-1}(x_i,x_j)>0}}
  \frac{\norm{u_\ell(x_i)-u_\ell(x_j)}_q}{d_{\ell-1}(x_i,x_j)},
  &A_{L,S}^{\mathrm{sec}}
  &:=\prod_{\ell=1}^L\sigma_{\ell,S}^{2^{-(L-\ell+1)}}.
  \label{eq:sample-secant-factor}
\end{align}
The maximum over an empty set is zero. Set
$\mathcal R_{0,S}(t):=t$ and
$\mathcal R_{r,S}(t):=\tau_{r,S}\mathcal R_{r-1,S}(t)^{1/2}+\beta_r$.
Then
\begin{equation}
  \sigma_{\ell,S}\le\tau_{\ell,S}^{\Delta}\le\tau_{\ell,S}\le\tau_\ell,
  \quad A_{L,S}^{\mathrm{sec}}\le A_{L,S}\le A_L,
  \quad \mathcal R_{L,S}(t)\le\mathcal R_L(t).
  \label{eq:sample-factor-hierarchy}
\end{equation}
Moreover,
\begin{align}
  \Rad_S(\cF_{\mathrm{chain},B})
  &\le\frac B{\sqrt n}
  \left[\frac1n\sum_{i=1}^n
  \mathcal R_{L,S}\!\left(\norm{\Psi_0(x_i)}_{\mathcal E_0}^2\right)\right]^{1/2},
  \label{eq:sample-adaptive-affine-chain-bound}\\
  \Rad_S(\cF_{\mathrm{chain},B})
  &\le\frac{BA_{L,S}}{\sqrt n}
  \left(\frac1n\sum_{i=1}^n
  \norm{\Psi_0(x_i)}_{\mathcal E_0}^{2^{1-L}}\right)^{1/2}
  \quad\text{if the chain is bias free},
  \label{eq:sample-adaptive-chain-rademacher}\\
  \Rad_S(\cF_{\mathrm{chain},B,S}^{\circ})
  &\le\frac{BA_{L,S}^{\mathrm{sec}}}{\sqrt n}
  \left(\frac1{2n^2}\sum_{i,j=1}^n
  d_0(x_i,x_j)^{2^{1-L}}\right)^{1/2}.
  \label{eq:sample-adaptive-centered-chain}
\end{align}
The uncentered bounds discard link mass outside the sample-feature spans;
the centered bound retains only contraction realized on sample secants.
\end{corollary}

\begin{proof}
The projected chain of \Cref{prop:chain-finite-sample-representer} has link
norms $\tau_{\ell,S}$ and agrees with the original chain at every sample point
and every level.  Applying \Cref{cor:dyadic-chain-rademacher} to that chain proves
\eqref{eq:sample-adaptive-affine-chain-bound}--\eqref{eq:sample-adaptive-chain-rademacher}.
Because $D_{\ell-1,S}\subseteq V_{\ell-1}$,
$T_\ell Q_{\ell-1,S}=(T_\ell P_{\ell-1})Q_{\ell-1,S}$, so
$\tau_{\ell,S}^{\Delta}\le\tau_{\ell,S}$.
On each sample difference, ReLU contraction gives
$\norm{u_\ell(x_i)-u_\ell(x_j)}_q
\le\norm{T_\ell Q_{\ell-1,S}(\Psi_{\ell-1}(x_i)-\Psi_{\ell-1}(x_j))}_q$.
If $d_{\ell-1}(x_i,x_j)>0$, division by that distance and the operator norm
give the ratio bound by $\tau_{\ell,S}^{\Delta}$.  If the denominator is
zero, then $\Psi_{\ell-1}(x_i)=\Psi_{\ell-1}(x_j)$, so the two affine
preactivations, ReLU states, and numerator are also identical.  If no pair
has positive denominator (in particular when $n=1$), the defining maximum is
zero.  Therefore $\sigma_{\ell,S}\le\tau_{\ell,S}^{\Delta}$ in all cases.
For each sample pair, induction in exactly the same way as in
\Cref{prop:main-chain} gives
\[
 d_L(x_i,x_j)^2
 \le (A_{L,S}^{\mathrm{sec}})^2
 d_0(x_i,x_j)^{2^{1-L}}.
\]
This remains true when a secant factor is zero, because the corresponding
one-step distance and all subsequent distances are zero.  Averaging this
inequality over ordered pairs and substituting it into the centered upper
bound in \eqref{eq:chain-centered-two-sided} proves
\eqref{eq:sample-adaptive-centered-chain}.
The remaining inequalities follow from the product and recursive definitions.
\end{proof}

\end{document}